\documentclass[opre, nonblindrev]{informs3NoHeading} 

\OneAndAHalfSpacedXI 

\usepackage{endnotes}
\let\footnote=\endnote
\let\enotesize=\normalsize
\def\notesname{Endnotes}%
\def\makeenmark{$^{\theenmark}$}
\def\enoteformat{\rightskip0pt\leftskip0pt\parindent=1.75em
  \leavevmode\llap{\theenmark.\enskip}}

\usepackage{natbib}
 \bibpunct[, ]{(}{)}{,}{a}{}{,}%
 \def\bibfont{\small}%
 \def\bibsep{\smallskipamount}%
 \def\bibhang{24pt}%

\usepackage{bm}
\usepackage{algorithm,algpseudocode}
\usepackage{diagbox}
\usepackage{multirow}
\usepackage{enumerate}
\usepackage{mathtools}
\usepackage{mathrsfs}
\usepackage{enumitem}
\usepackage{dsfont}
\usepackage{float}
\usepackage{amssymb}
\usepackage{lscape}
\usepackage{graphicx}
\usepackage{amsmath}
\usepackage{bbm}
\usepackage{comment}
\usepackage{tikz}
\usepackage{booktabs}
\usepackage{enumitem}
\usepackage{natbib}
\usepackage{caption}
\usepackage[colorlinks,linkcolor=blue,citecolor=blue]{hyperref}
\usepackage[labelformat=simple]{subcaption}

\usetikzlibrary{calc,positioning}

\tikzset{round/.style = { rounded corners=2mm }}

\usepackage{cleveref}
\crefname{section}{section}{sections}

\newcommand{\E}{\mathbb{E}}

\newcommand{\beq}{\begin{eqnarray}}
\newcommand{\eeq}{\end{eqnarray}}
\newcommand{\beqn}{\begin{eqnarray*}}
\newcommand{\eeqn}{\end{eqnarray*}}

\newtheorem{theorem}{Theorem}
\newtheorem{corollary}{Corollary}
\newtheorem{lemma}{Lemma}
\newtheorem{proposition}{Proposition}
\newtheorem{defn}{Definition}

\newtheorem{assumption}{Assumption}

\bibpunct{(}{)}{;}{a}{,}{,}

\DeclarePairedDelimiter\floor{\lfloor}{\rfloor}

\EquationsNumberedThrough    

\begin{document}


\RUNAUTHOR{Jiang, Li and Zhang}

\RUNTITLE{Online Stochastic Optimization with Wasserstein Based Non-stationarity}

\TITLE{Online Stochastic Optimization with Wasserstein Based Non-stationarity}

\ARTICLEAUTHORS{%
\AUTHOR{$\text{Jiashuo Jiang}^\dagger$ \quad $\text{Xiaocheng Li}^\ddagger$ \quad $\text{Jiawei Zhang}^\dagger$}

\AFF{\  \\
$\dagger~$Department of Technology, Operations \& Statistics, Stern School of Business, New York University\\
$\ddagger~$Imperial College Business School\\
}
}

\ABSTRACT{\textbf{Abstract:}
We consider a general online stochastic optimization problem with multiple budget constraints over a horizon of finite time periods. In each time period, a reward function and multiple cost functions are revealed, and the decision maker needs to specify an action from a convex and compact action set to collect the reward and consume the budgets. Each cost function corresponds to the consumption of one budget constraint. The reward function and the cost functions of each time period are drawn from an unknown distribution, which is non-stationary across time. The objective of the decision maker is to maximize the cumulative reward subject to the budget constraints. This formulation captures a wide range of applications including online linear programming and network revenue management, among others. In this paper, we consider two settings: (i) a data-driven setting where the true distribution is unknown but a prior estimate (possibly inaccurate) is available; (ii) an uninformative setting where the true distribution is completely unknown. We propose a unified Wasserstein-distance based measure to quantify the inaccuracy of the prior estimate in Setting (i) and the non-stationarity of the environment in Setting (ii). We show that the proposed measure leads to a necessary and sufficient condition for the attainability of a sublinear regret in both settings. For Setting (i), we propose an \textit{informative gradient descent} algorithm. The algorithm takes a primal-dual perspective and it integrates the prior information of the underlying distributions into an online gradient descent procedure in the dual space. The algorithm also naturally extends to the uninformative setting (ii). Under both settings, we show the corresponding algorithm achieves a regret of optimal order. We illustrate the algorithm performance through numerical experiments.}



\maketitle

\section{Introduction}

In this paper, we study a general online stochastic optimization problem with $m$ budget constraints, each with an initial capacity, over a finite horizon of discrete time periods. At each time $t$, a reward function $f_t:\mathcal{X}\rightarrow \mathbb{R}$ and a cost function $\bm{g}_t:\mathcal{X}\rightarrow \mathbb{R}^m$ are drawn independently from a distribution. Upon the observation of the functions, the decision maker specifies a decision $\bm{x}_t\in\mathcal{X}$, where $\mathcal{X}$ is assumed to be a convex and compact set. Accordingly, a reward $f(\bm{x}_t)$ is collected, and each budget $i\in\{1,2,\dots,m\}$ is consumed by an amount of $g_{it}(\bm{x}_t)$, where $\bm{g}_t(\bm{x})=(g_{1t}(\bm{x}),...,g_{mt}(\bm{x}))^\top$. The decision maker's objective is to maximize the totally collected reward subject to the budget capacity constraints.

Our formulation generalizes several existing problems studied in the literature. When $f_t$ and $\bm{g}_t$ are linear functions for each $t$, our formulation reduces to the online linear programming (OLP) problem \citep{buchbinder2009online}. Our formulation also covers the network revenue management (NRM) problem \citep{talluri2006theory}, including the quantity-based model, the price-based model and the choice-based model \citep{talluri2004revenue}. Note that for the OLP literature, the reward function and cost functions are assumed to be drawn from an unknown distribution which is stationary over time (or from a random permutation), while the NRM literature always assumes a precise knowledge of the true distribution, though it can be nonidentical across time. In this paper, our focus is an unknown non-stationary setting where the functions $f_t$ and $\bm{g}_t$ are drawn from a distribution $\mathcal{P}_t$ that is nonidentical over time and unknown to the decision maker. Specifically, we mainly consider two settings: (i) a data-driven setting where there exists an available prior estimate $\hat{\mathcal{P}}_t$ (possibly inaccurate) for the true distribution $\mathcal{P}_t$ of each time period and (ii) an uninformative setting where the true distribution $\mathcal{P}_t$ is completely unknown. When the prior estimates are identical to the true distributions, the data-driven setting reduces to the known non-stationary setting considered in the NRM literature. When the distribution is identical over time, the uninformative setting reduces to the unknown stationary setting considered in the OLP literature.

For both settings, we assume that the true distribution falls into an uncertainty set, which controls the inaccuracy of the prior estimate in the data-driven setting and the non-stationarity of the distributions in the uninformative setting. Our goal is to derive near-optimal policies for both settings, which perform well over the entire uncertainty set. We compare the performances of our algorithms/policies to the ``offline'' optimization problem which maximizes the total reward with full information/knowledge of all the $f_t$'s and $\bm{g}_t$'s. We use \textit{regret} as the performance metric, which is defined as worst-case additive optimality gap (over the uncertainty set) between the total reward generated by an algorithm/policy and the offline optimal objective value.

\subsection{Main Results and Contributions}

For the data-driven setting (in Section \ref{WBNB_P}), the true distribution $\mathcal{P}_t$ is unknown, but we assume the availability of a prior estimate $\hat{\mathcal{P}}_t$. The prior estimate can be based on some history data. We propose a Wasserstein-based deviation measure, \textit{Wasserstein-based deviation budget} (WBDB), to quantify the deviation of the prior estimate from the true distribution. Based on WBDB, we introduce an uncertainty set driven by the notion of WBDB and a parameter $W_T$ (called  \textit{deviation budget}), and the uncertainty set encapsulates all the distributions $\mathcal{P}_t$'s that have WBDB no greater than $W_T.$ We illustrate the sharpness of WBDB by showing that if the variation budget $W_T$ is linear in $T$, sublinear regret could not be achieved by any admissible policy. Next, we develop a new \textit{Informative Gradient Descent} with prior estimate (IGDP) algorithm, which adaptively combines the prior distribution knowledge into an update in the dual space. Our algorithm is motivated by the traditional online gradient descent (OGD) algorithm \citep{hazan2016introduction}. The OGD algorithm applies a linear update rule according to the gradient information at the current period and has been shown to work well in the stationary setting, even when the distribution is unknown \citep{lu2020dual,sun2020near,li2020simple}. The OGD type algorithm has also been developed under stationary setting with unknown distributions in \citet{agrawal2014fast} by assuming further stronger conditions on the dual optimal solution, and with known distributions for service level problems \citep{jiang2019achieving}. However, the update in OGD for each time period only involves information gathered up to the current time period, but for the non-stationary setting, we also need to take advantage of the prior estimates of the future time periods. Specifically, based on a primal-dual convex relaxation of the underlying offline problem, we obtain a prescribed allocation of the budgets over the entire horizon based on the prior estimates. Then, the IGDP algorithm uses this allocation to adjust the gradient descent direction. This idea is new for the related literature in that the IGDP descent direction at each period does not simply come from the historical observations, but it is also informed by the distribution knowledge of the entire horizon. We show that the IGDP algorithm achieves the \textit{first} optimal regret upper bound $O(\max\{\sqrt{T}, W_T\})$.

A few recent works also study similar online decision making problems in a non-stationary environment where $\mathcal{P}_t$ may vary over time. \cite{devanur2019near} study the case of known distribution ($W_T=0$) and obtain a $1-O(1/\sqrt{c})$ competitive ratio, where $c$ denotes the minimal capacity of the budget constraints. It remains unclear how to generalize the method in \citep{devanur2019near} to the setting of $W_T>0$ where the distribution knowledge is inaccurate or absent. A line of works \citep{vera2020bayesian, fruend2020a, fruend2020b} study the known distribution setting ($W_T=0$) with an additional assumption that the underlying distribution takes a finite support. These works develop algorithms that achieve bounded regret that bears no dependency on $T$. Compared to this stream of literature, the main results of our paper do not assume the finite supportedness. In addition, when the underlying distribution is finite, we extend the previous algorithm and analysis for the case of $W_T>0.$ Another recent work \citep{cheung2020online} studies the non-stationary problem and proposes dual-based algorithms that utilize the trajectories sampled from the prior estimate distribution, which can be classified as a static policy (see \Cref{sec:staticpolicy}). We show a major disadvantage of the static policy is that when directly applied to the data-driven setting with estimation erros ($W_T>0$), it could incur linear regret even when $W_T$ is sublinear; and also the Wasserstein distance is in general tighter than the total variation distance used therein.

For the uninformative setting, we assume no prior knowledge on the true distribution. This setting is consistent with the unknown distribution setting in the literature of OLP problem \citep{molinaro2013geometry,agrawal2014dynamic,gupta2014experts} and the setting of blind NRM \citep{besbes2012blind,jasin2015performance}. We modify the WBDB by replacing the prior estimate of each distribution with their uniform mixture distribution to propose a new measure called \textit{Wasserstein-based non-stationarity measure} (WBNB). By its definition, the WBNB captures the cumulative deviation for all the distribution $\mathcal{P}_t$ from their centric distribution, and it thus reflects the intensity of the non-stationarity associated with $\mathcal{P}_t$'s. In this sense, the WBNB concerns the global change of the distributions, whereas the previous non-stationarity measures \citep{besbes2014stochastic,besbes2015non,cheung2019non} in an unconstrained setting characterize the local and temporal change of the distributions over time. In Section \ref{WBNonSta}, we illustrate by a simple example that such temporal change measures actually fail in a constrained setting. Thus it addresses the necessity of such a global measure and reveals the interaction between the constraints and the non-stationary environment. Note that a simultaneous and independent work \citep{balseiro2020best} also uses global change of the distributions to derive a measure of non-stationarity. However, their measure is based on the total variation metric between distributions. With the same example, we illustrate the advantage of using Wasserstein distance instead of total variation distance or KL-divergence. Specifically, the Wasserstein distance compares both the support and densities between two distributions, while total variation distance or KL-divergence compares only the densities. Therefore the Wasserstein-based measure is sharper and more proper for the general online stochastic optimization problem. We formulate the uncertainty set accordingly with the WBNB and propose \textit{Uninformative Gradient Descent} Algorithm (UGD) as a natural reduction of the IGD algorithm in the uninformative setting. We prove that UGD algorithm achieves a regret bound of optimal order.

As a probability distance metric, the Wasserstein distance has been widely used as a measure of the deviation between estimate and true distribution in the distributionally robust optimization literature (e.g. \cite{esfahani2018data}) to represent confidence set and it has demonstrated good performance both theoretically and empirically. To the best of our knowledge, we are the first to use the Wasserstein distance in an online optimization/learning context. From a modeling perspective, the two proposed measures WBDB and WBNB contribute to the study of non-stationary environment for online optimization/learning problem. Specifically, the data-driven setting relaxes the common assumption adopted in the NRM literature that the true distributions are known to the decision maker by allowing the prior estimates to deviate from the true distributions. This deviation can be interpreted as an estimation or model misspecification error, and WBDB establishes a connection between the deviation and algorithmic performance. The uninformative setting generalizes a stream of online learning literature (e.g. \cite{besbes2015non}), which mainly concerned with the unconstrained settings and includes bandits problem \citep{garivier2008upper, besbes2014stochastic} and reinforcement learning problem \citep{cheung2019non, lecarpentier2019non} as special cases. WBDB adds to the current dictionary of non-stationarity definitions and it specializes for a characterization of the constrained setting.

Finally, we conclude our discussion with several extensions. For the majority of this paper, we analyze algorithm performance under the above two settings without imposing additional structures on the random functions $f_t$ and $\bm{g}_t$'s. When these functions have a stronger structure such as taking a finite support, we show that a better regret bound can be obtained through a re-solving design. The result complements to the existing works on re-solving algorithms \citep{jasin2012re, bumpensanti2020re, vera2020bayesian} in that it provides a robustness analysis of the re-solving algorithm when the underlying distribution is misspecified. Also, throughout the paper, we provide a number of lower bound results. These lower bound results illustrate the following points: (i) our proposed algorithm has an optimal order of regret in the worst-case sense; (ii) the Wasserstein's distance is sharp; (iii) a dynamic algorithm is necessary to achieve sublinear regret in a non-stationary environment.


\subsection{Other Related Literature}\label{sec:literaturereview}

As mentioned above, our formulation of the online stochastic optimization problem roots in two major applications: the online linear programming problem and the network revenue management problem.  The online linear programming (OLP) problem \citep{molinaro2013geometry, agrawal2014dynamic, gupta2014experts} covers a wide range of applications through different ways of specifying the underlying LP, including secretary problem \citep{ferguson1989solved}, online knapsack problem \citep{arlotto2020logarithmic, Jiang2020OnlineRA}, resource allocation problem \citep{vanderbei2015linear, asadpour2020online}, quantity-based network revenue management problem \citep{jasin2012re, jasin2015performance}, network routing problem \citep{buchbinder2009online}, matching problem \citep{mehta2005adwords}, etc. Notably, the problem has been studied under either (i) the stochastic input model where the coefficient in the objective function, together with the corresponding column in the constraint matrix is drawn from an unknown distribution $\mathcal{P}$, or (ii) the random permutation model where they arrive in a random order. As noted in the paper \citep{li2020simple}, the random permutation model exhibits similar concentration behavior as the stochastic input model. The non-stationary setting of our paper relaxes the i.i.d. structure and it can be viewed as a third paradigm for analyzing the OLP problem.

The network revenue management (NRM) problem has been extensively studied in the literature and a main focus is to propose near-optimal policies with strong theoretical guarantees. One popular way is to construct a deterministic linear program as an upper bound of the optimal revenue and use its optimal solution to derive heuristic policies. Specifically, \cite{talluri1998analysis} propose a static bid-price policy based on the dual variable of the linear programming upper bound and proves that the revenue loss is $O(\sqrt{k})$ when each period is repeated $k$ times and the capacities are scaled by $k$. Subsequently, \cite{reiman2008asymptotically} show that by re-solving the linear programming upper bound once, one can obtain an $o(\sqrt{k})$ upper bound on the revenue loss. Then, \cite{jasin2012re} show that under a  non-degeneracy condition for the underlying LP, a policy which re-solves the linear programming upper bound at each time period will lead to an $O(1)$ revenue loss, which is independent of the scaling factor $k$. The relationship between the performances of the control policies and the number of times of re-solving the linear programming upper bound is further discussed in their later paper \citep{jasin2013analysis}. Recently, \cite{bumpensanti2020re} propose an infrequent re-solving policy and show that their policy achieves an $O(1)$ upper bound of the revenue loss even without the ``non-degeneracy'' assumption. With a different approach, \cite{vera2020bayesian} prove the same $O(1)$ upper bound for the NRM problem and their approach is further generalized in \citep{vera2019online, fruend2020a, fruend2020b} for other online decision making problems, including online stochastic knapsack, online probing, bin packing, and dynamic pricing. Note that all the approaches mentioned above are mainly developed for the stochastic/stationary setting. When the arrival process of customers is non-stationary over time, \cite{adelman2007dynamic} develops a strong heuristic based on a novel approximate dynamic programming (DP) approach. This approach is further investigated under various settings in the literature (for example \citep{zhang2009approximate, kunnumkal2016piecewise}). Remarkably, although the approximate DP heuristic is one of the strongest heuristics in practice, it does not feature for a theoretical bound. Finally, by using non-linear basis functions to approximate the value of the DP, \cite{ma2020approximation} develop a novel approximate DP policy and derive a constant competitiveness ratio dependent on the problem parameters. Compared to this line of works, our contribution is two-fold. First, the two main settings that we consider generalize the existing framework of network revenue management in two aspects: (i) we do not assume the knowledge of the underlying distribution; (ii) we do not require the distribution to take a finite support. Second, with the finite-support condition on the underlying distribution, we characterize the relationship between the algorithm performance and the misspecification error of the underlying distribution.

Besides, our problem is also related to the literature of Online Convex Optimization (OCO) and model predictive control. We will discuss the model predictive control literature here and leave the discussion on OCO after presenting our formulation in the next secton. Specifically, our problem can also be formulated as a nonlinear model predictive control (MPC) problem \citep{rawlings2000tutorial}. In MPC, the decision maker executes the online decisions based on some benchmarks. In our paper, we adopt the offline optimum as the benchmark and  develop online algorithms based on the offline optimum. Specifically, for the data-driven setting, we assume that the prior estimates deviate from the true distribution, which corresponds to the broad literature on \textit{robust} MPC \citep{bemporad1999robust}, where there are model uncertainties and noises. Compared t this line of works, our paper is the first to propose the use of Wasserstein distance to measure the model uncertainty, and we derive tight regret bound that relates the best achievable algorithm performance with the model uncertainty. Our algorithm also extends to the uninformative setting where there are no prior estimates, and the information about the underlying benchmark has to be learned on-the-fly. In this way, our algorithm for the uninformative setting establishes a connection between online learning and MPC for a time-varying (non-stationary) environment. We refer interested readers to \cite{wagener2019online} for a more detailed discussion over the connections between online learning and MPC.

\section{Problem Formulation}

Consider the following convex optimization problem
\begin{align}
 \tag{CP}  \max \ \ &  \sum_{t=1}^T f_t(\bm{x}_t)  \label{CP} \\
    \text{s.t. }\ & \sum_{t=1}^T g_{it}(\bm{x}_t) \le c_i, \ \ i=1,...,m, \nonumber  \\
    & \bm{x}_t \in  \mathcal{X}, \ \ t=1,...,T, \nonumber
\end{align}
where the decision variables are $\bm{x}_t\in \mathcal{X}$ for $t=1,...,T$ and $\mathcal{X}$ is a compact convex set in $\mathbb{R}^k$. The function $f_{t}$'s are functions in the space $\mathcal{F} = \mathcal{F}(\mathcal{X})$ of concave continuous functions and $g_{it}$'s are functions in the space $\mathcal{G} = \mathcal{G}(\mathcal{X})$ of convex continuous functions, both of which are supported on $\mathcal{X}.$ Compactly, we define the vector-value function $\bm{g}_t(\bm{x}) = (g_{1t}(\bm{x}),...,g_{mt}(\bm{x}))^\top: \mathbb{R}^k \rightarrow \mathbb{R}^m$. Throughout the paper, we use $i$ to index the constraint and $t$ (or sometimes $j$) to index the decision variables, and we use bold symbols to denote vectors/matrices and normal symbols to denote scalars.

In this paper, we study the \textit{online stochastic optimization} problem where the functions in the optimization problem \eqref{CP} are revealed in an online fashion and one needs to determine the value of decision variables sequentially. Specifically, at each time $t,$  the functions $(f_t, \bm{g}_t)$ are first revealed, and then the decision maker needs to decide the value of $\bm{x}_t$. Different from the offline setting, at each time $t$, we do not have the information of the future part of the optimization problem (from time $t+1$ to $T$). Given the history $\mathcal{H}_{t-1} = \{f_{j}, \bm{g}_{j}, \bm{x}_j\}_{j=1}^{t-1}$, the decision of $\bm{x}_t$ can be expressed as a policy function of the history and the observation at the current time period. That is,
\begin{equation}
    \bm{x}_t = \pi_t(f_{t}, \bm{g}_{t}, \mathcal{H}_{t-1})
    \label{policy}
\end{equation}
where the policy function $\pi_t$ can be time-dependent. We denote the policy $\bm{\pi} = (\pi_{1},...,\pi_{T}).$ The decision variables $\bm{x}_t$'s must conform to the constraints in \eqref{CP} throughout the procedure, and the objective is aligned with the maximization objective of the offline problem \eqref{CP}.

\subsection{Parameterized Form, Probability Space, and Assumptions}

Consider a parametric form of the underlying problem \eqref{CP} where the functions $(f_t,\bm{g}_t)$ are parameterized by a (random) vector $\bm{\theta}_t\in \Theta \subset \mathbb{R}^l$. Specifically,
$$f_{t}(\bm{x}_t) \coloneqq f(\bm{x}_t; \bm{\theta}_t),\ \  {g}_{it}(\bm{x}_t; \bm{\theta}_t) \coloneqq {g}_i(\bm{x}_t; \bm{\theta}_t)$$
for each $i=1,...,m$ and $t=1,...,T$. We denote the vector-valued constraint function by $\bm{g}(\bm{x};\bm{\theta}) = (g_{1}(\bm{x};\bm{\theta}),...,g_{m}(\bm{x};\bm{\theta}))^\top: \mathcal{X} \rightarrow \mathbb{R}^m$. Then the problem \eqref{CP} can be rewritten as the following parameterized convex program
\begin{align}
 \tag{PCP}  \max \ \ &  \sum_{t=1}^T f(\bm{x}_t;\bm{\theta}_t)  \label{PCP} \\
    \text{s.t. }\ & \sum_{t=1}^T g_{i}(\bm{x}_t;\bm{\theta}_t) \le c_i, \ \ i=1,...,m, \nonumber  \\
    & \bm{x}_t \in  \mathcal{X}, \ \ t=1,...,T, \nonumber
\end{align}
where the decision variables are $(\bm{x}_1,....,\bm{x}_T).$ We note that this parametric form \eqref{PCP} avoids the complication of dealing with probability measure in function space. It is introduced mainly for presentation purpose, and it will change the nature of the problem. Moreover, we assume the knowledge of $f$ and $\bm{g}$ a priori. Here and hereafter, we will use \eqref{PCP} as the underlying form of the online stochastic optimization problem.

The problem of online stochastic optimization, as its name refers, involves stochasticity on the functions for the underlying optimization problem. The parametric form \eqref{PCP} reduces the randomness from the function to the parameters $\bm{\theta}_t$'s, and therefore the probability measure can be defined in the parameter space of $\Theta$. First, we consider the following distance function between two parameters $\bm{\theta}, \bm{\theta}'\in\Theta$,
\begin{equation}
    \rho(\bm{\theta}, \bm{\theta}') \coloneqq \sup_{\bm{x}\in \mathcal{X}} \|(f(\bm{x}; \bm{\theta}),\bm{g}(\bm{x}; \bm{\theta}))-(f(\bm{x}; \bm{\theta}'),\bm{g}(\bm{x}; \bm{\theta}'))\|_\infty \label{rho_dist}
\end{equation}
where $\|\cdot\|_\infty$ is the supremum norm in $\mathbb{R}^{m+1}.$ Without loss of generality, let $\Theta$ be a set of class representatives, that is, for any $\bm{\theta}\neq \bm{\theta}'\in\Theta$, $\rho(\bm{\theta}, \bm{\theta}')>0.$ In this way, the parameter space $\Theta$ can be viewed as a metric space equipped with metric $\rho(\cdot, \cdot).$ {We choose supremum norm because in our model, one \textit{single} action $\bm{x}_t$ is made at each period $t$ and it can take any arbitrary value in $\mathcal{X}$. As a result, the comparison between $(f(\cdot;\bm{\theta}), \bm{g}(\cdot;\bm{\theta}))$ and $(f(\cdot;\bm{\theta}'), \bm{g}(\cdot;\bm{\theta}'))$ should be made at each point $\bm{x}$ and thus the supremum norm is a natural choice for defining $\rho(\bm{\theta}, \bm{\theta}')$.} In this way, the definition of $\rho$ is based on the vector-valued function $(f,\bm{g}): \mathcal{X} \rightarrow \mathbb{R}^{m+1}$. Thus it captures the effect of different parameters on the function value rather than the original Euclidean difference in the parameter space. Let $\mathcal{B}_\Theta$ be the smallest $\sigma$-algebra in $\Theta$ that contains all open subsets (under metric $\rho$) of $\Theta.$ We denote the distribution of $\bm{\theta}_t$ as $\mathcal{P}_t$ which can be viewed as a probability measure on $(\Theta, \mathcal{B}_\Theta).$

Throughout the paper, we make the following assumptions. Assumption \ref{assume} (a) and (b) impose boundedness on function $f$ and $g_i$'s. Assumption \ref{assume} (c) states the ratio between $f$ and $g_i$ is uniformly bounded by $q$ for all $\bm{x}$ and $\bm{\theta}$. Intuitively, it tells that for each unit consumption of resource, the maximum amount of revenue earned is upper bounded by $q$. In this paper, this condition will mainly be used to give an upper bound on the dual optimal solution. In Assumption \ref{assume} (d), we assume $\mathcal{P}_t$'s are independent of each other but we do not assume the exact knowledge of them. Also, there can be dependence between components in the vector-value functions $(f, \bm{g}).$ In Assumption \ref{assume} (e), we require some convexity structure for the underlying functions. In the rest of the paper, this assumption will only be used to ensure that the Lagrangian problem $\max_{\bm{x}\in\mathcal{X}}\{f(\bm{x};\bm{\theta})-\bm{p}^\top g(\bm{x};\bm{\theta})\}$ can be efficiently solved for any fixed $\bm{p}\ge \bm{0}.$

\begin{assumption}[Boundedness and Independence] We assume
\begin{itemize}
    \item[(a)] $|f(\bm{x}; \bm{\theta})| \le 1$ for all $\bm{x} \in\mathcal{X}, \bm{\theta}\in \Theta$.
    \item[(b)] $g_i(\bm{x}; \bm{\theta})\in[0,1]$ for all $\bm{x} \in\mathcal{X}, \bm{\theta}\in \Theta$ and $i =1,...,m.$ In particular, $\bm{0}\in \mathcal{X}$ and $g_i(\bm{0}; \bm{\theta})=0$ for all $\bm{\theta}\in \Theta.$
    \item[(c)] There exists a positive constant $q$ such that for any $\bm{\theta}\in\Theta$ and each $i$, we have that $f(\bm{x};\bm{\theta})\leq q\cdot g_i(\bm{x};\bm{\theta})$ holds for any $\bm{x}\in\mathcal{X}$ when $g_i(\bm{x};\bm{\theta})>0$.
    \item[(d)] $\bm{\theta}_t\sim \mathcal{P}_t$ and $\mathcal{P}_t$'s are independent with each others.
    \item[(e)] {The function $f(\bm{x};\bm{\theta})$ is concave over $\bm{x}$ and the function $g_i(\bm{x}; \bm{\theta})$ is convex over $\bm{x}$ for any $\bm{\theta}\in\Theta$ and $i=1,\dots,m$.}
\end{itemize}
\label{assume}
\end{assumption}

In the following, we illustrate the online formulation through two main application contexts: online linear programming and online network revenue management. We choose the more general convex formulation (\ref{PCP}) to uncover the key mathematical structures for this online optimization problem, but we will occasionally return to these two examples to generate intuitions throughout the paper.

\subsection{Examples}\label{sec:example}

\textbf{Online linear programming (LP)}: The online LP problem \citep{molinaro2013geometry, agrawal2014dynamic, gupta2014experts} can be viewed as an example of the online stochastic optimization formulation of \eqref{PCP}. Specifically, the decision variable $x_t\in \mathcal{X}=[0,1]$, the functions $f$ and $g_{i}$ are linear functions, and the parameter $\bm{\theta}_t=(r_t, \bm{a}_t)$ where $\bm{a}_t = (a_{1t},...,a_{mt})$. That is, $f(x_t; \bm{\theta}_t)=r_tx_t$ and $g_{i}(x_t; \bm{\theta}_t)=a_{it}x_t.$ At each time $t=1,...,T$, the coefficient in the objective $r_t$ together with the corresponding column $\bm{a}_t$ in the constraint matrix is revealed, and then one needs to determine the value of $x_t$ immediately. As mentioned earlier, the online LP problem covers a wide range of applications, through different specifications of $\bm{\theta}_t$, including secretary problem, online knapsack problem, resource allocation problem, generalized assignment problem, network routing problem, and matching problem.

\noindent \textbf{Price-based network revenue management (NRM)}: In the price-based NRM problem \citep{gallego1994optimal}, a seller is selling a given stock of products over a finite time horizon by posting a price at each time. The demand is price-sensitive and the firm's objective is to maximize the total collected revenue. This problem could be cast in the formulation \eqref{PCP} as follows. The parameter $\theta_t$ refers to the type of the $t$-th arriving customer. There is usually a finite number of customer types, so the parameter set $\Theta$ is finite. The type of each customer arrival can be based on the side information such as demographic features and purchasing history. Each different customer type specifies a different demand function between the posted price and the realized demand/revenue. Specifically, the decision variable $\bm{x}_t$ represents to the price posted by the decision maker at time $t$. The constraint function $\bm{g}(\bm{x}_t; {\theta}_t)$ denotes the resource consumption (demand) under the price $\bm{x}_t$ and $f(\bm{x}_t; {\theta}_t)$ denotes the collected revenue. In such context, there is usually an extra layer of randomness for the functions $f(\bm{x}_t; {\theta}_t)$ and $\bm{g}(\bm{x}_t; {\theta}_t)$ given the parameter $\theta_t$, different from our main setting where $f$ and $\bm{g}$ are deterministic and known. We will discuss this extension in \Cref{sec:stochastic}.

\noindent \textbf{Choice-based network revenue management}: In the choice-based NRM problem \citep{talluri2004revenue}, the seller offers an assortment of the products to the customer arriving in each time period, and the customer chooses a product from the assortment to purchase according to a given choice model. Now we discuss how our formulation \eqref{PCP} covers the choice-based NRM problem as a special case. As in the price-based NRM model, the parameter $\theta_t$ represents the customer type at time $t$, and the parameter set $\Theta$ is finite. Then we denote by $\mathcal{S}$ the set of all possible assortments the seller can offer to the customers. At each time $t$, the decision variable encodes the assortment decision, i.e., $\bm{x}_t=(x_{t,1},\dots,x_{t,|\mathcal{S}|})\in[0,1]^{|\mathcal{S}|}$ and $\sum_{s\in\mathcal{S}}x_{t,s}=1$. Here, $x_{t,s}$ denotes the probability that an assortment $s\in\mathcal{S}$ is offered at period $t$. The parameter $\theta_t$ specifies the choice model of the $t$-th customer arrival, and, together with the assortment decision $\bm{x}_t$, determines the revenue through the function $f(\bm{x}_t;{\theta}_t)$ and the resource consumption (demand) through the function $g(\bm{x}_t;\theta_t)$. As in the price-based NRM problem, there is an extra layer of randomness for the functions $f(\bm{x}_t;\bm{\theta}_t)$ and $\bm{g}(\bm{x}_t;\bm{\theta}_t)$ given the parameter $\theta_t$ and the assortment $\bm{x}_t$. We will also elaborate more in \Cref{sec:stochastic}.

\subsection{Performance Measure}

We denote the offline optimal solution of optimization problem (\ref{CP}) as $\bm{x}^* = (\bm{x}_1^*,...,\bm{x}_T^*)$, and the offline (online) objective value as $R_T^*$ (${R}_T$). Specifically,
\begin{align*}
R_T^* & \coloneqq  \sum_{t=1}^T f_t(\bm{x}_t^*)\\
R_T(\bm{\pi}) & \coloneqq  \sum_{t=1}^T f_t(\bm{x}_t).
\end{align*}
where the online objective value depends on the policy $\bm{\pi}$.  Aligned with general online learning/optimization problems, we focus on minimizing the gap between the online and offline objective values. Specifically, the \textit{optimality gap} is defined as follows:
$$\text{Reg}_T(\mathcal{H}, \bm{\pi}) \coloneqq  R_T^*-{R}_T(\bm{\pi})$$
where the problem profile $\mathcal{H}$ encapsulates the realization of the random parameters, i.e., $\mathcal{H}\coloneqq (\bm{\theta}_1,...,\bm{\theta}_T).$ Note that $R_T^*$, $R_T(\pi)$, $\bm{x}_t^*$ and $\bm{x}_t$ are all dependent on the problem profile $\mathcal{H}$, but we omit $\mathcal{H}$ in these terms for notation simplicity when there is no ambiguity. We define the performance measure of the online stochastic optimization problem formally as \textit{regret}
\begin{equation}
   \text{Reg}_T(\bm{\pi}) \coloneqq \max_{\mathcal{P} \in \Xi}\  \E_{\mathcal{H}\sim\mathcal{P}}[\text{Reg}_T(\mathcal{H}, \bm{\pi})] \label{regDef}
\end{equation}
where $\mathcal{P} =  (\mathcal{P}_1,...,\mathcal{P}_T)$ denotes the probability measure of all time periods and the expectation is taken with respect to the parameter $\bm{\theta}_t \sim \mathcal{P}_t$; compactly, we write the problem profile $\mathcal{H} \sim \mathcal{P}$. We consider the worst-case regret for all the distribution $\mathcal{P}$ in a certain set $\Xi$ where the set $\Xi$ will be specified in later sections.

The specification of the set $\Xi$ imposes more structure on the distributions of $(\mathcal{P}_1,...,\mathcal{P}_T)$ and this is one of the main themes of our paper. In the canonical setting of online stochastic learning problem, all the distributions $\mathcal{P}_t$'s are the same, i.e., $\mathcal{P}_t= \mathcal{P}_0$ for $t=1,...,T.$ Meanwhile, the adversarial setting of online learning problem refers to the case when $\mathcal{P}_t$'s are adversarially chosen. Our work aims to bridge these two ends of the spectrum with a novel notion of non-stationarity, and to relate the algorithm performance with certain structural property of $\mathcal{P} =  (\mathcal{P}_1,...,\mathcal{P}_T)$. In the same spirit, the work on non-stationary stochastic optimization \citep{besbes2015non} proposes an elegant notion of non-stationarity called variation budget. Subsequent works consider similar notions in the settings of bandits \citep{besbes2014stochastic, russac2019weighted} and reinforcement learning \citep{cheung2019non}. To the best of our knowledge, all the previous works along this line consider unconstrained setting and thus our work contributes to this line of work in illustrating how the constraints interact with the non-stationarity.

\textbf{Discussion on the (constrained) online learning/optimization literature.}

Now we discuss the positioning of our work against the vast literature on the (constrained) online learning/optimization problem. Generally speaking, the formulations on this topic fall into two categories: (i) first-observe-then-decide and (ii) first-decide-then-observe. Our formulation belongs to the first category in that at each time $t$, the decision maker first observes the parameter $\bm{\theta}_t$ (and hence the functions $(f(\bm{x};\bm{\theta}_t), \bm{g}(\bm{x};\bm{\theta}_t))$), and then determines the value of $\bm{x}_t$. This formulation covers many applications in operations research and operations management, such as online LP and NRM. In these application contexts, the observations represent customers/orders arriving sequentially to the system, and the decision variables capture accordingly the acceptance/rejection/pricing decisions of the customers. Technically, the nature of having the observation before making the decision enables the stronger dynamic benchmark that allows a different $\bm{x}_t$ across different time periods as the definition of $R_T^*$ in above.

One representative problem for the second category is the online convex optimization (OCO) problem. The OCO problem can be viewed as a first-decide-then-observe problem in that at each time $t$, the decision maker first chooses the decision variable $\bm{x}_t$ and then observes the function $f_t$ (incurring a loss of $f_t(\bm{x}_t)$). The OCO problem is mainly motivated from machine learning applications such as online linear regression or online support vector machine \citep{hazan2016introduction}. From an information perspective, the OCO problem can be viewed as a partial information setting, whereas our online stochastic optimization can be viewed as a full information setting \citep{lattimore2020bandit}. Accordingly, the standard OCO problem generally adopts the (weaker) static benchmark, i.e., $\bm{x}_t^*$'s need to be the same when defining $R_T^*$. This discrepancy in regret benchmark prevents direct applications of OCO algorithms and analyses to our context.
There are results that consider a dynamic or partially dynamic regret benchmark for OCO in a non-stationary environment \citep{besbes2015non} or an adversarial environment \citep{hall2013dynamical, jadbabaie2015online}, but all these works consider the unconstrained problem. A line of works study the problem of online convex optimization with constraints (OCOwC) under a static or stochastic generation of the constraint functions $\bm{g}_t$'s. Specifically, \cite{jenatton2016adaptive,yuan2018online,yi2021regret} all consider the OCOwC problem with static constraint ($\bm{g}_t=\bm{g}$ for some $\bm{g}$), while \cite{neely2017online} mainly study a setting where $\bm{g}_t$ is i.i.d. generated from some distribution. To the best of our knowledge, no existing work along this line of literature allows non-stationarily generated constraint functions.

Another representative first-decide-then-observe problem is the bandits with knapsacks (BwK) problem which can also be viewed as a constrained online learning problem. The existing BwK results consider either a stochastic setting \citep{badanidiyuru2013bandits, agrawal2014bandits} or an adversarial setting \citep{rangi2018unifying, immorlica2019adversarial}. The algorithms along this line of literature are mainly based on the underlying primal and dual LPs. Specifically, \cite{agrawal2014fast} develop a fast dual-based algorithm for the online stochastic optimization problem (first-observe-then-decide) and analyzes the algorithm performance under further stronger conditions on the dual optimal solution.
\cite{agrawal2014bandits} then extend to the first-decide-then-observe problem of BwK. The idea has been recently further applied to online learning in revenue management problems in \citep{miao2021general}. But these learning models and algorithms are developed in the stationary (stochastic) environment, which cannot be applied to the non-stationary setting.

We remark that both the first-decide-then-observe and the first-observe-then-decide frameworks can be useful in modeling some application context. For example, an Adwords problem under pay for conversions can be modeled by OCOwO or BwK problems, while an Adwords problem under pay for impressions is usually modeled by our online stochastic optimization framework or OLP problem \citep{mehta2005adwords}.


\section{Known Distribution and Informative Gradient Descent}

\label{known_distr}

We begin our discussion with the case when the distributions $\mathcal{P}_t$'s are all known a priori. We use this case to motivate and present our prototypical algorithm -- \textit{informative gradient descent} which incorporates the prior information of $\mathcal{P}_t$'s with the online gradient descent algorithm. In the following sections, we will discuss the case when the distributions $\mathcal{P}_t$'s are not known precisely and analyze the algorithm performance accordingly.

\subsection{Deterministic Upper Bound and Dual Problem}
\label{firstAlgo}

We first introduce the standard deterministic upper bound for the regret benchmark -- the ``offline'' optimum $\mathbb{E}[R^*_T]$. We define the following expectation for a function $u(\bm{x};\bm{\theta}):\mathcal{X}\rightarrow \mathbb{R}$ and a probability measure $\mathcal{P}$ in the parameter space $\Theta,$
\[
\mathcal{P}u(\bm{x}(\bm{\theta});\bm{\theta}) \coloneqq \int_{\bm{\theta}'\in\Theta} u(\bm{x}(\bm{\theta}');\bm{\theta}')d\mathcal{P}(\bm{\theta}')
\]
where $\bm{x}(\bm{\theta}): \Theta \rightarrow \mathcal{X}$ is a measurable function. Thus $\mathcal{P}u(\cdot)$ can be viewed as a deterministic functional that maps function $\bm{x}(\bm{\theta})$ to a real value and it is obtained by taking expectation with respect to the parameter $\bm{\theta}\sim\mathcal{P}$.

Then, consider the following optimization problem
\begin{equation}\label{PUB}
\begin{aligned}
R_T^{\text{UB}}=&  \max \ \ &&  \sum_{t=1}^T \mathcal{P}_t f(\bm{x}_t(\bm{\theta}_t);\bm{\theta}_t)  \\
    &\text{s.t. }\ && \sum_{t=1}^T \mathcal{P}_t g_i(\bm{x}_t(\bm{\theta}_t);\bm{\theta}_t) \le c_i, \ \ i=1,...,m,  \\
    &&& \bm{x}_t(\bm{\theta}_t): \Theta \rightarrow \mathcal{X} \text{ is a measurable function for }  t=1,...,T.
\end{aligned}
\end{equation}
where $\bm{\theta}_t$ follows the distribution $\mathcal{P}_t.$ The optimization problem \eqref{PUB} can be viewed as a convex relaxation of \eqref{PCP} where the objective and constraints are all replaced with their expected counterparts. Here $\bm{x}_{1:T}=(\bm{x}_1(\bm{\theta}_1),...,\bm{x}_T(\bm{\theta}_T))$ encapsulates all the primal decision variables. The primal variables are expressed in a function form of $\bm{\theta}_t$ in that for each different $\bm{\theta}_t$, we allow a different choice of the primal variables. This is aligned with the ``first-observe-then-decide'' setting where the decision maker first observes the realization of $\bm{\theta}_t\sim \mathcal{P}_t$ and then decide the value of $\bm{x}_t$. As a standard result in literature \citep{gallego1994optimal}, Lemma \ref{upper1} establishes the optimal objective value $R_T^{\text{UB}}$ as an upper bound for $\mathbb{E}[R^*_T]$.

\begin{lemma}\label{upper1}
It holds that $R_T^{\text{UB}}\geq\mathbb{E}[R^*_T]$.
\end{lemma}

The deterministic upper bound and the optimization problem \eqref{PUB} are commonly used to design algorithms in the literature. To proceed, we introduce the Lagrangian of \eqref{PUB},
\[
L(\bm{p},\bm{x}_{1:T})\coloneqq\bm{c}^\top \bm{p}+ \sum_{t=1}^T \mathcal{P}_t\left( f(\bm{x}_t(\bm{\theta}_t);\bm{\theta}_t) - \bm{p}^\top \bm{g}(\bm{x}_t(\bm{\theta}_t);\bm{\theta}_t)\right)
\]
where $\bm{\theta}_t$ follows the distribution $\mathcal{P}_t.$
The (Lagrange multipliers) vector $\bm{p}=(p_1,...,p_m)^\top$ conveys a meaning of shadow price for each budget, and $p_i\ge0$ is the multiplier/dual variable associated with the $i$-th constraint. Furthermore, we define the following function based on a point-wise optimization for the primal variables,
$$h(\bm{p};\bm{\theta}) \coloneqq \max_{\bm{x}(\bm{\theta})\in\mathcal{X}}\left\{f(\bm{x}(\bm{\theta}); \bm{\theta})-\bm{p}^\top \bm{g}(\bm{x}(\bm{\theta});\bm{\theta}) \right\}.$$
Here the point-wise optimization emphasizes that the primal variables can be dependent on (and as a measurable function of) the parameter $\bm{\theta}_t$. For example, the pricing and assortment decisions can be made upon the observation of the customer type. Then the dual problem of \eqref{PUB} becomes
\[
\min_{\bm{p}\ge \bm{0}} L(\bm{p}):=\bm{c}^\top\bm{p}+\sum_{t=1}^{T}\mathcal{P}_th(\bm{p};\bm{\theta}_t)
\]
where $\bm{\theta}_t$ follows the distribution $\mathcal{P}_t$ as before.

Let $\bm{p}^*$ denote an optimal dual solution, i.e.,
\begin{equation}
\bm{p}^* \in\text{argmin}_{\bm{p}\geq\bm{0}}L(\bm{p})\label{new p_star}
\end{equation}
and for each $t$,
\begin{equation}\label{new2007}
\bm{\gamma}_t\coloneqq \mathcal{P}_t\bm{g}(\bm{x}^*(\bm{\theta});\bm{\theta})  \text{~~where~} \bm{x}^*(\bm{\theta})=\text{argmax}_{\bm{x}\in\mathcal{X}}\{f(\bm{x};\bm{\theta})-(\bm{p}^*)^\top\cdot\bm{g}(\bm{x};\bm{\theta})\}.
\end{equation}
Here, $\bm{x}^*(\bm{\theta})$ is the associated primal (optimal) solution under the dual optimal solution $\bm{p}^*$, and $\bm{\gamma}_t$ can be interpreted as the expected budget consumption in the $t$-th time period under the optimal primal-dual pair $(\bm{x}^*(\bm{\theta}), \bm{p}^*)$. Accordingly, for each $t$, we define
\begin{equation*}
L_t(\bm{p})\coloneqq \bm{\gamma}_t^\top\bm{p}+\mathcal{P}_t h(\bm{p};\bm{\theta}).
\end{equation*}
The following proposition states the relation between $L(\cdot)$ and $L_t(\cdot)$ and establishes an upper bound for the benchmark using the dual problem.

\begin{proposition}\label{uppernew2}
For each $t=1,...,T$, it holds that
\begin{equation*}
\bm{p}^* \in \mathrm{argmin}_{\bm{p}\geq0}L_t(\bm{p})
\end{equation*}
where $\bm{p}^*$ is defined in \eqref{new p_star} as one dual optimal solution. Moreover, we have
\begin{equation*}
\mathbb{E}[R^*_T]\leq\min_{\bm{p}\geq0}\sum_{t=1}^{T}L_t(\bm{p})=\sum_{t=1}^{T}\min_{\bm{p}_t\geq0}L_t(\bm{p}_t).
\end{equation*}
\end{proposition}

The first part of the proposition says that all the $L_t$'s share the same minimizer of $\bm{p}^*$ as the dual problem. In fact, this is the reason why we introduce $\bm{\gamma}_t$'s to define $L_t$'s and the key how the prior knowledge of non-stationarity can be used for algorithm design. This property of a shared minimizer can be useful for online optimization in a non-stationary environment. First, we note that the primal optimal solution can be largely determined by the dual optimal solution $\bm{p}^*.$ At each time $t$, the decision maker only has observations of the (random) realizations of $L_s$ for $s=1,...,t$. Intuitively, the proposition implies that these past observations can be effectively used to estimate the dual optimal solution $\bm{p}^*$ as all the $L_t$'s share the same optimal solution $\bm{p}^*$. Another advantage of this dual-based representation is that, through the adjustment of $\bm{\gamma}_t$'s, the dual optimal solution (of $L_t(\cdot)$) for each time period is the same, whereas the primal optimal solution $\bm{x}_t$ to \eqref{PCP} may be different from each $t$. As mentioned earlier, there are two types of regret benchmarks in the literature of online optimization: static oracle and dynamic oracle. The static oracle refers to the case where we compare to an offline decision maker adopting a common optimal solution throughout the entire horizon, and the dynamic oracle allows the offline decision maker to take an individual optimal solution for each time period. Proposition \ref{uppernew2} states that in our setting, the static oracle and the dynamic oracle connect with each other in the dual space through a careful construction of $L_t(\cdot)$: the primal optimal solution is dynamic (different over time) while the dual optimal solution is static after the adjustment of $\bm{\gamma}_t$'s. This connection makes it possible to apply the gradient descent-based algorithms from OCO literature for the dual space in the non-stationary setting. \cite{besbes2015non} derive a similar argument to connect the dynamic oracle and static oracle for the unconstrained setting as a backbone for the algorithm design and regret analysis therein.

\subsection{Main Algorithm and Regret Analysis}

\label{sec_IGD_algo}

Now we present our main algorithm -- \textit{Informative Gradient Descent} -- fully described as Algorithm \ref{alg:SOA}. {The algorithm is described as a meta algorithm with an input $\bm{\gamma}$. In the following sections, we will discuss how to apply the algorithm to different settings with different specifications of $\bm{\gamma}$.} When the distributions $\mathcal{P}_t$'s are known, the algorithm is motivated from the dual-based representation in Proposition \ref{uppernew2} and the input $\bm{\gamma}$ is accordingly defined by \eqref{new2007}. Specifically, the algorithm maintains a dual vector/price $\bm{p}_t$, and at each time $t$, it performs a stochastic gradient descent update for $\bm{p}_t$ with respect to the function $L_t(\cdot)$. To see that the expectation of the dual gradient update \eqref{003} is the gradient with respect to the function $L_t(\cdot)$ evaluated at $\bm{p}_t$,
\begin{align*}
    \E\left[\bm{g}(\tilde{\bm{x}}_t;\bm{\theta}_t) -\bm{\gamma}_t\right] & = -\bm{\gamma}_t +\mathcal{P}_t\bm{g}(\tilde{\bm{x}}_t;\bm{\theta}) \\
    & = -\frac{\partial}{\partial \bm{p}}\left(\bm{\gamma}_t^\top \bm{p}+\mathcal{P}_th(\bm{p};\bm{\theta})\right)\Bigg\vert_{\bm{p}=\bm{p}_t}
\end{align*}
where the first line comes from taking expectation with respect to $\bm{\theta}_t$ and the second line comes from the definition of $\tilde{\bm{x}}_t$ in the algorithm. At each time $t$, the primal decision variable $\bm{x}_t$ is determined based on the dual price $\bm{p}_t$ and the observation $\bm{\theta}_t$ jointly, in the same manner as the definition of the function $h(\bm{p};\bm{\theta})$. Assumption \ref{assume} (e) ensures the optimization problem that defines $h(\bm{p};\bm{\theta})$ can be solved efficiently (See further discussion in Section \ref{discussion_h}).

\begin{algorithm}[ht!]
\caption{Informative Gradient Descent Algorithm (IGD$(\bm{\gamma})$)}
\label{alg:SOA}
\begin{algorithmic}[1]
\State Input: parameters $\bm{\gamma}=(\bm{\gamma}_1,\dots,\bm{\gamma}_T)$.
\State Initialize the initial dual price $\bm{p}_1 = \bm{0}$ and initial constraint capacity $\bm{c}_{1}=\bm{c}$
\For {$t=1,..., T$}
\State Observe $\bm{\theta}_t$ and solve
\[
\tilde{\bm{x}}_t=\text{argmax}_{\bm{x}\in\mathcal{X}}\{ f(\bm{x};\bm{\theta}_t)-\bm{p}_t^\top \cdot \bm{g}(\bm{x};\bm{\theta}_t) \}
\]
where $\bm{g}(\bm{x},\bm{\theta}_t)=(g_1(\bm{x},\bm{\theta}_t),...,g_m(\bm{x},\bm{\theta}_t))^\top$
\State Set
\[
\bm{x}_t=\left\{\begin{aligned}
&\tilde{\bm{x}}_t,&&\text{if $\bm{c}_{t}$ permits a consumption of~}\bm{g}(\tilde{\bm{x}}_t;\bm{\theta}_t) \\
&\bm{0},&&\text{otherwise}
\end{aligned}\right.
\]
\State Update the dual price
\begin{equation}\label{003}
\bm{p}_{t+1} =\left(\bm{p}_t + \frac{1}{\sqrt{T}} \left(\bm{g}(\tilde{\bm{x}}_t;\bm{\theta}_t) - \bm{\gamma}_t\right)\right)\vee \bm{0}
\end{equation}
where the element-wise maximum operator $u\vee v = \max \{v,u\}$
\State Update the remaining capacity
\[
\bm{c}_{t+1}=\bm{c}_{t}-\bm{g}(\bm{x}_t;\bm{\theta}_t)
\]
\EndFor
\State Output: $\bm{x} = (\bm{x}_1,...,\bm{x}_T)$
\end{algorithmic}
\end{algorithm}

We now provide an alternative perspective to interpret Algorithm \ref{alg:SOA} for the case when the distributions are known. Note that by definition \eqref{new2007}, $\bm{\gamma}_t$'s represent the ``optimal'' way to allocate the resource budget over time according to the dual optimal solution $\bm{p}^*$. Specifically, a larger (resp. smaller) value of $\gamma_{i,t}$, where $\gamma_{i,t}$ denotes the $i$-th component of $\bm{\gamma}_t$, indicates that more (resp. less) budget should be allocated to time period $t$ for constraint $i$. In Algorithm \ref{alg:SOA}, from the update rule \eqref{003} of the dual variable at time period $t$, we know that if the budget consumption of constraint $i$ is larger (resp. smaller) than $\gamma_{i,t}$, i.e., $g_i(\tilde{\bm{x}}_t;\bm{\theta}_t)>\gamma_{i,t}$ (resp. $g_i(\tilde{\bm{x}}_t;\bm{\theta}_t)<\gamma_{i,t}$), then we have that $p_{i,t+1}\le p_{i,t}$ (resp. $p_{i,t+1}>p_{i,t}$), where $p_{i,t}$ denotes the $i$-th component of $\bm{p}_t$. That is, if more (resp. less) budget is consumed in the earlier periods, then the dual price will be more likely to increase (resp. decrease), and consequently, less (resp. more) budget will be consumed in the future periods. In this way, the dual variable $\bm{p}_t$ dynamically balances the budget consumption: it ensures that for each $t$, the cumulative budget consumption of Algorithm \ref{alg:SOA} during the first $t$ time periods always stay ``close'' to the optimal scheme of $\sum_{j=1}^{t}\bm{\gamma}_{j}$. Later in Section \ref{sec:staticpolicy}, we will show that a dynamic policy that incorporates the resource consumption process into the online decisions is necessary in a non-stationary environment, and any static policy can incur a linear regret even when the underlying non-stationarity intensity is sublinear.

Now we analyze the performance of Algorithm \ref{alg:SOA} for the known distribution setting. The following lemma says that the dual price vector $\bm{p}_t$ remains bounded throughout the procedure. Its proof largely relies on Assumption \ref{assume} (c), and also, the main usage of Assumption \ref{assume} (c) throughout our analysis is to ensure the boundedness of the dual vector.

\begin{lemma}\label{newupperlemma}
Under Assumption \ref{assume}, the dual price vector satisfies $\|\bm{p}_t\|_{\infty}\leq q+1$ for $t=1,2,\dots,T$. Here $\bm{p}_t$ is computed by \eqref{003} of IGD($\bm{\gamma}$) in Algorithm \ref{alg:SOA} with $\bm{\gamma}$ specified by \eqref{new2007}, and the constant $q$ is defined in Assumption \ref{assume} (c).
\end{lemma}
The following theorem builds upon Proposition \ref{uppernew2} and Lemma \ref{newupperlemma} and it states that the regret of Algorithm \ref{alg:SOA} is upper bounded by $O(\sqrt{T})$.
\begin{theorem}
\label{newnonstationtheorem}
Under Assumption \ref{assume}, if we consider the set $\Xi=\{\mathcal{P}:\mathcal{P}=(\mathcal{P}_1,...,\mathcal{P}_T), \forall \mathcal{P}_1,\dots,\forall \mathcal{P}_T\}$, then the regret of IGD($\bm{\gamma}$) in Algorithm \ref{alg:SOA}, with $\bm{\gamma}$ defined by \eqref{new2007}, has the following upper bound
$$\text{Reg}_T(\pi_{\text{IGD}}) \le O(\sqrt{T})$$
where $\pi_{\text{IGD}}$ stands for the policy specified by IGD($\bm{\gamma}$) in Algorithm \ref{alg:SOA}.
\end{theorem}

The $O(\sqrt{T})$ regret upper bound of Algorithm \ref{alg:SOA} under a known non-stationary environment complements to several existing results in the literature. For the NRM problem, \cite{talluri1998analysis} derives a $O(\sqrt{k})$ regret bound when the system size is scaled by $k$ times, i.e., each period in the original problem is split into $k$ statistically independent and identical periods and the capacities are scaled up $k$ times. Following works subsequently improves this bound to $O(1)$ (e.g. \citep{reiman2008asymptotically, jasin2012re, bumpensanti2020re, vera2020bayesian}). However, these methods are developed for a stationary setting under the finite support assumption over the distributions and do not account for the achievability of a sublinear regret under a non-stationary setting where the distribution at each period could be arbitrarily different from each other. Along this line, a recent work \citep{fruend2020a} considers a setting that the length horizon $T$ is random, but the paper still requires the underlying distribution to be finite-support and a ``pseudo-stationary'' environment where a concentration over time holds. The most similar result to Theorem \ref{newnonstationtheorem} is from \citep{devanur2019near} where the authors derives a $1-1/\sqrt{c_{\min}}$ competitive ratio under the non-stationary setting, where $c_{\min}=\min\{c_1,c_2,\dots,c_m\}$ is the minimal budget and they assume that $f_t(\cdot)$ and $\bm{g}_t(\cdot)$ are all linear functions for each $t$. However, the competitive ratio result cannot be translated into a regret bound in our setting since we do not assume any relationship between the horizon length $T$ and the initial budget $\bm{c}$. Specifically, the method therein is based on showing that the arrivals possess a concentration property and applying Chernoff-type inequality to derive high probability bounds on the event that all the constraints are not violated. In contrast, our method is based on using the dual variable $\bm{p}_t$ to balance the budget consumption on every sample path and showing that the dual variable is bounded over time according to our update rule \eqref{003}. A recent work \citep{cheung2020online} also derives an $O(\sqrt{T})$ regret bound for the setting of non-stationary environment with known distribution. The analysis therein builds upon a static dual-based algorithm; we defer more discussions to Section \ref{sec:staticpolicy} and Section \ref{sec_numeric}. In the following sections, we continue our pursuit and investigate how the IGD algorithm rolls out in a non-stationary environment when the true distribution is unknown.

\begin{proposition}\label{prop:infinitesupportlower}
Under Assumption \ref{assume}, there is no algorithm that can achieve a regret better than $\tilde{\Omega}(\sqrt{T})$ where $\tilde{\Omega}$ hides a logarithmic term.
\end{proposition}
Proposition \ref{prop:infinitesupportlower} constructs a problem instance showing that even for a stationary setting where $\mathcal{P}_t$ for each $t$ is identical to each other and known a priori, the lower bound of any online policy is $\tilde{\Omega}(\sqrt{T})$. For the problem instance, the underlying distribution $\mathcal{P}_t$ takes an infinite support. We note that when the parameter distribution has a finite support, an $O(1)$ regret bound can be derived following the approach in \cite{vera2020bayesian, fruend2020a, fruend2020b}, which implies that the gap between $O(\sqrt{T})$ and $O(1)$ is caused by whether the support of the parameter distribution is infinite or finite. In Section \ref{sec_bounded}, we will further exploit the finite support structure and achieve better regret bound. The proof of Proposition \ref{prop:infinitesupportlower} builds upon the analysis of Lemma $1$ in \citet{arlotto2019uniformly}. Different from the existing lower bound examples \citep{arlotto2019uniformly, bumpensanti2020re, balseiro2020best}, the distribution of our problem instance bears no dependence on the horizon $T$; that is, the same static problem instance establishes the lower bound for all $T$. Theorem \ref{newnonstationtheorem} and Proposition \ref{prop:infinitesupportlower} altogether state that Algorithm \ref{alg:SOA} is optimal in a worst-case sense when no additional structure is imposed on $\mathcal{P}_t$'s.


\section{Non-stationary Environment with Prior Estimate: Wasserstein Based Ambiguity and Analysis}

\label{WBNB_P}

In this section, we consider a ``data-driven'' setting where the true distribution is unknown, but a prior estimate of the true distribution is available. The setting relaxes the assumption on the exact knowledge of the true distribution in the last section. In practice, the availability of the prior estimate may characterize the predictable patterns of the non-stationarity in various application contexts. For example, the decision maker may not be able to foresee the future demand (distribution), but (s)he can construct some estimate based on history data or domain expertise of demand seasonality, the day-of-week effect, and demand surge due to pre-scheduled promotions or shopping festivals. When such prior estimate is accurate (the same as the true distribution), the setting of prior estimate in this section reduces to the discussion in the last section. However, when the prior estimate deviates from the true distribution, as is often the case in reality, then two natural questions are: (i) how to properly measure the inaccuracy of the prior estimate from the true distribution, (ii) how to design and analyze algorithm with such prior estimate. We answer these two questions in this section.

\subsection{Wasserstein-Based Measure of Deviation}

\label{sec_WBDB}
Consider the decision maker has a prior estimate/prediction $\hat{\mathcal{P}}_t$ for the true distribution $\mathcal{P}_t$ for each $t$, and all the predictions $\{\hat{\mathcal{P}}_t\}_{t=1}^T$ are made available at the very beginning of the procedure.  We use the Wasserstein distance between $\hat{\mathcal{P}}_t$ and $\mathcal{P}_t$ to measure the deviation of the prior estimate from the true distribution. In following, we first formalize the definition and then discuss the suitability of the proposed Wasserstein-based measure.

The Wasserstein distance, also known as Kantorovich-Rubinstein metric or optimal transport distance  \citep{villani2008optimal, galichon2018optimal}, is a distance function defined between probability distributions on a metric space. Its notion has a long history dating back over decades ago and gains increasingly popularity in recent years with a wide range of applications including generative modeling \citep{arjovsky2017wasserstein}, robust optimization \citep{esfahani2018data}, statistical estimation \citep{blanchet2019robust}, etc.
In our context, the Wasserstein distance for two probability distributions $\mathcal{Q}_1$ and $\mathcal{Q}_2$ on the metric parameter space $(\Theta, \mathcal{B}_{\Theta})$ is defined as follows,
\begin{equation}\label{wasserstein}
\mathcal{W}(\mathcal{Q}_1, \mathcal{Q}_2) \coloneqq \inf_{\mathcal{Q}_{1,2} \in \mathcal{J}(\mathcal{Q}_1, \mathcal{Q}_2)} \int \rho(\bm{\theta_1},\bm{\theta_2}) d\mathcal{Q}_{1,2}(\bm{\theta_1},\bm{\theta_2})
\end{equation}
where $\mathcal{J}(\mathcal{Q}_1, \mathcal{Q}_2)$ denotes all the joint distributions $\mathcal{Q}_{1,2}$ for $(\bm{\theta_1},\bm{\theta_2})$ that have marginals $\mathcal{Q}_1$ and $\mathcal{Q}_2$. The distance function $\rho(\cdot,\cdot)$ is defined earlier in \eqref{rho_dist}.

We define the following Wasserstein-based deviation budget (WBDB) to measure the cumulative deviation of the prior estimate,
\[
\mathcal{W}_T(\mathcal{P},\hat{\mathcal{P}})\coloneqq\sum_{t=1}^{T}\mathcal{W}(\mathcal{P}_t,\hat{\mathcal{P}}_t)
\label{def_WBDB}
\]
where $\mathcal{P}=(\mathcal{P}_1,...,\mathcal{P}_T)$ denotes the true distribution and $\hat{\mathcal{P}}=(\hat{\mathcal{P}}_1,...,\hat{\mathcal{P}}_T)$ denotes the prior estimate. 

Based on the notion of WBDB, we define a set of distributions
\[
   \Xi_P(W_T) \coloneqq \{\mathcal{P}: \mathcal{W}_T(\mathcal{P},\hat{\mathcal{P}})\le W_T,\mathcal{P}=(\mathcal{P}_1,...,\mathcal{P}_T)\} \label{defXi1}
\]
for a non-negative constant $W_T$, which we call as \textit{deviation budget}. In this section, we consider a regret based on the set $\Xi_P$ as defined in \eqref{regDef}. In this way, the regret characterizes a worst-case performance of a certain algorithm for all the distributions $\mathcal{P}=(\mathcal{P}_1,...,\mathcal{P}_T)$ within the set $\Xi_P$ prescribed by some $W_T$. Specifically, the deviation budget $W_T$ defines the set $\Xi_P$ by inducing an upper bound for the deviation of prior estimate. Our next theorem provides an intuitive result that $W_T$ is an  inevitable loss (in terms of the algorithm regret) as a result of the inaccuracy of the prior estimate.
\begin{theorem}
\label{newlowertheorem}
Under Assumption \ref{assume}, if we consider the set $\Xi_P(W_T) \coloneqq \{\mathcal{P}: \mathcal{W}_T(\mathcal{P},\hat{\mathcal{P}})\le W_T,\mathcal{P}=(\mathcal{P}_1,...,\mathcal{P}_T)\}$, there is no algorithm that can achieve a regret better than $\Omega(\max\{\sqrt{T},W_T\})$.
\end{theorem}

Theorem \ref{newlowertheorem} states that the lower bound of the regret is $\Omega(\max\{\sqrt{T},W_T\})$. The theorem characterizes the best achievable algorithm performance under an inaccurate prior estimate, and precisely, the lower bound is linear in respect with the deviation of the prior estimate from the true distribution. The $\Omega(\sqrt{T})$ part inherits the result in Proposition \ref{prop:infinitesupportlower} and it captures the intrinsic uncertainty of the underlying stochastic process over a time horizon $T$. The $\Omega(W_T)$ part captures the uncertainty arising from the inaccurate prior estimate.

\subsection{Informative Gradient Descent Algorithm with Prior Estimate}

Now we apply our informative gradient descent algorithm to the setting of prior estimate. A natural idea is to pretend that the prior estimate $\hat{\mathcal{P}}_t$ is indeed the true distribution $\mathcal{P}_t$. To implement the idea, we define
\[
\hat{L}(\bm{p})=\bm{c}^\top\bm{p}+\sum_{t=1}^{T}\hat{\mathcal{P}}_t h(\bm{p};\bm{\theta})
\]
where the true distribution $\mathcal{P}_t$ is replaced by its estimate $\hat{\mathcal{P}}_t$ for each component in function $L(\cdot)$. Thus it can be viewed as an approximation for the true dual function $L(\cdot)$ based on prior estimate. Let $\hat{\bm{p}}^*$ denote an optimal solution to $\hat{L}(\cdot)$,
\begin{equation}
\hat{\bm{p}}^* \in\text{argmin}_{\bm{p}\geq0}\hat{L}(\bm{p})\label{p_star}
\end{equation} and for each $t$, define
\begin{equation}\label{2007}
\hat{\bm{\gamma}}_t\coloneqq \hat{\mathcal{P}}_t\bm{g}(\hat{\bm{x}}(\bm{\theta});\bm{\theta})  \text{~~where~} \hat{\bm{x}}(\bm{\theta})=\text{argmax}_{\bm{x}\in\mathcal{X}}\{f(\bm{x};\bm{\theta})-(\hat{\bm{p}}^*)^\top\cdot\bm{g}(\bm{x};\bm{\theta})\}.
\end{equation}
Here, $\hat{\bm{\gamma}}_t$ denotes the ``optimal'' expected budget consumption in the $t$-th time under the prior estimate.
In the setting of prior estimate, we do not have the exact knowledge of the true distributions $\mathcal{P}_t$'s and therefore $\bm{\gamma}_t$'s, so we alternatively use $\hat{\bm{\gamma}}_t$ as a substitute. Thus, the algorithm for the prior estimate setting, denoted by IGD$(\hat{\bm{\gamma}})$, implements \Cref{alg:SOA} with the input $\hat{\bm{\gamma}}$ defined by \eqref{2007}. Specifically, the dual update step will become
\begin{equation}\label{2003}
\bm{p}_{t+1} =\left(\bm{p}_t + \frac{1}{\sqrt{T}} \left(\bm{g}(\tilde{\bm{x}}_t;\bm{\theta}_t) - \hat{\bm{\gamma}}_t\right)\right)\vee \bm{0}
\end{equation}
in \Cref{alg:SOA}.

\subsection{Regret Analysis}

The analysis of IGD($\hat{\bm{\gamma}}$) is slightly more complicated than that of IGD($\bm{\gamma}$) in \cref{newnonstationtheorem} because the algorithm is built upon the function $\hat{L}(\cdot)$ defined by the prior estimate instead of the true distribution. So we first study how to bound the difference between the function $\hat{L}(\cdot)$ and $L(\cdot)$ using the deviation between prior estimate and true distribution. For a probability measure $\mathcal{Q}$ over the metric parameter space $(\Theta, \mathcal{B}_{\Theta})$, we denote
\begin{equation*}\label{01}
L_{\mathcal{Q}}(\bm{p})\coloneqq \mathcal{Q}h(\bm{p};\bm{\theta})=\int_{\bm{\theta}'\in\Theta} h(\bm{p};\bm{\theta}')d\mathcal{Q}(\bm{\theta}').
\end{equation*}
Then the function $\hat{L}(\cdot)$ and $L(\cdot)$ can be expressed as
\[
\hat{L}(\bm{p})=\bm{p}^\top\bm{c}+\sum_{t=1}^T L_{\hat{\mathcal{P}}_t}(\bm{p})\text{~~and~~}   L(\bm{p}) =\bm{p}^\top\bm{c}+ \sum_{t=1}^T L_{\mathcal{P}_t}(\bm{p})
\]
Lemma \ref{wasserlemma} states that the function $L_{\mathcal{Q}}(\bm{p})$ has certain ``Lipschitz continuity'' in regard with the underlying distribution $\mathcal{Q}$. Specifically, the supremum norm between two functions $L_{\mathcal{Q}_1}(\bm{p})$ and $L_{\mathcal{Q}_2}(\bm{p})$ is bounded by the Wasserstein distance between two distributions $\mathcal{Q}_1$ and $\mathcal{Q}_2$ up to a constant dependent on the dimension and the boundedness of the function's argument.

\begin{lemma}\label{wasserlemma}
For two probability measures $\mathcal{Q}_1$ and $\mathcal{Q}_2$ over the metric parameter space $(\Theta, \mathcal{B}_{\Theta})$, we have that
\begin{equation}\label{02}
\sup_{\bm{p}\in\Omega_{\bar{p}}}\left|L_{\mathcal{Q}_1}(\bm{p})-L_{\mathcal{Q}_2}(\bm{p}) \right|\leq  \max\{1,\bar{p}\}\cdot (m+1)\mathcal{W}(\mathcal{Q}_1, \mathcal{Q}_2)
\end{equation}
where $\Omega_{\bar{p}}=[0,\bar{p}]^m$ and $\bar{p}$ is an arbitrary positive constant.
\end{lemma}
Note that the Lipschitz constant in Lemma \ref{wasserlemma} involves an upper bound of the function argument $\bm{p}.$ The following lemma provides such an upper bound for the dual price $\bm{p}_t$ in IGD($\hat{\bm{\gamma}}$). The derivation is essentially the same as Lemma \ref{newupperlemma}.

\begin{lemma}\label{upperlemmawp2}
For each $t=1,2,\dots,T$, we have that $\|\bm{p}_t\|_{\infty}\leq q+1$ with probability $1$, where $\bm{p}_t$ is specified by \eqref{2003} in IGD($\hat{\bm{\gamma}}$).
\end{lemma}

The rest of the analysis for IGD($\hat{\bm{\gamma}}$) is similar to that of IGD($\bm{\gamma})$ in Theorem \ref{newnonstationtheorem}. The regret of IGD($\hat{\bm{\gamma}}$) is formally stated in Theorem \ref{uppertheorem1}. Notably, the algorithm's regret matches the lower bound of Theorem \ref{newlowertheorem} and thus it establishes the optimality of the algorithm.

\begin{theorem}\label{uppertheorem1}
Under Assumption \ref{assume}, suppose a prior estimate $\hat{\mathcal{P}}$ is available and the regret is defined based on the set $\Xi_{P}(W_T)$, then the regret of IGD($\hat{\bm{\gamma}}$) has the following upper bound
$$\text{Reg}_T(\pi_{\text{IGDP}}) \le O(\max\{\sqrt{T},W_T\})$$
where $\pi_{\text{IGDP}}$ stands for the policy specified by IGD($\hat{\bm{\gamma}}$).
\end{theorem}

We remark the algorithm IGD($\hat{\bm{\gamma}}$) does not depend on or utilize the knowledge of the quantity $W_T.$ On the upside, this avoids the assumption on the prior knowledge of $W_T$ (as the knowledge of variation budget $V_T$ \citep{besbes2014stochastic, besbes2015non, cheung2019non}). On the downside, there is nothing the algorithm can do even when it knows a priori $W_T$ is small or large. Technically, it means for our algorithm IGD($\hat{\bm{\gamma}}$), the WBNB contributes nothing in the dimension of algorithm design, and it will only influence the algorithm analysis.  In particular, if we compare Theorem \ref{uppertheorem1} with Theorem \ref{newnonstationtheorem}, the extra term $W_T$ captures how the deviation of the prior estimate from the true distribution will deteriorate the performance of the gradient-based algorithm. When $W_T$ is small, the $O(\sqrt{T})$ will be dominant and we do not need to worry about the deviation because its effect on the regret is secondary. In this light, the regret bound illuminates the effect of model misspecification/estimation error on the algorithm's performance in a non-stationary environment.

\section{Non-stationary Environment Without Prior Estimate}

\label{WBNB}

In this section, we consider an uninformative setting where the true distribution is completely unknown to the decision maker. To one end, the discussion in this section can be viewed as a reduction of the results in the last section to a setting with an ``uninformative'' prior estimate. To the other, the uninformative setting draws an interesting comparison with the literature on (unconstrained) online learning/optimization in non-stationary environment \citep{besbes2014stochastic,besbes2015non, cheung2019non} and its analysis exemplifies the interaction between the constraints and the non-stationarity.


\subsection{Wasserstein-based Non-stationarity}
\label{WBNonSta}
We first illustrate how the non-stationarity over $\{\mathcal{P}_t\}_{t=1}^T$ interplays with the constraints through the following example adapted from \citep{golrezaei2014real}. The original usage of the example in their paper is to stress the importance of balancing resource consumption in an online setting. Specifically, consider the following two linear programs as the underlying problem \eqref{PCP} for two online stochastic optimization problems,
\begin{align}
   \max \ \ &  x_1+...+x_{c}+(1+\kappa)x_{c+1}+...+(1+\kappa)x_{T}  \label{eg1} \\
    \text{s.t. }\ & x_1+...+x_{c}+x_{c+1}+...+x_{T} \le c\nonumber \\
    & 0 \le x_t \le 1\ \text{ for } t=1,...,T. \nonumber \\
   \max \ \ &  x_1+...+x_{c}+(1-\kappa)x_{c+1}+...+(1-\kappa)x_{T}  \label{eg2} \\
    \text{s.t. }\ & x_1+...+x_{c}+x_{c+1}+...+x_{T} \le c \nonumber\\
    & 0 \le x_t \le 1\ \text{ for } t=1,...,T.\nonumber
\end{align}
where $\kappa\in(0,1)$, $c=\frac{T}{2}$ and the true distributions for both scenarios are point-mass distributions. Without loss of generality, we assume $c$ is an integer. For the first LP \eqref{eg1}, the optimal solution is to wait and accept the later half of the orders, while for the second LP \eqref{eg2}, the optimal solution is to accept the first half of the orders and deplete the resource at half time. The contrast between these two LPs (two scenarios of whether the first half or the second half is more profitable) creates difficulty for the online decision making. Without knowledge of the future orders, there is no way we can obtain a sub-linear regret in both scenarios simultaneously. Because if we exhaust too much resource in the first half of the time, then for the first scenario \eqref{eg1}, we do not have enough capacity to accept all the relatively profitable orders in the second half. On the contrary, if we have too much remaining resource at the half way, then for the second scenario \eqref{eg2}, those relatively profitable orders that we miss in the first half are irrevocable.

An equivalent view of these two examples is to consider the existence of an adversary:
the adversary is aware of the policy of the decision maker at the very beginning and then chooses the distribution $\mathcal{P}_t$ in an adversarial manner. Specifically, the adversary acts against us and aims to maximize the optimality gap between the offline optimal objective value and the online objective value. For example, in \eqref{eg1} and \eqref{eg2}, the adversary can make a decision of which scenario for us to enter for the second half of the time based on our remaining inventory at the half way. The adversary view augments our previous interpretation of $W_T$ as the maximal derivation (of the prior estimate from the true distribution): the regret definition based on $\Xi_P$ in the last section can be viewed as a partially adversarial setting where the adversary chooses the true distribution against our will in a sequential manner but the choice of the distributions $\mathcal{P}_t$'s is subject to the set $\Xi_P.$ Then the parameter $W_T$ that defines the set $\Xi_P$ serves as a measure of both the estimation error and the intensity of adversity.

\begin{proposition}
\label{worstCase}
The worst-case regret of online constrained  stochastic optimization in an adversarial setting is $\Omega(T)$.
\end{proposition}

Proposition \ref{worstCase} states that a fully adversarial setting where $\mathcal{P}_t$ can change arbitrarily over $t$ does not permit a sub-linear regret. The same observation is also made in the literature \citep{besbes2014stochastic, besbes2015non, cheung2019non} for unconstrained online learning problems where there is no function $\bm{g}(\cdot)$ and the decision $\bm{x}_t$ is made before the revealing of $f(\cdot)$. Specifically, \cite{besbes2015non} propose a novel measure of non-stationarity as follows (in the language of our paper),
\[
V_T \coloneqq \sum_{t=1}^{T-1} TV(\mathcal{P}_t,\mathcal{P}_{t+1})
\]
where $TV(\cdot,\cdot)$ denotes the total variation distance between two distributions. The quantity $V_T$ represents the cumulative \textit{temporal change} of the distributions by comparing $\mathcal{P}_t$ and $\mathcal{P}_{t+1}$. Unfortunately, such temporal measure fails in the constrained setting. Note that for both \eqref{eg1} and \eqref{eg2}, there is only one change point throughout the whole procedure thus the non-stationarity; their \textit{temporal change} measure is $O(1)$ but a sub-linear regret is still unattainable.


Now, we propose the definition of the Wasserstein-based non-stationarity budget (WBNB) as
\[
    \mathcal{W}_T(\mathcal{P}) \coloneqq \sum_{t=1}^T \mathcal{W}(\mathcal{P}_t, \bar{\mathcal{P}}_T)
    \label{def_WBNB}
\]
where $\mathcal{P}=(\mathcal{P}_1,...,\mathcal{P}_T)$ and $\bar{\mathcal{P}}_T$ is defined to be the uniform mixture distribution of $\{\mathcal{P}_t\}_{t=1}^T$, i.e., $\bar{\mathcal{P}}_T\coloneqq \frac{1}{T} \sum_{t=1}^T \mathcal{P}_t$. The non-stationarity measure WBNB can be viewed as a degeneration of our previous deviation measure WBDB in that WBNB replaces all the prior estimates $\hat{\mathcal{P}}_t$'s with the uniform mixture $\bar{\mathcal{P}}_T$. The caveat is that in the uninformative setting, no distribution knowledge is assumed, so the decision maker does not have access to $\bar{\mathcal{P}}_T$ unlike the prior estimate in the last section. As we will see shortly, the knowledge of $\bar{\mathcal{P}}_T$ does not affect anything in terms of the algorithm design and analysis.

Based on the notion of WBNB, we define a set of distributions
\[
   \Xi_{U}(W_T) = \{\mathcal{P}: \mathcal{W}_T(\mathcal{P})\le W_T,\mathcal{P}=(\mathcal{P}_1,...,\mathcal{P}_T)\} \label{defXi}
\]
for a non-negative constant $W_T$, which we call as variation budget. Throughout this section, we consider a regret based on the set $\Xi_{U}$ as defined in \eqref{regDef} in aim to characterize a ``worst-case'' performance of certain policy/algorithm for all the distributions in the set $\Xi_{U}$.

The variation budget $W_T$ defines the uncertainty set $\Xi_{U}$ by providing an upper bound on the non-stationarity of the distributions. Our next theorem states that it is impossible to get rid of $W_T$ in the regret bound of any algorithm, which illustrates the sharpness of our definition of WBNB. Intuitively, it means that apart from the intrinsic stochasticity term $O(\sqrt{T}),$ the (unknown) non-stationarity of the underlying distributions defined by WBNB appears to be a second bottleneck for algorithm performance. The proof of the theorem follows the same argument as Theorem \ref{newlowertheorem}.

\begin{theorem}
\label{lowertheorem}
Under Assumption \ref{assume}, if we consider the set $\Xi_{U}(W_T) = \{\mathcal{P}: \mathcal{W}_T(\mathcal{P})\le W_T,\mathcal{P}=(\mathcal{P}_1,...,\mathcal{P}_T)\}$, there is no algorithm that can achieve a regret better than $\Omega(\max\{\sqrt{T},W_T\})$.
\end{theorem}

\subsection{Algorithm and Regret Analysis}

One pillar of designing IGD($\bm{\gamma}$) and IGD($\hat{\bm{\gamma}}$) is the budget allocation plan $\bm{\gamma}_t$'s (or $\hat{\bm{\gamma}}_t$) prescribed by either the true distribution or the prior estimate. In the uninformative setting, the most straightforward (and probably optimal) plan is to allocate the budget evenly over the entire horizon. Algorithm \ref{alg:SOAWOP} -- \textit{uninformative gradient descent} (UGD) -- implements the intuition by evenly distributing the budget without referring to any information. Thus the UGD algorithm can be viewed as IGD$\left(\frac{\bm{c}}{T}\right).$ Returning to the point mentioned earlier on the knowledge of the centric distribution $\bar{\mathcal{P}}_T,$ it does not matter we know it or not; because as long as all the prior estimate distributions $\hat{\mathcal{P}}_t$ are the same over time, we always have the same budget allocation plan. Furthermore, when all the $\mathcal{P}_t$'s are the same, which means the variation budget $W_T=0,$ Algorithm \ref{alg:SOAWOP} and its analysis collapse into several recent studies on the gradient-based online algorithm under a stationary environment \citep{lu2020dual, li2020simple}. 



\begin{algorithm}[ht!]
\caption{Uninformative Gradient Descent Algorithm (UGD)}
\label{alg:SOAWOP}
\begin{algorithmic}[1]
\State Initialize the initial dual price $\bm{p}_1 = \bm{0}$ and initial constraint capacity $\bm{c}_{1}=\bm{c}$.
\For {$t=1,..., T$}
\State Observe $\bm{\theta}_t$ and solve
\[
\tilde{\bm{x}}_t=\text{argmax}_{\bm{x}\in\mathcal{X}}\{ f(\bm{x};\bm{\theta}_t)-\bm{p}_t^\top \cdot \bm{g}(\bm{x};\bm{\theta}_t) \}
\]
where $\bm{g}(\bm{x},\bm{\theta}_t)=(g_1(\bm{x},\bm{\theta}_t),...,g_m(\bm{x},\bm{\theta}_t))^\top$
\State Set
\[
\bm{x}_t=\left\{\begin{aligned}
&\tilde{\bm{x}}_t,&&\text{if $\bm{c}_{t}$ permits a consumption of~}\bm{g}(\tilde{\bm{x}}_t;\bm{\theta}_t) \\
&\bm{0},&&\text{otherwise}
\end{aligned}\right.
\]
\State Update the dual price
\begin{equation}\label{new2003}
\bm{p}_{t+1} =\left(\bm{p}_t + \frac{1}{\sqrt{T}} \left(\bm{g}(\tilde{\bm{x}}_t;\bm{\theta}_t) - \frac{\bm{c}}{T}\right)\right)\vee \bm{0}
\end{equation}
where the element-wise maximum operator $u\vee v = \max \{v,u\}$
\State Update the remaining capacity
\[
\bm{c}_{t+1}=\bm{c}_{t}-\bm{g}(\bm{x}_t;\bm{\theta}_t)
\]
\EndFor
\State Output: $\bm{x} = (\bm{x}_1,...,\bm{x}_T)$
\end{algorithmic}
\end{algorithm}


\begin{theorem}\label{uppertheorem}
Under Assumption \ref{assume}, if we consider the set $\Xi_{U}(W_T) = \{\mathcal{P}: \mathcal{W}_T(\mathcal{P})\le W_T,\mathcal{P}=(\mathcal{P}_1,...,\mathcal{P}_T)\}$, then the regret of Algorithm \ref{alg:SOAWOP} has the following upper bound
$$\text{Reg}_T(\pi_{\text{UGD}}) \le O(\max\{\sqrt{T},W_T\})$$
where $\pi_{\text{UGD}}$ stands for the policy specified by Algorithm \ref{alg:SOAWOP}.
\end{theorem}

Theorem \ref{uppertheorem} states the upper bound of Algorithm \ref{alg:SOAWOP}, which matches the regret lower bound in Theorem \ref{lowertheorem}. Remarkably, the factors on $T$ and $W_T$ are additive in the regret upper bound of Algorithm \ref{alg:SOAWOP}. In comparison, the factor on $T$ and the variation budget $V_T$ are usually multiplicative in the regret upper bounds in the line of works that adopts the temporal change variation budget as nonstationary measure \citep{besbes2014stochastic, besbes2015non, cheung2019non}. The price of such an advantage for WBNB is that the WBNB is a more restrictive notion than the variation budget; recall that in \eqref{eg1} and \eqref{eg2}, the temporal change variational budget is $O(1)$, but the WBNB is $O(\kappa T)$.  Again, as the setting with prior estimate, the knowledge of the quantity $W_T$ does not affect the algorithm design. By putting together Theorem \ref{lowertheorem} and Theorem \ref{uppertheorem}, we argue that the knowledge of $W_T$ does not help to further improve the algorithm performance.

\section{Extensions and Discussions}

\label{sec_extension}

\subsection{Improved Regret Bound under Finite Support Assumption}

\label{sec_bounded}

In previous sections, we impose no additional structure on the underlying distributions $\mathcal{P}_t$'s other than the nonstationarity variation $W_T$. A natural question is whether a better regret bound is achievable when the underlying distribution has more structures. In this subsection, we provide a positive answer to the question under an online LP formulation. Specifically, for the online linear programming problem, $\bm{\theta}_t=(r_t, \bm{a}_t)$ and $x_t\in[0,1]$. The functions $f(x_t,\bm{\theta}_t)=r_tx_t$ and $g_i(x_t,\bm{\theta}_t)=a_{it}x_t$ are linear. We make the following assumption on the randomness of $\bm{\theta}_t$.

\begin{assumption}\label{assump:finitesupport}
For each $t$, the distribution of $\bm{\theta}_t$ has a finite support, denoted by $\{\bm{\theta}^{(1)},\dots,\bm{\theta}^{(n)}\}$. Moreover, there exists $p_{\min}>0$ such that for each $t=1,\dots,T$ and each $j=1,\dots,n$, it holds that $P(\bm{\theta}_t=\bm{\theta}^{(j)})\geq p_{\min}$.
\end{assumption}

Under this finite support assumption, the online LP problem is also known as the quantity based network revenue management problem and has been studied extensively in the literature. Here we consider a setting where the true distribution $\mathcal{P}_t$ that governs $\bm{\theta}_t$ is unknown but a prior estimate $\hat{\mathcal{P}}_t$ is available. \cite{jasin2012re, vera2019online} develop re-solving algorithms and derive bounded regret results when the underlying distribution is known, i.e., $\hat{P}_t$ is accurate. Compared to these existing works, our setting of prior estimate captures the potential misspecification or estimation error of the underlying distribution. Thus our analysis examines the robustness of the online algorithm against such error.

Next, we introduce an algorithm that extends  the Fluid Bayes Selector algorithm in \citep{vera2019online}. Specifically, our algorithm uses the prior estimate instead of the true distribution for prescribing decisions. To describe the algorithm, we denote by $\bm{c}_t$ the remaining capacity at the beginning of each period $t$, and denote by $\mathcal{H}(t)=\{ \bm{\theta}_t,\dots,\bm{\theta}_T \}$ the (future) trajectory from period $t$ to period $T$. Let $\mathcal{H}_j(t)$ be the number of times that $\bm{\theta}_{\tau}$ is realized as $\bm{\theta}^{(j)}$ for $\tau=t,\dots,T$. Then, the trajectory $\mathcal{H}(t)$ can be equivalently represented by $\mathcal{H}(t)=(\mathcal{H}_1(t),\dots,\mathcal{H}_n(t))$. The offline hindsight problem \eqref{PCP} with the remaining capacity $\bm{c}_t$ from period $t$ to period $T$ can be written as
\[\begin{aligned}
R^*_t(\mathbf{c}_t, \mathcal{H}(t))=&\max~~\sum_{j=1}^{n}r_{j}z_{j},\\
&~~\mbox{s.t.}~~~\sum_{j=1}^{n}a_{ji} z_{j}\leq c_{ti},~~\forall i=1,\dots, m,\\
&~~~~~~~~~0\leq x_{j}\leq \mathcal{H}_j(t),~~\forall j=1,\dots,n\\
\end{aligned}\]
where the decision variable $z_j$ can be interpreted as the number of accepted orders of the $j$-th type, and each unit acceptance is associated with an reward $r_j$ and a resource consumption vector of $(a_{j1},...,a_{jm}).$ By taking expectation of the arrival counts $\mathcal{H}_{j}(t)$, an upper bound on the expected optimal objective value of the hindsight problem can be obtained as follows
\begin{align}
R^{\text{UB}}_t(\mathbf{c}_t, \mathcal{P})=&\max~~\sum_{j=1}^{n}r_{j}z_{j},\label{UBfinitesupport}\\
&~~\mbox{s.t.}~~~\sum_{j=1}^{n}a_{ji} z_{j}\leq c_{ti},~~\forall i=1,\dots, m\nonumber,\\
&~~~~~~~~~0\leq z_{j}\leq \mathbb{E}_{\mathcal{H}\sim\mathcal{P}}[\mathcal{H}_j(t)],~~\forall j=1,\dots,n,\nonumber
\end{align}
where $\mathcal{P}=(\mathcal{P}_t,...,\mathcal{P}_T)$ encapsulates the underlying distributions. The Fluid Bayes Selector in \citep{vera2019online} solves the LP \eqref{UBfinitesupport} and uses its optimal solution to guide the online decisions. As the resource capacity $c_{it}$ changes over time and the LP \eqref{UBfinitesupport} is solved at each time period, this type of algorithms is known as a re-solving algorithm.

When the true distribution $\mathcal{P}$ is unknown but some prior estimate $\hat{\mathcal{P}}$ is available, a natural idea is to calculate the right-hand-side of constraints (the expectations) in \eqref{UBfinitesupport} using the prior estimate distributions. Specifically, in Algorithm \ref{alg:Finitesupport}, the decision maker first observes the parameter type $j_t$ and refer to the optimal solution of \eqref{UBfinitesupport} (using $\hat{\mathcal{P}}$ instead of $\mathcal{P}$ as its input) to choose the value of $x_t$.

\begin{algorithm}[ht!]
\caption{Re-solving with Prior Estimate}
\label{alg:Finitesupport}
\begin{algorithmic}[1]
\State Initialize constraint capacity $\bm{c}_{1}=\bm{c}$.
\For {$t=1,..., T$}
\State Solve $R^{\text{UB}}_t(\mathbf{c}_t, \hat{\mathcal{P}})$ in \eqref{UBfinitesupport} and obtain $\{\hat{z}_j(\mathbf{c}_t)\}_{j=1}^n$ as one optimal solution.
\State Observe the realization of $\bm{\theta}_t$ as $\bm{\theta}^{(j_t)}$.
\If {$\hat{z}_{j_t}(\mathbf{c}_t)\geq \frac{1}{2}\cdot \mathbb{E}_{\mathcal{H}\sim\hat{\mathcal{P}}}[\mathcal{H}_{j_t}(t)]$,}{~set $x_t=1$.}
\Else {~set $x_t=0$.}
\EndIf
\EndFor
\State Output: $\bm{x} = (x_1,...,x_T)$
\end{algorithmic}
\end{algorithm}

\begin{theorem}\label{thm:finitesupport} Suppose a prior estimate $\hat{\mathcal{P}}$ is available and the regret is defined based on the set $\Xi_{P}(W_T)$.
Then, under \Cref{assume} and \Cref{assump:finitesupport}, the regret of \Cref{alg:Finitesupport} has the following upper bound
$$\text{Reg}_T(\pi_{\text{Resolve}}) \le O\left(\frac{1}{p_{\min}}(1+W_T)\right)$$
where $\pi_{\text{Resolve}}$ stands for the policy specified by \Cref{alg:Finitesupport}.
\end{theorem}

Theorem \ref{thm:finitesupport} states the regret bound of Algorithm \ref{alg:Finitesupport}. We provide two ways to interpret the regret bound. First, when the prior estimate is accurate ($W_T=0$), the algorithm and its analysis reduce to the results in \citep{vera2020bayesian}. Compared to the $O(\sqrt{T})$ result in Section \ref{known_distr}, the bounded regret here relies on the finite-supportedness and the existence of $p_{\min}$ in Assumption \ref{assump:finitesupport}. Specifically, the assumption essentially requires the distributions $\mathcal{P}_t$ to be almost stationary over time (some concentration property holds) and thus guarantees that the budget consumption stays close to its expectation. \cite{fruend2020a, fruend2020b} further relax the assumption into a concentration condition and obtain more general results for problems such as online bin packing. When Assumption \ref{assump:finitesupport} or the concentration condition does not hold, the analyses in \citep{vera2020bayesian, fruend2020a, fruend2020b} no longer work, but our regret bound of $O(\sqrt{T})$ in Section \ref{known_distr} is still valid as it does not impose any condition on $\mathcal{P}_t$'s. Second, when the prior estimate is inaccurate ($W_T>0$), the regret bound in Theorem \ref{thm:finitesupport} captures the algorithm performance deterioration caused by the inaccurate prior estimation. Importantly, this term related to $W_T$ is not reducible even with stronger condition of the functions $f$ and $g_i$ such as Assumption \ref{assump:finitesupport}. In this sense, the result extends the existing ones which assumes the knowledge of true distributions \citep{vera2020bayesian,fruend2020a, fruend2020b} to a setting with model (distribution) misspecification.


\subsection{Sub-optimality of Static Policies}\label{sec:staticpolicy}

Both of our main algorithms -- Algorithm \ref{alg:SOA} and Algorithm \ref{alg:SOAWOP} are gradient based. Compared to Algorithm \ref{alg:Finitesupport} and other existing re-solving algorithms \citep{jasin2012re, bumpensanti2020re, vera2020bayesian}, the gradient-based algorithms feature for simplicity and computational efficiency. In addition, the gradient-based algorithms have an adaptive and dynamic design that is crucial in stabilizing the resource consumption (i.e., not to exhaust the resource too early or have too much resource left-over). As discussed in Section \ref{sec_IGD_algo}, this is achieved inherently by using the realized resource consumption at each time period in the gradient update. We argue that such a dynamic design that relates the online decisions with the realized resource consumption process is necessary in achieving an optimal order of regret. In contrast, for a \textit{static policy}, the online decisions can be dependent on the realized parameters $\bm{\theta}_t$'s but will not be affected by the dynamic of the resource consumption process. We remark that a static policy can utilize the prior estimate and be time-dependent; by ``static'', it means the policy remains the same regardless the realization of the resource consumption process. Examples of static policies include the bid-price policy \citep{talluri1998analysis} and the offline-to-online policy \citep{cheung2020online}.

\begin{defn}\label{def:staticpolicy}
A static policy $\pi$ is described by a set of functions $\{h_t^{\pi}\}_{t=1}^T$, where $h_t^{\pi}:\Theta\rightarrow \mathcal{X}$ is allowed to be a random function. At each period $t$, given the type $\bm{\theta}_t$, the policy $\pi$ will take the action $h_t^\pi(\bm{\theta}_t)$ if the budget constraints are not violated.
\end{defn}

Next, we illustrate the sub-optimality of any static policy's for both of the two settings. First, for the uninformative setting in \Cref{WBNB},  a static policy clearly cannot work since there is no prior information that can be used to design the policy. Second, for the data-driven setting in \Cref{WBNB_P}, it is not as obvious whether a static policy can achieve the same order of regret optimality as the gradient-based algorithms. For example, what if we simply use the ``offline'' optimal dual solution $\bm{p}^*$ or $\hat{\bm{p}}^*$ to form a static decision rule throughout the procedure? This implements bid-price policy \citep{talluri1998analysis} for the network revenue management problem under the known distribution setting. In Section \ref{bidPrice}, we show that this bid-price policy can incur a linear regret under an environment with slight non-stationarity (arbitrarily small $W_T$). The following proposition provides a more general statement on the sub-optimality of any static policy under a non-stationarity environment.

\begin{proposition}\label{prop:staticpolicy}
Suppose $\pi$ is a static policy such that $\mathbb{E}_{\mathcal{H}\sim\hat{\mathcal{P}}}[R^{\text{UB}}_T-R_T(\pi)]\leq C_1\cdot\sqrt{T}$ for some constant $C_1>0$, where $\hat{\mathcal{P}}=\{\hat{\mathcal{P}}_1,\dots,\hat{\mathcal{P}}_T\}$ denotes the set of prior estimates. There exists a true distribution $\mathcal{P}=\{\mathcal{P}_1,\dots,\mathcal{P}_T\}$ such that $\mathcal{W}_T(\mathcal{P}, \hat{\mathcal{P}})\leq W_T=4\sqrt{C_1}\cdot  T^{3/4}$ and $\mathbb{E}_{\mathcal{H}\sim\mathcal{P}}[R^{\text{UB}}_T-R_T(\pi)]\geq C_2\cdot T$ for some constant $C_2>0$.
\end{proposition}

In the proposition, $\pi$ denotes an arbitrary static policy that achieves an $O(\sqrt{T})$ regret when the prior estimate $\hat{\mathcal{P}}$ is accurate. But when there is a difference between the prior estimate $\hat{\mathcal{P}}$ and $\mathcal{P}$ and even if the deviation budget $W_T$ is sublinear in $T$, the static policy $\pi$ may still incur a linear regret for some problem instances. The implication is that for a static policy that works well under a known distribution setting, its performance can drastically deteriorate when there exists an estimation error or non-stationarity. Thus the dynamic design of gradient update or re-solving is both effective and necessary in overcoming the estimation error and the environment non-stationarity.

\subsection{Advantage of Wasserstein distance}

\label{sec:WassersteinWBNB}

In the previous sections, we use Wasserstein distance to define both the deviation budget WBDN and the non-stationarity budget WBNB. We note that the analyses and regret bounds still hold if we change the underlying distance to total variation distance or KL-divergence. We choose the Wasserstein distance because it is a tighter measure than the total variation distance or the KL-divergence for defining WBNB. If we revisit the examples \eqref{eg1} and \eqref{eg2}, a smaller value of $\kappa\in(0,1)$ should indicate a smaller variation/non-stationarity between the first half and the second half of observations in both examples. However, the total variation distance fails to characterize this subtlety in that for any non-zero value of $\kappa$, the total variation distance between $\mathcal{P}_t$ and $\mathcal{P}_{t'}$ for $t\le \frac{T}{2} <t'$ is always $1$ (since $\mathcal{P}_t$ and $\mathcal{P}_{t'}$ have different supports). In other words, if we replace the Wasserstein distance with the total variation distance in our definition of WBNB, then the quantity will always be $\frac{T}{2}$ for all $\kappa\in(0,1).$ Formally, for any two distributions $\mathcal{P}_1$ and $\mathcal{P}_2,$ the following inequality holds \citep{gibbs2002choosing},
\[
\mathcal{W}(\mathcal{P}_1, \mathcal{P}_2)\leq \text{TV}(\mathcal{P}_1, \mathcal{P}_2).
\]
Interestingly, this coincides with the intuitions in the literature of generative adversarial network (GAN) where \cite{arjovsky2017wasserstein} replace the KL-divergence with the Wasserstein distance in training GANs. Simultaneously and independently, \cite{balseiro2020best} analyze the dual mirror descent algorithm under a similar setting as our results in this section. The paper only discusses our uninformative setting, but not the known true distribution setting and the prior estimate setting. For the uninformative setting, their definition of non-stationarity is parallel to WBNB; both can be viewed as a reduction from the more general WBDB. The key difference is that \cite{balseiro2020best} from WBNB consider the total variation distance, which inherits the definition of variation functional from \citep{besbes2015non}.


A recent paper \cite{cheung2020online} proposes to measure the deviation budget as follows:
\[
\Psi(\mathcal{P}, \hat{\mathcal{P}})=\frac{1}{m+1}\cdot\sup_{\bm{x}(\bm{\theta}): \Theta \rightarrow \mathcal{X}}\left\|\sum_{t=1}^T\mathbb{E}_{\bm{\theta}_t\sim\mathcal{P}_t}\left[\left(f(\bm{x}(\bm{\theta}_t),\bm{\theta}_t), \bm{g}(\bm{x}(\bm{\theta}_t),\bm{\theta}_t)\right)\right]-\mathbb{E}_{\bm{\theta}_t\sim\hat{\mathcal{P}}_t}\left[\left(f(\bm{x}(\bm{\theta}_t),\bm{\theta}_t), \bm{g}(\bm{x}(\bm{\theta}_t),\bm{\theta}_t)\right)\right]\right\|_2
\]
where $m$ is the number of resource constraints.
The measure is used to analyze the proposed offline-to-online policy, and it captures the supremum distance between the true environment and the estimated environment under the same decision function $\bm{x}(\bm{\theta})$. Apart from the static policy's sub-optimality discussed earlier, the measure is also looser than the Wasserstein-based measure. Specifically, as the previous examples \eqref{eg1} and \eqref{eg2}, if we choose the function $\bm{x}(\bm{\theta})$ such that $\bm{x}(\bm{\theta})=0$ when $\bm{\theta}=1$ and $\bm{\theta}=1-\kappa$ (the reward $r_t$ is $1$ or $1-\kappa$) and $\bm{x}(\bm{\theta})=1$ when $\bm{\theta}=1+\kappa$ (the reward $r_t$ is $1+\kappa$), then $\Psi(\mathcal{P}, \hat{\mathcal{P}})$ is always lower bounded by $\frac{T}{2(m+1)}$ regardless of the value of $\kappa$. In contrast, the Wasserstein distance is more sensitive and can capture the intensity of the parameter $\kappa$.

An additional practical benefit of Wasserstein distance is that the measure features for natural data-driven bounds. For example, the prior estimate $\hat{\mathcal{P}}_t$ can be constructed by the empirical distribution over $N_t$ history samples over $\mathcal{P}_t$. Suppose that there exists a constant $c_1$ such that $\rho(\bm{\theta}_1,\bm{\theta}_2)\leq c_1\cdot\|\bm{\theta}_1-\bm{\theta}_2\|_2$. Then the following bound can be obtained for $\mathcal{W}(\mathcal{P}_t, \hat{\mathcal{P}}_t)$ (Theorem 1 of \citet{fournier2015rate}), which explicitly relates the deviation (from $\hat{\mathcal{P}}_t$ to ${\mathcal{P}}_t$) with the number of history samples:
\[
P\left(\mathcal{W}(\mathcal{P}_t, \hat{\mathcal{P}}_t)\leq \epsilon \right)\geq 1-c_2\cdot\exp(-c_3\cdot N_t\epsilon^{\max\{K,2\}})
\]
where $K$ is the dimension of the parameter $\bm{\theta}_t$, and $c_2, c_3$ are two positive constants. With $\epsilon=\frac{1}{\sqrt{T}}$ and $N_t=\frac{1}{c_3}\cdot T^{\frac{1}{2}\cdot\max\{K,2\}}\cdot\log (T/\delta)$, we have the following inequality holds
\[
P\left(\mathcal{W}(\mathcal{P}_t,\hat{\mathcal{P}}_t)\leq\frac{1}{\sqrt{T}}\right)\geq1-\frac{c_2\cdot\delta}{T}
\]
holds for each $t$. Therefore, from union bound over $t=1,\dots,T$, we know that when $N_t=\frac{1}{c_3}\cdot T^{\frac{1}{2}\cdot\max\{K,2\}}\cdot\log (T/\delta)$, we have
\[
W_T=\sum_{t=1}^T \mathcal{W}(\mathcal{P}_t,\hat{\mathcal{P}}_t)\leq \sqrt{T}
\]
holds with probability at least $1-c_2\cdot\delta$.

\subsection{Algorithm with Random Reward and Consumption}\label{sec:stochastic}

In this subsection, we consider another natural extension of our model where the reward and the budget consumption can be stochastic after the parameter $\bm{\theta}_t$ is revealed and the action $\bm{x}_t$ is made. To be specific, given each $\bm{x}$ and $\bm{\theta}$, the reward $f(\bm{x}; \bm{\theta})$ and the budget consumption $\bm{g}(\bm{x};\bm{\theta})$ are all random variables. Such an extension allows us to cover the price-based NRM problem and the choice-based NRM problem first described in \Cref{sec:example}.

We note that the algorithms and results in the previous sections can be directly extended to this setting. We first illustrate the setting of known distribution described in \Cref{known_distr} where the distributions $\mathcal{P}_t$'s are known a priori. For each $\bm{x}$ and $\bm{\theta}$, we introduce
\[
\hat{f}(\bm{x}; \bm{\theta})\coloneqq\mathbb{E}[f(\bm{x}; \bm{\theta})]\text{~and~}\hat{\bm{g}}(\bm{x}; \bm{\theta})\coloneqq\mathbb{E}[\bm{g}(\bm{x}; \bm{\theta})]
\]
where the expectation is taken with respect the functions ($f$ and $\bm{g}$) and conditional on $\bm{x}$ and $\bm{\theta}.$
Then, we use $\hat{f}, \hat{\bm{g}}$ to revise the definition of the function $h$ by
\[
h(\bm{p}; \bm{\theta})\coloneqq \max_{\bm{x}\in\mathcal{X}}\{\hat{f}(\bm{x};\bm{\theta})-\bm{p}^T\hat{\bm{g}}(\bm{x};\bm{\theta})\}.
\]
Accordingly, the function $L$ can be defined by the new function $h$
\begin{equation}\label{eqn:Lrandom}
L(\bm{p}):=\bm{c}^T\bm{p}+\sum_{t=1}^T\mathcal{P}_th(\bm{p};\bm{\theta})
\end{equation}
and then the definition of $\bm{\gamma}$ is given by
\begin{equation}\label{eqn:randomgamma}
\bm{\gamma}_t\coloneqq \mathcal{P}_t\hat{\bm{g}}(\bm{x}^*(\bm{\theta});\bm{\theta})  \text{~~where~} \bm{x}^*(\bm{\theta})=\text{argmax}_{\bm{x}\in\mathcal{X}}\{\hat{f}(\bm{x};\bm{\theta})-(\bm{p}^*)^\top\cdot\hat{\bm{g}}(\bm{x};\bm{\theta})\}
\end{equation}
where $\bm{p}^*\in\text{argmin}_{\bm{p}\geq0} L(\bm{p})$. Also, the benchmark is set as $R_T^{\text{UB}}$ in \eqref{PUB} with the functions $f, \bm{g}$ replaced by their expectations $\hat{f}, \hat{\bm{g}}$. While implementing the algorithm IGD$(\bm{\gamma})$ in \Cref{alg:SOA}, the decision $\tilde{\bm{x}}_t$ is then set as
\begin{equation}\label{eqn:randomtildex}
    \tilde{\bm{x}}_t=\text{argmax}_{\bm{x}\in\mathcal{X}}\{\hat{f}(\bm{x};\bm{\theta}_t)-\bm{p}_t^T\cdot\hat{\bm{g}}(\bm{x};\bm{\theta}_t)\}.
\end{equation}
Note that \eqref{eqn:randomtildex} can be solved efficiently (see discussions on function $h$ in \Cref{discussion_h}).
We have the following regret upper bound regarding the algorithm IGD($\bm{\gamma}$) for this new setting with random reward and resource consumption. The result can be viewed as a generalization of Theorem \ref{newnonstationtheorem}.

\begin{theorem}
\label{thm:random}
Under Assumption \ref{assume}, if we consider the set $\Xi=\{\mathcal{P}:\mathcal{P}=(\mathcal{P}_1,...,\mathcal{P}_T), \forall \mathcal{P}_1,\dots,\forall \mathcal{P}_T\}$, then the regret of IGD($\bm{\gamma}$) has the following upper bound
$$\text{Reg}_T(\pi_{\text{IGD}}) \le O(\sqrt{T})$$
where $\pi_{\text{IGD}}$ stands for the policy specified by IGD($\bm{\gamma}$) with $\bm{\gamma}$ computed from \eqref{eqn:randomgamma}.
\end{theorem}

For the data-driven setting with prior estimate, the distance between two parameters $\bm{\theta}, \bm{\theta}'\in\Theta$ is adjusted by
\begin{equation}\label{eqn:distancerandom}
    \rho(\bm{\theta}, \bm{\theta}') \coloneqq \sup_{\bm{x}\in \mathcal{X}} \|(\hat{f}(\bm{x}; \bm{\theta}),\hat{\bm{g}}(\bm{x}; \bm{\theta}))-(\hat{f}(\bm{x}; \bm{\theta}'),\hat{\bm{g}}(\bm{x}; \bm{\theta}'))\|_\infty,
\end{equation}
and the definition of $\mathcal{W}(\mathcal{P}_t, \hat{\mathcal{P}}_t)$ follows. Then, let the parameters $\hat{\bm{\gamma}}$ be computed from \eqref{2007} with respect to $\hat{f}, \hat{\bm{g}}$. As a corollary of \Cref{thm:random}, we have the following regret bound under the data-driven setting.
\begin{corollary}\label{coro:random}
Under Assumption \ref{assume}, suppose a prior estimate $\hat{\mathcal{P}}$ is available and the regret is defined based on the set $\Xi_{P}(W_T)=\{\mathcal{P}: \sum_{t=1}^T\mathcal{W}(\mathcal{P}_t,\hat{\mathcal{P}}_t)\leq W_t\}$, where the distance $\mathcal{W}(\mathcal{P}_t, \hat{\mathcal{P}}_t)$ is defined following the distance in \eqref{eqn:distancerandom}, then the regret of IGD($\hat{\bm{\gamma}}$) has the following upper bound
$$\text{Reg}_T(\pi_{\text{IGDP}}) \le O(\max\{\sqrt{T},W_T\})$$
where $\pi_{\text{IGDP}}$ stands for the policy specified by IGD($\hat{\bm{\gamma}}$).
\end{corollary}
For the uninformative setting where the prior estimates are not available, it is direct to see that the regret bound in \Cref{coro:random} continues to hold as long as WBNB is defined following the distance in \eqref{eqn:distancerandom}.
\subsection{Algorithm with Infrequent Update}
In this subsection, we study the performance of our algorithm with an infrequent update scheme. Specifically, the infrequent update refers to that the dual update step \eqref{003} for \Cref{alg:SOA} is not done at each period, instead, the dual variable vector $\bm{p}$ is updated only after $K$ periods. In this way, our algorithm can be viewed as a batched algorithm, where the arrivals from period $t$ to period $t+K$, for some $t$, are viewed as a batch of samples to update the dual variable. Note that the dual vector determines the online decision rule of the primal variable $\bm{x}_t$. An infrequent scheme has the practical benefits of inducing more consistency in the decision rule over time.

The algorithm is formally described in \Cref{alg:BatchSOA}. It takes the same structure as Algorithm \ref{alg:SOA} and Algorithm \ref{alg:SOAWOP} but only updates the dual variables every $K$ time periods. We have the following regret bound for \Cref{alg:BatchSOA} for the known distribution setting.
\begin{algorithm}[ht!]
\caption{Batched Informative Gradient Descent Algorithm (B-IGD$(\bm{\gamma}, K, \alpha_{T,K})$)}
\label{alg:BatchSOA}
\begin{algorithmic}[1]
\State Input: parameters $\bm{\gamma}=(\bm{\gamma}_1,\dots,\bm{\gamma}_T)$, the batch size $K$, and the step size $\alpha_{T,K}$.
\State Initialize the initial dual price $\bm{p}_1 = \bm{0}$ and initial constraint capacity $\bm{c}_{1}=\bm{c}$
\For {$t=1,..., T$}
\State Specify the constant $l=\lfloor t/K \rfloor+1$, where $\lfloor\cdot\rfloor$ denotes the floor function.
\State Observe $\bm{\theta}_t$ and solve
\[
\tilde{\bm{x}}_t=\text{argmax}_{\bm{x}\in\mathcal{X}}\{ f(\bm{x};\bm{\theta}_t)-\bm{p}_l^\top \cdot \bm{g}(\bm{x};\bm{\theta}_t) \}
\]
where $\bm{g}(\bm{x},\bm{\theta}_t)=(g_1(\bm{x},\bm{\theta}_t),...,g_m(\bm{x},\bm{\theta}_t))^\top$
\State Set
\[
\bm{x}_t=\left\{\begin{aligned}
&\tilde{\bm{x}}_t,&&\text{if $\bm{c}_{t}$ permits a consumption of~}\bm{g}(\tilde{\bm{x}}_t;\bm{\theta}_t) \\
&\bm{0},&&\text{otherwise}
\end{aligned}\right.
\]
\State If $t=l\cdot K$, then we update the dual price
\begin{equation}\label{batch003}
\bm{p}_{l+1} =\left(\bm{p}_l + \alpha_{T,K}\cdot\sum_{\tau=(l-1)K+1}^{lK} \left(\bm{g}(\tilde{\bm{x}}_{\tau};\bm{\theta}_{\tau}) - \bm{\gamma}_{\tau}\right)\right)\vee \bm{0}
\end{equation}
where the element-wise maximum operator $u\vee v = \max \{v,u\}$
\State Update the remaining capacity
\[
\bm{c}_{t+1}=\bm{c}_{t}-\bm{g}(\bm{x}_t;\bm{\theta}_t)
\]
\EndFor
\State Output: $\bm{x} = (\bm{x}_1,...,\bm{x}_T)$
\end{algorithmic}
\end{algorithm}
\begin{theorem}
\label{thm:batch}
Under Assumption \ref{assume}, if we consider the set $\Xi=\{\mathcal{P}:\mathcal{P}=(\mathcal{P}_1,...,\mathcal{P}_T), \forall \mathcal{P}_1,\dots,\forall \mathcal{P}_T\}$, then the regret of B-IGD($\bm{\gamma}, K, \alpha_{T,K}$), with $\bm{\gamma}$ computed from \eqref{new2007}, has the following upper bound
\[
\text{Reg}_T(\pi_{\text{B-IGD}}) \le O\left(TK\alpha_{T,K}+\frac{1}{\alpha_{T,K}}+K\right)
\]
where $\pi_{\text{B-IGD}}$ stands for the policy specified by IGD($\bm{\gamma}$) in \Cref{alg:BatchSOA}.
\end{theorem}
With a choice of $\alpha_{T,K}=\frac{1}{\sqrt{TK}}$, we obtain an regret upper bound of $O(\sqrt{TK})$ for the known distribution setting. Following the same spirit, for the data-driven setting, we obtain the following regret bound for \Cref{alg:BatchSOA}.

\begin{corollary}\label{coro:batch}
Suppose a prior estimate $\hat{\mathcal{P}}$ is available and the regret is defined based on the set $\Xi_{P}(W_T)$. Then, under Assumption \ref{assume}, the regret of B-IGD($\hat{\bm{\gamma}}, K, \alpha_{T,K}$), with $\hat{\bm{\gamma}}$ computed from \eqref{2007}, has the following upper bound
$$\text{Reg}_T(\pi_{\text{B-IGD}}) \le O(TK\alpha_{T,K}+\frac{1}{\alpha_{T,K}}+K+W_T)$$
where $\pi_{\text{B-IGD}}$ stands for the policy specified by B-IGD($\hat{\bm{\gamma}}, K, \alpha_{T,K}$).
\end{corollary}
With a choice of $\alpha_{T,K}=\frac{1}{\sqrt{TK}}$, we obtain the regret bound $O(\max\{\sqrt{TK}, W_T\})$ in \Cref{coro:batch} under the data-driven setting. For the uninformative setting, it is clear to see that the regret bound in \Cref{coro:batch} still holds when $W_T$ refers to the WBNB. For both regret bounds in Theorem \ref{thm:batch} and Corollary \ref{coro:batch}, the additional factor of $\sqrt{K}$ is the price paid for the infrequent update.

\section{Numerical Experiments}

\label{sec_numeric}

\subsection{Experiment I: Online Linear Programming}

We first present some synthetic experiments for the setting of online LP where both the reward and cost function are linear, i.e., $f_t(x)=r_tx$ and $g_{it}(x)=a_{it}x$ for $i=1,...,m$ and $t=1,...,T$. Suppose that both the true distribution and the prior estimate distribution of $a_{it}$ follow Uniform$[0.1,1.1]$ throughout the entire horizon. We consider three different settings for the true distribution and the prior estimate distribution of $r_t$ (summarized in Table \ref{experimentinput}). The parameters $\alpha$ and $\beta$ are to be specified: $\alpha$ reflects the intensity of the non-stationarity over time and $\beta$ represents the error of the prior estimate. Specifically, the deviation budget WBDB (in Section \ref{WBNB_P}) grows linearly with $\beta$, while the non-stationarity budget WBNB (in Section \ref{WBNB}) grows linearly with $\alpha$.

The first setting considers a uniform distribution for $r_t$: the true distribution of $r_t$ follows Uniform$[0,1]$ for the first half of the time and Uniform$[0,\alpha]$ for the second half, while the prior estimate distribution follows Uniform$[0,1+\beta]$ for the first half of the time and Uniform$[0,\alpha+\beta]$ for the second half.

The second setting considers a truncated normal distribution for $r_t$: the true distribution of $r_t$ follows Normal$(1,1)$ for the first half of the time and Normal$(\alpha,1)$ for the second half, while the prior estimate distribution follows Normal$(1+\beta,1)$ for the first half of the time and Normal$(\alpha+\beta,1)$ for the second half. All the normal distributions are made non-negative by truncating at 0.

The third setting considers a mixture of the previous two distributions of uniform and truncated normal: the true distribution of $r_t$ follows a uniform mixture of Uniform$[0,1]$ and Normal$(1,1)$ for the first half of the time and a uniform mixture of Uniform$[0,\alpha]$ and Normal$(\alpha,1)$ for the second half, while the prior estimate distribution follows a uniform mixture of Uniform$[0,1+\beta]$ and Normal$(1+\beta,1)$ for the first half of the time and a uniform mixture of Uniform$[0,\alpha+\beta]$ and Normal$(\alpha+\beta,1)$ for the second half. As the second setting, all the normal distributions are made non-negative by truncating at 0.
\begin{table}[ht!]
  \centering
  \begin{tabular}{c|c|cc|cc}
    \toprule
\multirow{2}{*}{}&    & \multicolumn{2}{c|}{True Distribution}& \multicolumn{2}{c}{Prior Estimate} \\
&    Periods & $r_t$ & $a_{it}$ & $\hat{r}_t$   & $\hat{a}_{it}$ \\
   \hline
\multirow{2}{*}{The uniform setting} &   $t=1,\dots,\floor{\frac{T}{2}}$ & $\text{Unif}[0,1]$  & $\text{Unif}[0.1,1.1]$ & $\text{Unif}[0,1+\beta]$ & $\text{Unif}[0.1,1.1]$ \\

 &   $t=\floor{\frac{T}{2}}+1,\dots,T$& $\text{Unif}[0,\alpha]$  & $\text{Unif}[0.1,1.1]$ & $\text{Unif}[0,\alpha+\beta]$ &$\text{Unif}[0.1,1.1]$ \\
 \midrule
 \multirow{2}{*}{The normal setting} &   $t=1,\dots,\floor{\frac{T}{2}}$ & $\text{Norm}(1,1)$  & $\text{Unif}[0.1,1.1]$ & $\text{Norm}(1+\beta,1)$ & $\text{Unif}[0.1,1.1]$ \\

 &   $t=\floor{\frac{T}{2}}+1,\dots,T$& $\text{Norm}(\alpha,1)$  & $\text{Unif}[0.1,1.1]$ & $\text{Norm}(\alpha+\beta,1)$ &$\text{Unif}[0.1,1.1]$ \\
 \midrule
 \multirow{2}{*}{The mixed setting} &   $t=1,\dots,\floor{\frac{T}{2}}$ & $\text{Mix}(1)$  & $\text{Unif}[0.1,1.1]$ & $\text{Mix}(\beta)$ & $\text{Unif}[0.1,1.1]$ \\

 &   $t=\floor{\frac{T}{2}}+1,\dots,T$& $\text{Mix}(\alpha)$  & $\text{Unif}[0.1,1.1]$ & $\text{Mix}(\alpha+\beta)$ &$\text{Unif}[0.1,1.1]$ \\
    \bottomrule
  \end{tabular}
  \caption{Input of Experiment I: Unif$[a,b]$ denotes the uniform distribution over the interval $[a,b]$. Norm$(a,b)$ denotes a normal distribution with mean $a$ and standard deviation $b$, truncated at 0. Mix($a$) denotes a uniform mixture of Unif$[0,a]$ and Norm$(a,1)$.}\label{experimentinput}
\end{table}
\OneAndAHalfSpacedXI

For the numerical experiments, we implement Algorithm \ref{alg:SOA} -- IGD($\hat{\bm{\gamma}}$) and Algorithm \ref{alg:SOAWOP} (UGD) under different choices of $\alpha$ and $\beta$ to study the algorithm performance under both settings in our paper. Besides, we implement the classic fixed bid price control heuristics (FBP) proposed in \cite{talluri1998analysis}. The FBP method uses $\hat{\bm{p}}^*$ computed from the prior estimate \eqref{p_star} as the bid price. It then accepts the order (setting $x_t=1$) when $r_t\geq\bm{a}_t^\top\hat{\bm{p}}^*$ and there is enough resource; otherwise, it will reject the order (setting $x_t=0$). We also compare with the offline-to-online (O2O) algorithm proposed in \citet{cheung2020online}. The O2O algorithm is a dual-based online algorithm for a non-stationary setting. It first runs an offline procedure to construct a set of dual variables using prior estimates $\{\hat{\mathcal{P}}_t\}_{t=1}^T$. Then, at each period $t$, after observing the realized parameter $\bm{\theta}_t$, the O2O algorithm draws a dual variable $\bm{p}$ uniformly randomly from the set, and then implements $\tilde{\bm{x}}_t=\text{argmax}_{\bm{x}\in\mathcal{X}}\{ f(\bm{x};\bm{\theta}_t)-\bm{p}^\top \cdot \bm{g}(\bm{x};\bm{\theta}_t) \}$. Note that the O2O algorithm is a static policy discussed in Section \ref{sec:staticpolicy}. In our implementation, we select different step sizes to generate multiple sets of dual variables. Then, we apply each set of dual variables to the online problem and we select the best reward to report the performance of the O2O algorithm.

Figure \ref{numericalfigure1} and Tables \ref{tablenumerical2}-\ref{tablenumericalmixed}  report the performances of the four algorithms with the horizon $T=1000$, number of constraints $m=10,$ and initial resource capacity $c_i=200$ for each $i=1,...,m$.

\begin{figure*}[ht!]
    \centering
    \begin{subfigure}[b]{0.5\textwidth}
        \centering
        \includegraphics[width=0.9\textwidth]{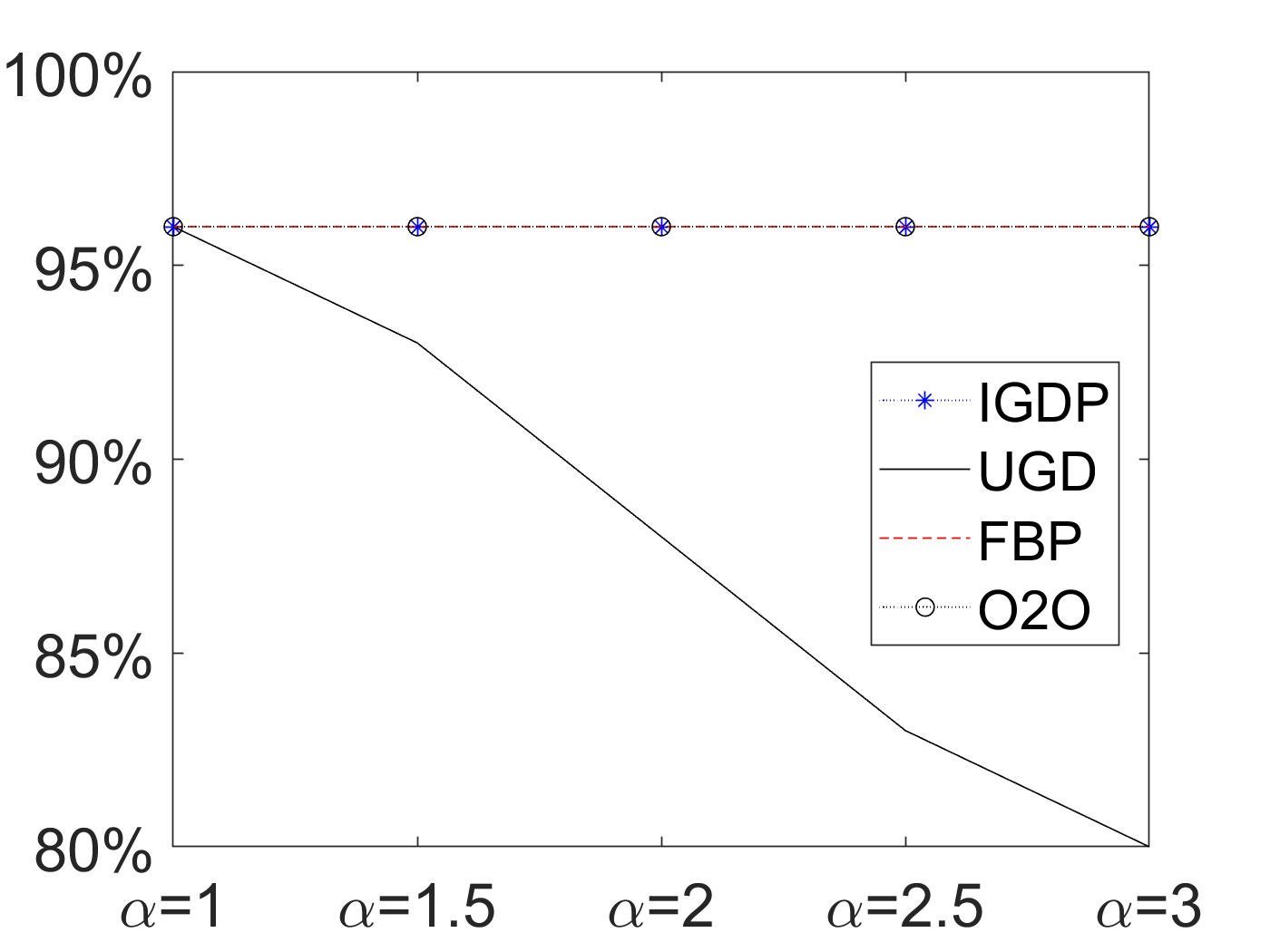}
        \caption{$\beta=0$ for the uniform setting}
    \end{subfigure}%
    ~
    \begin{subfigure}[b]{0.5\textwidth}
        \centering
        \includegraphics[width=0.9\textwidth]{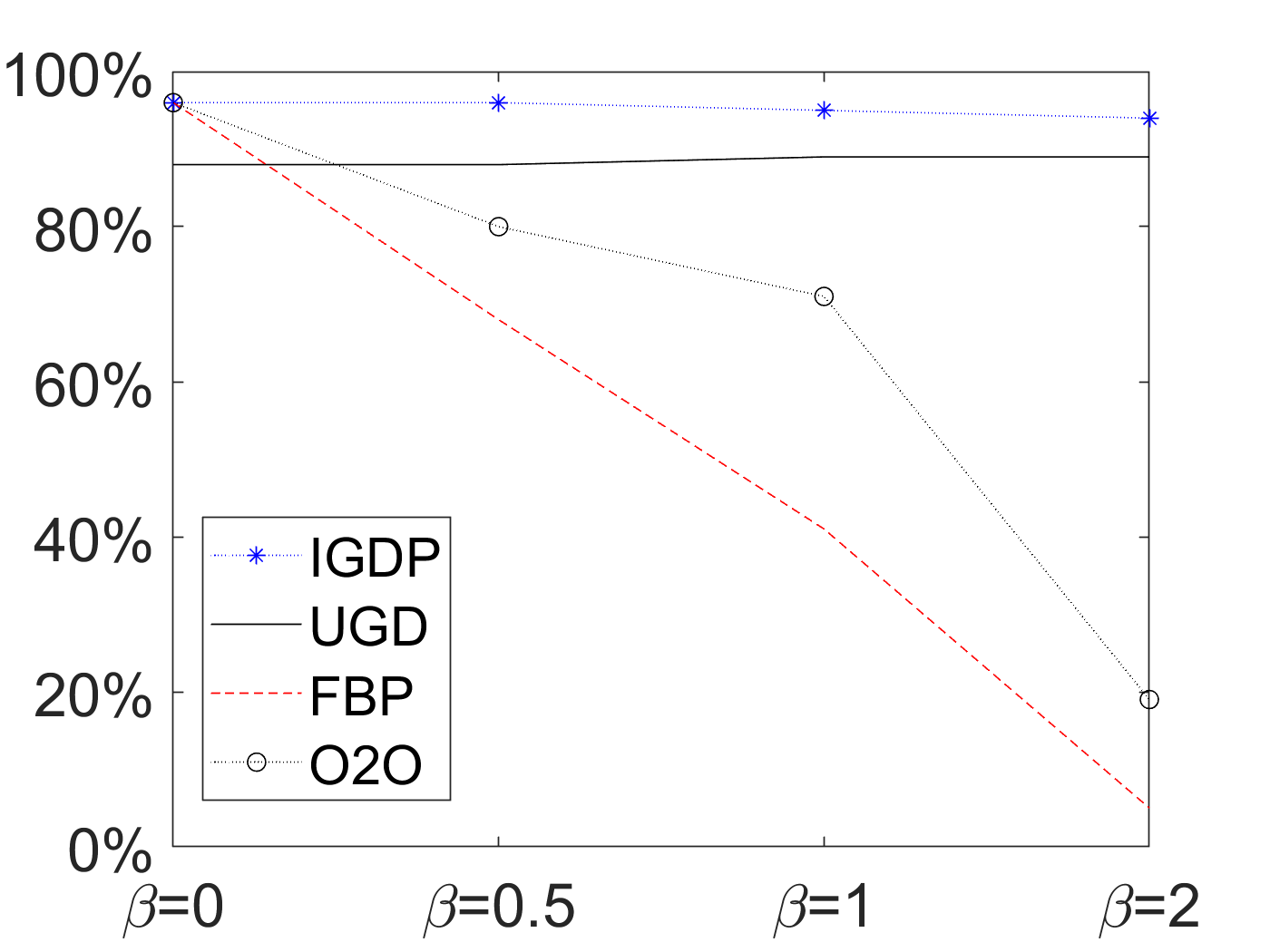}
        \caption{$\alpha=2$ for the uniform setting}
    \end{subfigure}
    \begin{subfigure}[b]{0.5\textwidth}
        \centering
        \includegraphics[width=0.9\textwidth]{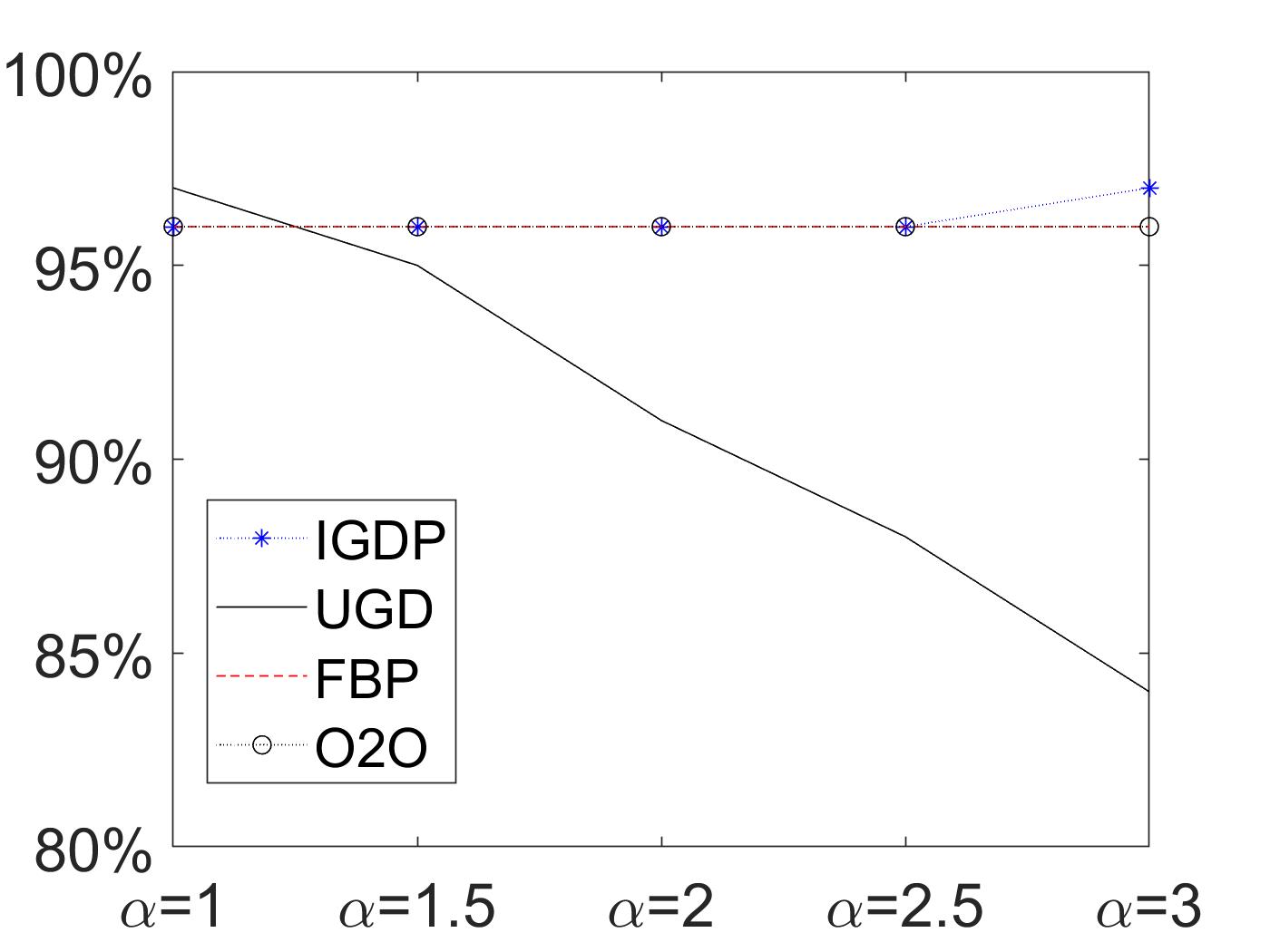}
        \caption{$\beta=0$ for the normal setting}
    \end{subfigure}%
    ~
    \begin{subfigure}[b]{0.5\textwidth}
        \centering
        \includegraphics[width=0.9\textwidth]{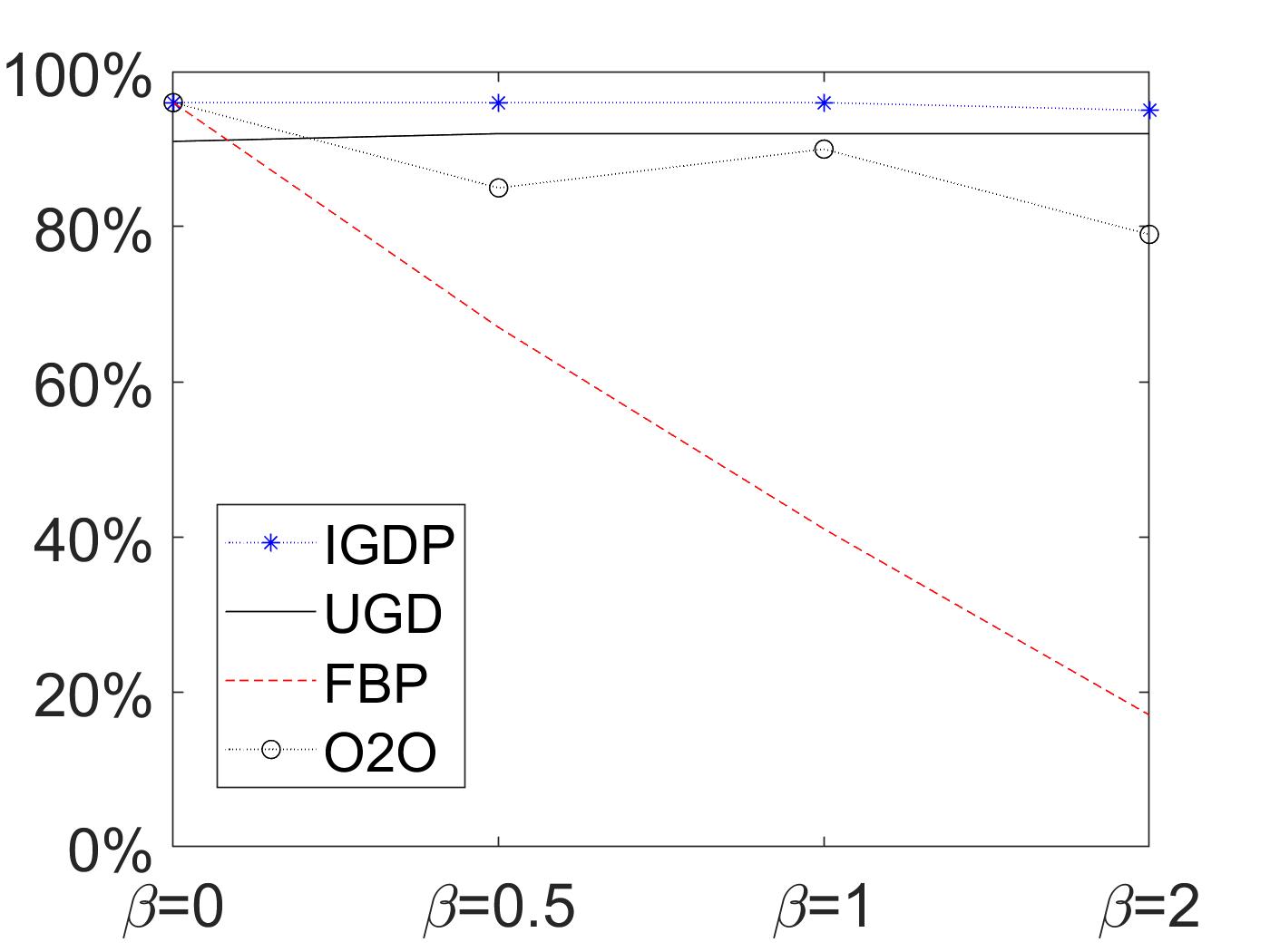}
        \caption{$\alpha=2$ for the normal setting}
    \end{subfigure}
    \begin{subfigure}[b]{0.5\textwidth}
        \centering
        \includegraphics[width=0.9\textwidth]{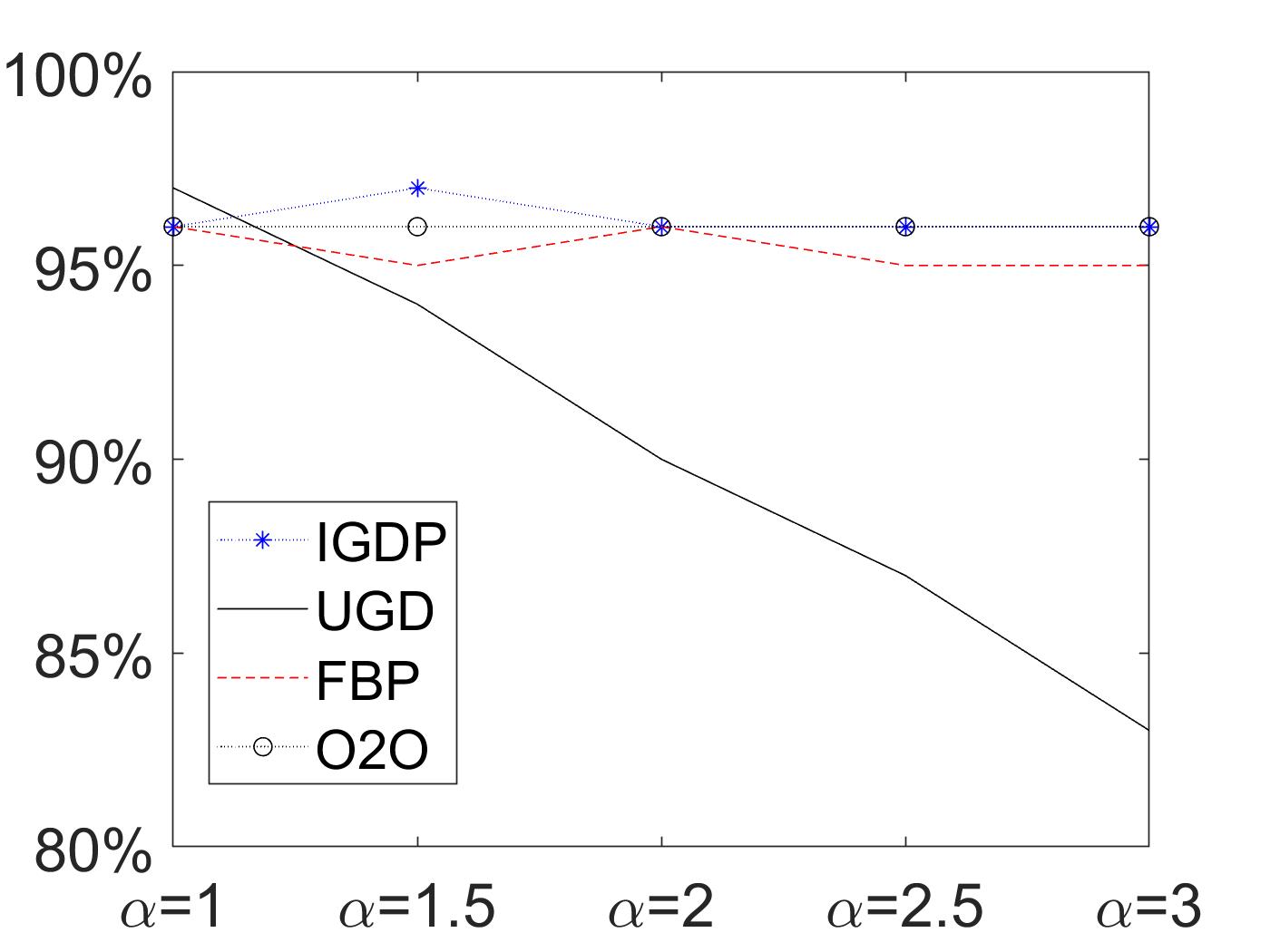}
        \caption{$\beta=0$ for the mixed setting}
    \end{subfigure}%
    ~
    \begin{subfigure}[b]{0.5\textwidth}
        \centering
        \includegraphics[width=0.9\textwidth]{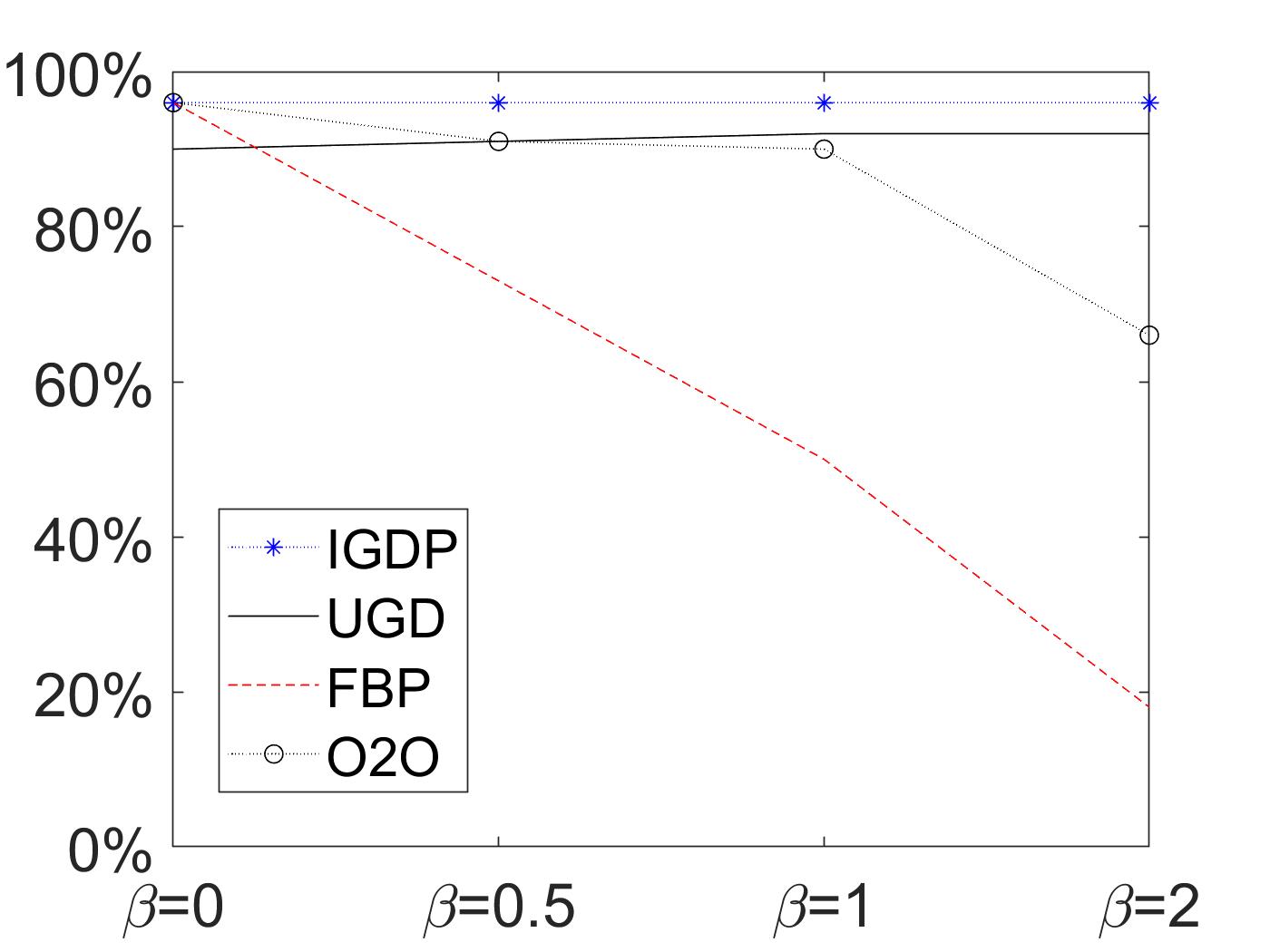}
        \caption{$\alpha=2$ for the mixed setting}
    \end{subfigure}
    \caption{The percentage of the total reward collected by IGDP, UGD, O2O and FBP over the upper bound for different $\alpha$ and different $\beta$ under different setting.}\label{numericalfigure1}
\end{figure*}

In Figure \ref{numericalfigure1}, the case of fixed $\alpha$ and varying $\beta$ examines the robustness of an algorithm to prior estimation error (different $\beta$) under a nonstationary environment ($\alpha>0$). In comparison, Algorithm UGD and IGDP have significant better performance than FBP and O2O. This shows the effectiveness of the gradient-based dynamic update and distinguishes our algorithms of UGD and IGDP from static policies such as FBP and O2O. Specifically, the performances of FBP and O2O deteriorate quickly as the estimation error $\beta$ increases, while both gradient-based algorithms remain stable. Note that FBP computes a fixed dual price vector and O2O constructs a set of fixed dual vectors completely based on the prior estimates. However, as the estimation errors increase, the gap between the ``true'' dual vector and the pre-computed dual vector(s) in FBP or O2O will also increase, and the static nature of the two algorithms prevents a dynamic correction of the dual vector. As a result, the performances of FBP and O2O become worse as the estimation error increases. In contrast, the dynamic dual update of UGD and IGDP naturally combines the demand realization and the prior estimates (through parameters $\hat{\bm{\gamma}}$). Therefore, the performance of IGDP remains stable as the estimation errors increase. These numerical evidence reinforces our previous theoretical arguments on the sub-optimality of static policies in \Cref{sec:staticpolicy}.

In Figure \ref{numericalfigure1}, the case of varying $\alpha$ and $\beta=0$ corresponds to the setting of Section \ref{WBNB} where the value of $\alpha$ indicates the intensity of the non-stationarity for the underlying distribution $\mathcal{P}_t$'s. The performance of UGD deteriorates as $\alpha$ increases, and thus it validates the role of WBNB in characterizing the algorithm performance. The algorithms IGDP, O2O and FBP provide unsurprisingly better and more stable performance because all of them utilize the prior estimate (which is exactly the true distribution when $\beta=0$). This comparison highlights the effectiveness of WBNB in characterizing the learnability of a non-stationary environment when there is no prior knowledge available (as for UGD), and also it underscores the usefulness of prior knowledge.

Based on these experiments, we make the following remarks: First, both of FBP and O2O algorithms can deal with the non-stationarity to some extent (as in Figure \ref{numericalfigure1}(a)). However, their performances highly depends on the accuracy of the prior estimates. If the deviation of the prior estimates from the true distributions is large, both performances can become very poor. Second, the performance of UGD is robust to the deviation of prior estimate since it does not utilize any prior estimate. Meanwhile, the downside is that UGD is more sensitive to the intensity of non-stationarity compared to the other two algorithms. Third, the performance of IGDP is relatively more robust to both the non-stationarity and the deviation of the prior estimates. Specifically, IGDP obtains parameters $\bm{\gamma}_t$ for each $t$ to handle the non-stationarity and also utilizes the gradient descent updates in a dynamic way to hedge against the deviation. In this light, IGDP combines both the advantages of FBP, UGD and O2O.

\subsection{Experiment II: Network Revenue Management with Resolving Heuristics}
Our second numerical experiment is adapted from the network revenue management experiment in \cite{jasin2015performance}. We consider a hub-and-spoke model with $8$ cities, $14$ connecting flights, and $41$ itineraries. The detailed itinerary structure and the normalized capacities in our experiment are referred to Table $1$ and Table $2$ in \cite{jasin2015performance}, respectively. To model the non-stationarity, we add random noises to the arrival probabilities of the itineraries. Specifically, we denote $\mathcal{P}_0\in\mathbb{R}^{41}$ as the arrival probabilities in \cite{jasin2015performance} and we set $\alpha\in[0,1]$ as a parameter, which will be specified later. For each $t$, the arrival probabilities are set as
\[
\mathcal{P}_t=(\mathcal{P}_0+\alpha\cdot \text{Unif}(41,1))/\text{sum}(\mathcal{P}_0+\alpha\cdot \text{Unif}(41,1))
\]
where $\text{Unif}(41,1)$ denotes a $41$-dimensional vector of i.i.d. uniformly distributed random variables over $[0,1]$ and $\text{sum}(\cdot)$ denotes the summation of all the components of a vector. For each time $t$, an independent random noise vector will be generated. In our experiment, the probability vector $\mathcal{P}_t$ is only generated once at the very beginning, and it will remain the same in all the simulation trials. Intuitively, the value of $\alpha$ represents the magnitude of adjustment compared to the original probability $\mathcal{P}_0$, and thus it reflects the intensity of non-stationarity of the underlying distributions. Similarly, we introduce another parameter $\beta$ to reflect the deviation of the prior estimates $\hat{\mathcal{P}}_t$ from the true distribution $\mathcal{P}_t$. Let
\[
\hat{\mathcal{P}}_t=(\mathcal{P}_t+\beta\cdot \text{Unif}(41,1))/\text{sum}(\mathcal{P}_t+\beta\cdot \text{Unif}(41,1)).
\]
We report the algorithm performance for different $(\alpha, \beta )$ in Table \ref{numericaltable3}.

In the table, we also implement a re-solving version of IGD($\hat{\bm{\gamma}}$) which periodically  invokes a re-solving procedure based on the remaining resource and the prior estimate to the dual updates in IGD($\hat{\bm{\gamma}}$). Specifically, the re-solving heuristic re-solves the upper bound function $\hat{L}(\cdot)$ based on the current time period $t$ and the remaining budget $\bm{c}_t$ to obtain an updated $\bm{p}_t$ on a regular basis. For each time $t\in\mathcal{T}\subset\{1,...,T\},$ we solve the following problem
\[
\hat{\bm{p}}_t^*=\text{argmin}_{\bm{p}\geq0} \bm{c}_t^\top\bm{p}+\sum_{j=t}^{T}\hat{\mathcal{P}}_j h(\bm{p},\bm{\theta})
\]
and use its optimal solution as the dual price for the $t$-th time period. Also, we update $\hat{\bm{\gamma}}_t$'s accordingly. Here the set $\mathcal{T}$ contains the time periods that we will re-solve the problem. For each time $t\in\mathcal{T}$, the preceding dual price $\bm{p}_{t-1}$ will be discarded and after time period $t$, we will continue to implement the gradient-based update as in IGD($\hat{\bm{\gamma}}$) until the next re-solving time. The re-solving method has gain great popularity in both theory and application \citep{gallego2019revenue}, and in this experiment, we investigate how the technique can be used to further boost the performance of IGD($\hat{\bm{\gamma}}$).

Table \ref{numericaltable3} reports the performance of IGD($\hat{\bm{\gamma}}$) and its re-solving heuristics with frequency $k=1,50,100,200.$ First, we note that the original version of IGD($\hat{\bm{\gamma}}$) exhibits stably well performance across different combinations of $\alpha$ (the non-stationary intensity) and $\beta$ (the estimation error). Second, the re-solving heuristic further boosts the performance of IGD($\hat{\bm{\gamma}}$). We observe an improved performance even when the frequency $k=200$ (which means only re-solving for $5$ times throughout the horizon), but the performance improvement becomes marginal as we further increase the re-solving frequency. From a computational perspective, the re-solving heuristic for IGD($\hat{\bm{\gamma}}$) provides a good trade-off between computational efficiency and algorithm performance, especially for large-scale system. It naturally blends the computational advantage of IGD($\hat{\bm{\gamma}}$) and the algorithmic adaptivity of the re-solving technique. In recent years, the literature has developed a good understanding of both the advantages and drawbacks of the re-solving technique in a stationary environment (See \citep{cooper2002asymptotic, reiman2008asymptotically, jasin2012re, jasin2015performance, bumpensanti2020re} among others). In parallel, the experiment here raises an interesting but challenging future direction of understanding the re-solving technique in a non-stationary environment.

\DoubleSpacedXI
\begin{table}[ht!]
  \centering
  \begin{tabular}{|c|c|c|c|c|c|}
    \hline
     \multicolumn{2}{|c|}{}  & $\alpha=0$ & $\alpha=0.01$ & $\alpha=0.02$ & $\alpha=0.04$ \\
     \hline
     \multicolumn{2}{|c|}{Upper Bound} & $537000$ & $550158$ & $559358$ & $571775$\\
     \hline
    \multirow{5}{*}{$\beta=0$} & IGDP & $520173 (96.9\%)$ & $528884 (96.1\%)$  & $535678 (95.8\%)$ &  $542302 (94.9\%)$  \\
    \cline{2-6}
    & Re-solve ($200$) & $524233 (97.6\%)$ & $536703 (97.6\%)$ & $543961 (97.3\%)$ & $555291 (97.1\%)$ \\
    \cline{2-6}
    & Re-solve ($100$) & $523685 (97.5\%)$ & $536427 (97.5\%)$ & $544318 (97.3\%)$ & $556507 (97.3\%)$\\
    \cline{2-6}
    & Re-solve ($50$) & $523700 (97.5\%)$ & $537134 (97.6\%)$ & $544105 (97.3\%)$ & $555119 (97.1\%)$ \\
    \cline{2-6}
     & Re-solve ($1$) & $527839 (98.3\%)$  & $540176 (98.2\%)$ & $547443 (97.9\%)$  &  $558640 (97.7\%)$ \\
     \hline
    \hline
    \multirow{5}{*}{$\beta=0.01$} & IGDP & $518718 (96.6\%)$ &  $528611 (96.1\%)$ & $534221 (95.6\%)$  & $543138 (95.0\%)$  \\
    \cline{2-6}
    & Re-solve ($200$) & $524267 (97.6\%)$ & $536168 (97.6\%)$ & $545057 (97.5\%)$ & $555260 (97.1\%)$ \\
    \cline{2-6}
    & Re-solve ($100$) & $524438 (97.7\%)$ & $537389 (97.7\%)$ & $544832 (97.5\%)$ & $557295 (97.5\%)$\\
    \cline{2-6}
    & Re-solve ($50$) & $525178 (97.8\%)$ & $537462 (97.7\%)$ & $545845 (97.7\%)$ & $556276 (97.3\%)$ \\
    \cline{2-6}
     & Re-solve ($1$) & $529752 (98.7\%)$  & $541583 (98.5\%)$ & $549673 (98.3\%)$  & $560308 (98.0\%)$  \\
     \hline
    \hline
    \multirow{5}{*}{$\beta=0.02$} & IGDP & $518287 (96.5\%)$ & $527937 (96.0\%)$  & $533538 (95.5\%)$  & $545654 (95.5\%)$  \\
    \cline{2-6}
    & Re-solve ($200$) & $523689 (97.5\%)$ & $534737 (97.3\%)$ & $543426 (97.2\%)$ & $556498 (97.4\%)$ \\
    \cline{2-6}
    & Re-solve ($100$) & $525031 (97.8\%)$ & $537302 (97.7\%)$ & $545829 (97.7\%)$ & $555397 (97.3\%)$\\
    \cline{2-6}
    & Re-solve ($50$) & $525356 (97.8\%)$ & $537465 (97.8\%)$ & $544992 (97.5\%)$ & $555949 (97.3\%)$   \\
    \cline{2-6}
     & Re-solve ($1$) & $529732 (98.7\%)$ & $541779 (98.5\%)$ & $550519 (98.5\%)$ & $560719 (98.2\%)$ \\
     \hline
    \hline
    \multirow{5}{*}{$\beta=0.04$} & IGDP & $518405 (96.5\%)$ & $528083 (96.0\%)$ & $536737 (96.0\%)$ & $545405 (95.5\%)$ \\
    \cline{2-6}
    & Re-solve ($200$) & $524127 (97.6\%)$ & $535318 (97.4\%)$ & $544818 (97.5\%)$ & $554169 (97.0\%)$\\
    \cline{2-6}
    & Re-solve ($100$) & $523991 (97.6\%)$ & $536770 (97.6\%)$ & $544874 (97.5\%)$ & $556560 (97.5\%)$\\
    \cline{2-6}
    & Re-solve ($50$) & $525299 (97.8\%)$ & $536481 (97.6\%)$ & $545335 (97.6\%)$ & $556776 (97.5\%)$ \\
    \cline{2-6}
     & Re-solve ($1$) & $530349 (98.8\%)$ & $541633 (98.5\%)$ & $550063 (98.4\%)$ & $560365 (98.1\%)$ \\
     \hline
    \hline
  \end{tabular}
  \caption{The numerical experiment based on the network revenue management problem (\cite{jasin2015performance}). The results are obtained for $T=1000$ and $100$ trails. Re-solve ($a$) denotes the frequency of re-solving, namely, re-solving the upper bound problem to update dual variable and $\hat{\bm{\gamma}}_t$ for every $a$ periods. For each entry b(c) of the table, b denotes the expected reward collected by the algorithm and c denotes the percentage of the expected reward of the algorithm over the upper bound.}\label{numericaltable3}
\end{table}
\OneAndAHalfSpacedXI






\ \

\bibliographystyle{ormsv080} 
\bibliography{myreferences} 

\begin{thebibliography}{78}
\expandafter\ifx\csname natexlab\endcsname\relax\def\natexlab#1{#1}\fi
\expandafter\ifx\csname url\endcsname\relax
  \def\url#1{{\tt #1}}\fi
\expandafter\ifx\csname urlprefix\endcsname\relax\def\urlprefix{URL }\fi
\expandafter\ifx\csname urlstyle\endcsname\relax
  \expandafter\ifx\csname doi\endcsname\relax
  \def\doi#1{doi:\discretionary{}{}{}#1}\fi \else
  \expandafter\ifx\csname doi\endcsname\relax
  \def\doi{doi:\discretionary{}{}{}\begingroup \urlstyle{rm}\Url}\fi \fi

\bibitem[{Adelman(2007)}]{adelman2007dynamic}
Adelman, Daniel. 2007.
\newblock Dynamic bid prices in revenue management.
\newblock {\it Operations Research\/} {\bf 55}(4) 647--661.

\bibitem[{Agrawal and Devanur(2014{\natexlab{a}})}]{agrawal2014bandits}
Agrawal, Shipra, Nikhil~R Devanur. 2014{\natexlab{a}}.
\newblock Bandits with concave rewards and convex knapsacks.
\newblock {\it Proceedings of the fifteenth ACM conference on Economics and
  computation\/}. 989--1006.

\bibitem[{Agrawal and Devanur(2014{\natexlab{b}})}]{agrawal2014fast}
Agrawal, Shipra, Nikhil~R Devanur. 2014{\natexlab{b}}.
\newblock Fast algorithms for online stochastic convex programming.
\newblock {\it Proceedings of the twenty-sixth annual ACM-SIAM symposium on
  Discrete algorithms\/}. SIAM, 1405--1424.

\bibitem[{Agrawal et~al.(2014)Agrawal, Wang, and Ye}]{agrawal2014dynamic}
Agrawal, Shipra, Zizhuo Wang, Yinyu Ye. 2014.
\newblock A dynamic near-optimal algorithm for online linear programming.
\newblock {\it Operations Research\/} {\bf 62}(4) 876--890.

\bibitem[{Arjovsky et~al.(2017)Arjovsky, Chintala, and
  Bottou}]{arjovsky2017wasserstein}
Arjovsky, Martin, Soumith Chintala, L{\'e}on Bottou. 2017.
\newblock Wasserstein generative adversarial networks.
\newblock {\it Proceedings of the 34th International Conference on Machine
  Learning-Volume 70\/}. 214--223.

\bibitem[{Arlotto and Gurvich(2019)}]{arlotto2019uniformly}
Arlotto, Alessandro, Itai Gurvich. 2019.
\newblock Uniformly bounded regret in the multisecretary problem.
\newblock {\it Stochastic Systems\/} .

\bibitem[{Arlotto and Xie(2020)}]{arlotto2020logarithmic}
Arlotto, Alessandro, Xinchang Xie. 2020.
\newblock Logarithmic regret in the dynamic and stochastic knapsack problem
  with equal rewards.
\newblock {\it Stochastic Systems\/} .

\bibitem[{Asadpour et~al.(2020)Asadpour, Wang, and Zhang}]{asadpour2020online}
Asadpour, Arash, Xuan Wang, Jiawei Zhang. 2020.
\newblock Online resource allocation with limited flexibility.
\newblock {\it Management Science\/} {\bf 66}(2) 642--666.

\bibitem[{Badanidiyuru et~al.(2013)Badanidiyuru, Kleinberg, and
  Slivkins}]{badanidiyuru2013bandits}
Badanidiyuru, Ashwinkumar, Robert Kleinberg, Aleksandrs Slivkins. 2013.
\newblock Bandits with knapsacks.
\newblock {\it 2013 IEEE 54th Annual Symposium on Foundations of Computer
  Science\/}. IEEE, 207--216.

\bibitem[{Balseiro et~al.(2020)Balseiro, Lu, and Mirrokni}]{balseiro2020best}
Balseiro, Santiago, Haihao Lu, Vahab Mirrokni. 2020.
\newblock The best of many worlds: Dual mirror descent for online allocation
  problems.
\newblock {\it arXiv preprint arXiv:2011.10124\/} .

\bibitem[{Banerjee and Freund(2020{\natexlab{a}})}]{fruend2020a}
Banerjee, Siddhartha, Daniel Freund. 2020{\natexlab{a}}.
\newblock Good prophets know when the end is near.
\newblock {\it Available at SSRN 3479189\/} .

\bibitem[{Banerjee and Freund(2020{\natexlab{b}})}]{fruend2020b}
Banerjee, Siddhartha, Daniel Freund. 2020{\natexlab{b}}.
\newblock Uniform loss algorithms for online stochastic decision-making with
  applications to bin packing.
\newblock {\it SIGMETRICS\/} .

\bibitem[{Bemporad and Morari(1999)}]{bemporad1999robust}
Bemporad, Alberto, Manfred Morari. 1999.
\newblock Robust model predictive control: A survey.
\newblock {\it Robustness in identification and control\/}. Springer, 207--226.

\bibitem[{Berbeglia and Joret(2020)}]{berbeglia2020assortment}
Berbeglia, Gerardo, Gwena{\"e}l Joret. 2020.
\newblock Assortment optimisation under a general discrete choice model: A
  tight analysis of revenue-ordered assortments.
\newblock {\it Algorithmica\/} {\bf 82}(4) 681--720.

\bibitem[{Besbes et~al.(2014)Besbes, Gur, and Zeevi}]{besbes2014stochastic}
Besbes, Omar, Yonatan Gur, Assaf Zeevi. 2014.
\newblock Stochastic multi-armed-bandit problem with non-stationary rewards.
\newblock {\it Advances in neural information processing systems\/}. 199--207.

\bibitem[{Besbes et~al.(2015)Besbes, Gur, and Zeevi}]{besbes2015non}
Besbes, Omar, Yonatan Gur, Assaf Zeevi. 2015.
\newblock Non-stationary stochastic optimization.
\newblock {\it Operations research\/} {\bf 63}(5) 1227--1244.

\bibitem[{Besbes and Zeevi(2012)}]{besbes2012blind}
Besbes, Omar, Assaf Zeevi. 2012.
\newblock Blind network revenue management.
\newblock {\it Operations research\/} {\bf 60}(6) 1537--1550.

\bibitem[{Blanchet et~al.(2016)Blanchet, Gallego, and
  Goyal}]{blanchet2016markov}
Blanchet, Jose, Guillermo Gallego, Vineet Goyal. 2016.
\newblock A markov chain approximation to choice modeling.
\newblock {\it Operations Research\/} {\bf 64}(4) 886--905.

\bibitem[{Blanchet et~al.(2019)Blanchet, Kang, and Murthy}]{blanchet2019robust}
Blanchet, Jose, Yang Kang, Karthyek Murthy. 2019.
\newblock Robust wasserstein profile inference and applications to machine
  learning.
\newblock {\it Journal of Applied Probability\/} {\bf 56}(3) 830--857.

\bibitem[{Buchbinder and Naor(2009)}]{buchbinder2009online}
Buchbinder, Niv, Joseph Naor. 2009.
\newblock Online primal-dual algorithms for covering and packing.
\newblock {\it Mathematics of Operations Research\/} {\bf 34}(2) 270--286.

\bibitem[{Bumpensanti and Wang(2020)}]{bumpensanti2020re}
Bumpensanti, Pornpawee, He~Wang. 2020.
\newblock A re-solving heuristic with uniformly bounded loss for network
  revenue management.
\newblock {\it Management Science\/} .

\bibitem[{Cheung et~al.(2020)Cheung, Lyu, Teo, and Wang}]{cheung2020online}
Cheung, Wang~Chi, Guodong Lyu, Chung-Piaw Teo, Hai Wang. 2020.
\newblock Online planning with offline simulation.
\newblock {\it Available at SSRN 3709882\/} .

\bibitem[{Cheung et~al.(2019)Cheung, Simchi-Levi, and Zhu}]{cheung2019non}
Cheung, Wang~Chi, David Simchi-Levi, Ruihao Zhu. 2019.
\newblock Non-stationary reinforcement learning: The blessing of (more)
  optimism.
\newblock {\it Available at SSRN 3397818\/} .

\bibitem[{Cooper(2002)}]{cooper2002asymptotic}
Cooper, William~L. 2002.
\newblock Asymptotic behavior of an allocation policy for revenue management.
\newblock {\it Operations Research\/} {\bf 50}(4) 720--727.

\bibitem[{Davis et~al.(2013)Davis, Gallego, and
  Topaloglu}]{davis2013assortment}
Davis, James, Guillermo Gallego, Huseyin Topaloglu. 2013.
\newblock Assortment planning under the multinomial logit model with totally
  unimodular constraint structures.
\newblock {\it Work in Progress\/} .

\bibitem[{Devanur et~al.(2019)Devanur, Jain, Sivan, and
  Wilkens}]{devanur2019near}
Devanur, Nikhil~R, Kamal Jain, Balasubramanian Sivan, Christopher~A Wilkens.
  2019.
\newblock Near optimal online algorithms and fast approximation algorithms for
  resource allocation problems.
\newblock {\it Journal of the ACM (JACM)\/} {\bf 66}(1) 7.

\bibitem[{Esfahani and Kuhn(2018)}]{esfahani2018data}
Esfahani, Peyman~Mohajerin, Daniel Kuhn. 2018.
\newblock Data-driven distributionally robust optimization using the
  wasserstein metric: Performance guarantees and tractable reformulations.
\newblock {\it Mathematical Programming\/} {\bf 171}(1-2) 115--166.

\bibitem[{Feldman and Topaloglu(2017)}]{feldman2017revenue}
Feldman, Jacob~B, Huseyin Topaloglu. 2017.
\newblock Revenue management under the markov chain choice model.
\newblock {\it Operations Research\/} {\bf 65}(5) 1322--1342.

\bibitem[{Ferguson et~al.(1989)}]{ferguson1989solved}
Ferguson, Thomas~S, et~al. 1989.
\newblock Who solved the secretary problem?
\newblock {\it Statistical science\/} {\bf 4}(3) 282--289.

\bibitem[{Fournier and Guillin(2015)}]{fournier2015rate}
Fournier, Nicolas, Arnaud Guillin. 2015.
\newblock On the rate of convergence in wasserstein distance of the empirical
  measure.
\newblock {\it Probability Theory and Related Fields\/} {\bf 162}(3) 707--738.

\bibitem[{Galichon(2018)}]{galichon2018optimal}
Galichon, Alfred. 2018.
\newblock {\it Optimal transport methods in economics\/}.
\newblock Princeton University Press.

\bibitem[{Gallego and Topaloglu(2014)}]{gallego2014constrained}
Gallego, Guillermo, Huseyin Topaloglu. 2014.
\newblock Constrained assortment optimization for the nested logit model.
\newblock {\it Management Science\/} {\bf 60}(10) 2583--2601.

\bibitem[{Gallego et~al.(2019)Gallego, Topaloglu et~al.}]{gallego2019revenue}
Gallego, Guillermo, Huseyin Topaloglu, et~al. 2019.
\newblock {\it Revenue management and pricing analytics\/}, vol. 209.
\newblock Springer.

\bibitem[{Gallego and Van~Ryzin(1994)}]{gallego1994optimal}
Gallego, Guillermo, Garrett Van~Ryzin. 1994.
\newblock Optimal dynamic pricing of inventories with stochastic demand over
  finite horizons.
\newblock {\it Management science\/} {\bf 40}(8) 999--1020.

\bibitem[{Garivier and Moulines(2008)}]{garivier2008upper}
Garivier, Aur{\'e}lien, Eric Moulines. 2008.
\newblock On upper-confidence bound policies for non-stationary bandit
  problems.
\newblock {\it arXiv preprint arXiv:0805.3415\/} .

\bibitem[{Gibbs and Su(2002)}]{gibbs2002choosing}
Gibbs, Alison~L, Francis~Edward Su. 2002.
\newblock On choosing and bounding probability metrics.
\newblock {\it International statistical review\/} {\bf 70}(3) 419--435.

\bibitem[{Golrezaei et~al.(2014)Golrezaei, Nazerzadeh, and
  Rusmevichientong}]{golrezaei2014real}
Golrezaei, Negin, Hamid Nazerzadeh, Paat Rusmevichientong. 2014.
\newblock Real-time optimization of personalized assortments.
\newblock {\it Management Science\/} {\bf 60}(6) 1532--1551.

\bibitem[{Gupta and Molinaro(2014)}]{gupta2014experts}
Gupta, Anupam, Marco Molinaro. 2014.
\newblock How experts can solve lps online.
\newblock {\it European Symposium on Algorithms\/}. Springer, 517--529.

\bibitem[{Hall and Willett(2013)}]{hall2013dynamical}
Hall, Eric, Rebecca Willett. 2013.
\newblock Dynamical models and tracking regret in online convex programming.
\newblock {\it International Conference on Machine Learning\/}. PMLR, 579--587.

\bibitem[{Hazan(2016)}]{hazan2016introduction}
Hazan, Elad. 2016.
\newblock Introduction to online convex optimization.
\newblock {\it Foundations and Trends in Optimization\/} {\bf 2}(3-4) 157--325.

\bibitem[{Immorlica et~al.(2019)Immorlica, Sankararaman, Schapire, and
  Slivkins}]{immorlica2019adversarial}
Immorlica, Nicole, Karthik~Abinav Sankararaman, Robert Schapire, Aleksandrs
  Slivkins. 2019.
\newblock Adversarial bandits with knapsacks.
\newblock {\it 2019 IEEE 60th Annual Symposium on Foundations of Computer
  Science (FOCS)\/}. IEEE, 202--219.

\bibitem[{Jadbabaie et~al.(2015)Jadbabaie, Rakhlin, Shahrampour, and
  Sridharan}]{jadbabaie2015online}
Jadbabaie, Ali, Alexander Rakhlin, Shahin Shahrampour, Karthik Sridharan. 2015.
\newblock Online optimization: Competing with dynamic comparators.
\newblock {\it Artificial Intelligence and Statistics\/}. PMLR, 398--406.

\bibitem[{Jagabathula(2014)}]{jagabathula2014assortment}
Jagabathula, Srikanth. 2014.
\newblock Assortment optimization under general choice.
\newblock {\it Available at SSRN 2512831\/} .

\bibitem[{Jasin(2015)}]{jasin2015performance}
Jasin, Stefanus. 2015.
\newblock Performance of an lp-based control for revenue management with
  unknown demand parameters.
\newblock {\it Operations Research\/} {\bf 63}(4) 909--915.

\bibitem[{Jasin and Kumar(2012)}]{jasin2012re}
Jasin, Stefanus, Sunil Kumar. 2012.
\newblock A re-solving heuristic with bounded revenue loss for network revenue
  management with customer choice.
\newblock {\it Mathematics of Operations Research\/} {\bf 37}(2) 313--345.

\bibitem[{Jasin and Kumar(2013)}]{jasin2013analysis}
Jasin, Stefanus, Sunil Kumar. 2013.
\newblock Analysis of deterministic lp-based booking limit and bid price
  controls for revenue management.
\newblock {\it Operations Research\/} {\bf 61}(6) 1312--1320.

\bibitem[{Jenatton et~al.(2016)Jenatton, Huang, and
  Archambeau}]{jenatton2016adaptive}
Jenatton, Rodolphe, Jim Huang, C{\'e}dric Archambeau. 2016.
\newblock Adaptive algorithms for online convex optimization with long-term
  constraints.
\newblock {\it International Conference on Machine Learning\/}. PMLR, 402--411.

\bibitem[{Jiang et~al.(2019)Jiang, Wang, and Zhang}]{jiang2019achieving}
Jiang, Jiashuo, Shixin Wang, Jiawei Zhang. 2019.
\newblock Achieving high individual service-levels without safety stock?
  optimal rationing policy of pooled resources.
\newblock {\it Optimal Rationing Policy of Pooled Resources (May 2, 2019).\/} .

\bibitem[{Jiang and Zhang(2020)}]{Jiang2020OnlineRA}
Jiang, Jiashuo, Jiawei Zhang. 2020.
\newblock Online resource allocation with stochastic resource consumption.
\newblock {\it arXiv preprint arXiv:2012.07933\/} .

\bibitem[{Kunnumkal and Talluri(2016)}]{kunnumkal2016piecewise}
Kunnumkal, Sumit, Kalyan Talluri. 2016.
\newblock On a piecewise-linear approximation for network revenue management.
\newblock {\it Mathematics of Operations Research\/} {\bf 41}(1) 72--91.

\bibitem[{Lattimore and Szepesv{\'a}ri(2020)}]{lattimore2020bandit}
Lattimore, Tor, Csaba Szepesv{\'a}ri. 2020.
\newblock {\it Bandit algorithms\/}.
\newblock Cambridge University Press.

\bibitem[{Lecarpentier and Rachelson(2019)}]{lecarpentier2019non}
Lecarpentier, Erwan, Emmanuel Rachelson. 2019.
\newblock Non-stationary markov decision processes, a worst-case approach using
  model-based reinforcement learning.
\newblock {\it Advances in Neural Information Processing Systems\/}.
  7216--7225.

\bibitem[{Li et~al.(2015)Li, Rusmevichientong, and Topaloglu}]{li2015d}
Li, Guang, Paat Rusmevichientong, Huseyin Topaloglu. 2015.
\newblock The d-level nested logit model: Assortment and price optimization
  problems.
\newblock {\it Operations Research\/} {\bf 63}(2) 325--342.

\bibitem[{Li et~al.(2020)Li, Sun, and Ye}]{li2020simple}
Li, Xiaocheng, Chunlin Sun, Yinyu Ye. 2020.
\newblock Simple and fast algorithm for binary integer and online linear
  programming.
\newblock {\it arXiv preprint arXiv:2003.02513\/} .

\bibitem[{Lu et~al.(2020)Lu, Balseiro, and Mirrokni}]{lu2020dual}
Lu, Haihao, Santiago Balseiro, Vahab Mirrokni. 2020.
\newblock Dual mirror descent for online allocation problems.
\newblock {\it arXiv preprint arXiv:2002.10421\/} .

\bibitem[{Ma et~al.(2020)Ma, Rusmevichientong, Sumida, and
  Topaloglu}]{ma2020approximation}
Ma, Yuhang, Paat Rusmevichientong, Mika Sumida, Huseyin Topaloglu. 2020.
\newblock An approximation algorithm for network revenue management under
  nonstationary arrivals.
\newblock {\it Operations Research\/} {\bf 68}(3) 834--855.

\bibitem[{Mangasarian and Shiau(1987)}]{mangasarian1987lipschitz}
Mangasarian, Olvi~L, T-H Shiau. 1987.
\newblock Lipschitz continuity of solutions of linear inequalities, programs
  and complementarity problems.
\newblock {\it SIAM Journal on Control and Optimization\/} {\bf 25}(3)
  583--595.

\bibitem[{Mehta et~al.(2005)Mehta, Saberi, Vazirani, and
  Vazirani}]{mehta2005adwords}
Mehta, Aranyak, Amin Saberi, Umesh Vazirani, Vijay Vazirani. 2005.
\newblock Adwords and generalized on-line matching.
\newblock {\it 46th Annual IEEE Symposium on Foundations of Computer Science
  (FOCS'05)\/}. IEEE, 264--273.

\bibitem[{Miao et~al.(2021)Miao, Wang, and Zhang}]{miao2021general}
Miao, Sentao, Yining Wang, Jiawei Zhang. 2021.
\newblock A general framework for resource constrained revenue management with
  demand learning and large action space.
\newblock {\it Available at SSRN 3841273\/} .

\bibitem[{Molinaro and Ravi(2013)}]{molinaro2013geometry}
Molinaro, Marco, Ramamoorthi Ravi. 2013.
\newblock The geometry of online packing linear programs.
\newblock {\it Mathematics of Operations Research\/} {\bf 39}(1) 46--59.

\bibitem[{Neely and Yu(2017)}]{neely2017online}
Neely, Michael~J, Hao Yu. 2017.
\newblock Online convex optimization with time-varying constraints.
\newblock {\it arXiv preprint arXiv:1702.04783\/} .

\bibitem[{Rangi et~al.(2018)Rangi, Franceschetti, and
  Tran-Thanh}]{rangi2018unifying}
Rangi, Anshuka, Massimo Franceschetti, Long Tran-Thanh. 2018.
\newblock Unifying the stochastic and the adversarial bandits with knapsack.
\newblock {\it arXiv preprint arXiv:1811.12253\/} .

\bibitem[{Rawlings(2000)}]{rawlings2000tutorial}
Rawlings, James~B. 2000.
\newblock Tutorial overview of model predictive control.
\newblock {\it IEEE control systems magazine\/} {\bf 20}(3) 38--52.

\bibitem[{Reiman and Wang(2008)}]{reiman2008asymptotically}
Reiman, Martin~I, Qiong Wang. 2008.
\newblock An asymptotically optimal policy for a quantity-based network revenue
  management problem.
\newblock {\it Mathematics of Operations Research\/} {\bf 33}(2) 257--282.

\bibitem[{Rusmevichientong et~al.(2010)Rusmevichientong, Shen, and
  Shmoys}]{rusmevichientong2010dynamic}
Rusmevichientong, Paat, Zuo-Jun~Max Shen, David~B Shmoys. 2010.
\newblock Dynamic assortment optimization with a multinomial logit choice model
  and capacity constraint.
\newblock {\it Operations research\/} {\bf 58}(6) 1666--1680.

\bibitem[{Russac et~al.(2019)Russac, Vernade, and
  Capp{\'e}}]{russac2019weighted}
Russac, Yoan, Claire Vernade, Olivier Capp{\'e}. 2019.
\newblock Weighted linear bandits for non-stationary environments.
\newblock {\it Advances in Neural Information Processing Systems\/}.
  12040--12049.

\bibitem[{Sun et~al.(2020)Sun, Wang, and Zhou}]{sun2020near}
Sun, Rui, Xinshang Wang, Zijie Zhou. 2020.
\newblock Near-optimal primal-dual algorithms for quantity-based network
  revenue management.
\newblock {\it arXiv preprint arXiv:2011.06327\/} .

\bibitem[{Talluri and Van~Ryzin(1998)}]{talluri1998analysis}
Talluri, Kalyan, Garrett Van~Ryzin. 1998.
\newblock An analysis of bid-price controls for network revenue management.
\newblock {\it Management science\/} {\bf 44}(11-part-1) 1577--1593.

\bibitem[{Talluri and Van~Ryzin(2004)}]{talluri2004revenue}
Talluri, Kalyan, Garrett Van~Ryzin. 2004.
\newblock Revenue management under a general discrete choice model of consumer
  behavior.
\newblock {\it Management Science\/} {\bf 50}(1) 15--33.

\bibitem[{Talluri and Van~Ryzin(2006)}]{talluri2006theory}
Talluri, Kalyan~T, Garrett~J Van~Ryzin. 2006.
\newblock {\it The theory and practice of revenue management\/}, vol.~68.
\newblock Springer Science \& Business Media.

\bibitem[{Vanderbei et~al.(2015)}]{vanderbei2015linear}
Vanderbei, Robert~J, et~al. 2015.
\newblock {\it Linear programming\/}.
\newblock Springer.

\bibitem[{Vera and Banerjee(2020)}]{vera2020bayesian}
Vera, Alberto, Siddhartha Banerjee. 2020.
\newblock The bayesian prophet: A low-regret framework for online decision
  making.
\newblock {\it Management Science\/} .

\bibitem[{Vera et~al.(2019)Vera, Banerjee, and Gurvich}]{vera2019online}
Vera, Alberto, Siddhartha Banerjee, Itai Gurvich. 2019.
\newblock Online allocation and pricing: Constant regret via bellman
  inequalities.
\newblock {\it arXiv preprint arXiv:1906.06361\/} .

\bibitem[{Villani(2008)}]{villani2008optimal}
Villani, C{\'e}dric. 2008.
\newblock {\it Optimal transport: old and new\/}, vol. 338.
\newblock Springer Science \& Business Media.

\bibitem[{Wagener et~al.(2019)Wagener, Cheng, Sacks, and
  Boots}]{wagener2019online}
Wagener, Nolan, Ching-An Cheng, Jacob Sacks, Byron Boots. 2019.
\newblock An online learning approach to model predictive control.
\newblock {\it arXiv preprint arXiv:1902.08967\/} .

\bibitem[{Yi et~al.(2021)Yi, Li, Yang, Xie, Chai, and Johansson}]{yi2021regret}
Yi, Xinlei, Xiuxian Li, Tao Yang, Lihua Xie, Tianyou Chai, Karl Johansson.
  2021.
\newblock Regret and cumulative constraint violation analysis for online convex
  optimization with long term constraints.
\newblock {\it International Conference on Machine Learning\/}. PMLR,
  11998--12008.

\bibitem[{Yuan and Lamperski(2018)}]{yuan2018online}
Yuan, Jianjun, Andrew Lamperski. 2018.
\newblock Online convex optimization for cumulative constraints.
\newblock {\it Advances in Neural Information Processing Systems\/}.
  6137--6146.

\bibitem[{Zhang and Adelman(2009)}]{zhang2009approximate}
Zhang, Dan, Daniel Adelman. 2009.
\newblock An approximate dynamic programming approach to network revenue
  management with customer choice.
\newblock {\it Transportation Science\/} {\bf 43}(3) 381--394.

\end{thebibliography}

\begin{APPENDIX}{Proofs of Lemmas, Propositions and Theorems}

\OneAndAHalfSpacedXI

\renewcommand{\thesubsection}{A\arabic{subsection}}

\section{Proofs of Section \ref{known_distr}}

\subsection{Proof of Lemma \ref{upper1}}

\begin{proof}{Proof:}
Given the realized parameters $\mathcal{H}=(\bm{\theta}_1, \bm{\theta}_2,\dots,\bm{\theta}_T)$, we can denote the offline optimum of \eqref{PCP} as a function of $\mathcal{H}$, namely, $\{\bm{x}^*_t(\mathcal{H})\}_{t=1}^T$. Let $$\tilde{\bm{x}}_t(\bm{\theta})=\mathbb{E}\left[\bm{x}^*_t(\mathcal{H})|\bm{\theta}_t=\bm{\theta}\right]$$
where the conditional expectation is taken with respect to $\mathcal{P}_j$ for $j\neq t$. We show that $\tilde{\bm{x}}_t(\bm{\theta})$ is a feasible solution to \eqref{PUB}. Specifically, note that for each $i=1,...,m$,
\[\begin{aligned}
c_i&\geq\mathbb{E}\left[\sum_{t=1}^{T}g_i(\bm{x}^*_t(\mathcal{H});\bm{\theta}_t)\right]=\sum_{t=1}^{T}\mathbb{E}_{\bm{\theta}_t\sim\mathcal{P}_t}\left[\mathbb{E}[g_i(\bm{x}^*_t(\mathcal{H});\bm{\theta}_t)|\bm{\theta}_t=\bm{\theta}] \right]\\
&\geq\sum_{t=1}^{T}\int_{\bm{\theta}\in\Theta}g_i(\tilde{\bm{x}}_t(\bm{\theta});\bm{\theta})d\mathcal{P}_t(\bm{\theta}) = \sum_{t=1}^T \mathcal{P}_t g_i(\bm{x}_t;\bm{\theta})
\end{aligned}\]
where the first inequality comes from the feasibility of the optimal solution $\bm{x}_t^*(\mathcal{H})$ and the second inequality follows from that the function $g_i(\cdot;\bm{\theta}_t)$ is a convex function for each $i$ and $\bm{\theta}_t\in\Theta$. Thus, $\{\tilde{\bm{x}}_t(\bm{\theta})\}$ is a feasible solution to \eqref{PUB}. Similarly, we can analyze the objective function
\[\begin{aligned}
\mathbb{E}[R^*_T]&=\mathbb{E}\left[\sum_{t=1}^{T}f(\bm{x}^*_t(\mathcal{H});\bm{\theta}_t)\right]=\sum_{t=1}^{T}\mathbb{E}_{\bm{\theta}_t\sim\mathcal{P}_t}\left[\mathbb{E}[f(\bm{x}^*_t(\mathcal{H});\bm{\theta}_t)|\bm{\theta}_t=\bm{\theta}] \right]\\
&\leq\sum_{t=1}^{T}\int_{\bm{\theta}\in\Theta}f(\tilde{\bm{x}}_t(\bm{\theta});\bm{\theta})d\mathcal{P}_t(\bm{\theta})\leq R^{\text{UB}}_T
\end{aligned}\]
where the first inequality follows from that the function $f(\cdot;\bm{\theta})$ is a concave function for any $\bm{\theta}\in\Theta$ and the last inequality comes from the optimality of $R^{\text{UB}}_T$. Thus we complete the proof.
\Halmos
\end{proof}

\subsection{Proof of Proposition \ref{uppernew2}}

\begin{proof}{Proof:}
We first prove that $\bm{p}^*$ is an optimal solution for $L_t$. Note that for each $t$, $L_t(\bm{p})$ is a convex function over $\bm{p}$ and
\[
\nabla L_t(\bm{p}^*)=\bm{\gamma}_t+\mathcal{P}_t\nabla h(\bm{p}^*;\bm{\theta}_t)=\bm{\gamma}_t-\mathcal{P}_t\bm{g}(\bm{x}^*(\bm{\theta}_t);\bm{\theta}_t)
\]
where $\bm{x}^*(\bm{\theta})=\text{argmax}_{\bm{x}\in\mathcal{X}}\{f(\bm{x};\bm{\theta})-(\bm{p}^*)^T\cdot\bm{g}(\bm{x};\bm{\theta})\}$. With the definition of $\bm{\gamma}_t$ in \eqref{new2007}, it follows immediately that
\[
\nabla L_t(\bm{p}^*)=0
\]
which implies that $\bm{p}^*$ is a minimizer of the function $L(\cdot)$ for each $t$. We then prove that $L(\bm{p}^*)=\sum_{t=1}^{T}L_t(\bm{p}^*)$. Define the set of binding constraints $\mathcal{I}_B=\{i: p_i^*>0, i=1,...,m\}$. From the convexity of the function $L(\bm{p})$ over $\bm{p}$, for each $i\in \mathcal{I}_B$, it holds that
\[
0=\nabla_i L(\bm{p}^*)=c_i-\sum_{t=1}^{T}\mathcal{P}_t\bm{g}(\bm{x}^*(\bm{\theta}_t);\bm{\theta}_t)=c_i-\sum_{t=1}^{T}\gamma_{t,i}
\]
Thus, we have that
\[
\bm{c}^\top\cdot\bm{p}^*=\sum_{t=1}^{T}\bm{\gamma}_t^\top\cdot\bm{p}^*
\]
It follows immediately that $L(\bm{p}^*)=\sum_{t=1}^{T}L_t(\bm{p}^*)$.
\Halmos
\end{proof}

\subsection{Discussions on the function $h$}
\label{discussion_h}
The definition of the function $h(\bm{p};\bm{\theta})$, described as follows for completeness, plays a critical role in deriving our results,
\begin{equation}\label{appen1}
h(\bm{p};\bm{\theta})=\max_{\bm{x}\in\mathcal{X}}\left\{ f(\bm{x};\bm{\theta})-\sum_{i=1}^{n} p_i\cdot g_i(\bm{x};\bm{\theta}) \right\}.
\end{equation}
Throughout the paper, we assume that the optimization problem in \eqref{appen1} can be solved efficiently so as to obtain both its optimal solution and optimal objective value. Here, we justify this assumption with a discussion of the computational aspect of solving \eqref{appen1}. Note that when the function $h$ is defined in the setting with random reward and consumption given in \Cref{sec:stochastic} where $f, \bm{g}$ are random functions, $f, \bm{g}$ can be simply replaced by their expectations $\hat{f}, \hat{\bm{g}}$ in the definition of $h$ in \eqref{appen1}. In the following discussion, we use the notations of $f, \bm{g}$ and $\hat{f}, \hat{\bm{g}}$ exchangeably for simplicity.

When \eqref{PCP} refers to the online linear programming problem or price-based network revenue management (NRM) problem as described in \Cref{sec:example}, the optimization problem in \eqref{appen1} is reduced to a simple convex optimization problem and it can indeed be solved in polynomial time by existing methods. We now consider \eqref{appen1} when \eqref{PCP} represents the choice-based NRM problem.

In the choice-based NRM problem, there are $n$ products and each product $i$ is associated with a revenue $r_i$. The definitions of $\bm{x}_t$ and the functions $f, \bm{g}$ are described in \Cref{sec:example}. Given the assortment $s$, the customer with the type ${\theta}_t$ chooses one product $i$ to purchase with a probability $\eta_i(s;{\theta}_t)$, where $\bm{\eta}$ is specified by the choice model of the customer and is assumed to be known to the decision maker. Here, the function $f$ refers to the expected revenue of the assortment $\bm{x}_t$ and the function $g_i$ refers to the probability that a product $i$ is purchased:
\[
f(\bm{x}_t;{\theta}_t)=\sum_{s\in\mathcal{S}}\sum_{j=1}^{n}x_{t,s}\cdot r_j\cdot\eta_j(s;{\theta}_t)\text{~~and~~}g_i(\bm{x}_t;{\theta}_t)=\sum_{s\in\mathcal{S}}x_{t,s}\cdot\eta_i(s;{\theta}_t)~~~\forall i
\]
Then, the optimization problem in \eqref{appen1} could be rewritten in the following equivalent formulation:
\begin{equation}\label{appen2}
\max_{s\in\mathcal{S}\subset\{0,1\}^n}\sum_{i=1}^{n}(r_i-p_i)\cdot\eta_i(s;{\theta}_t)
\end{equation}
With the following choice models to specify the function $\eta_i$'s, the optimization problem in \eqref{appen2} could all be solved efficiently.\\
\textbf{Multinomial logit model (NML)}: \cite{talluri2004revenue} show that without further constraints, it is optimal to sort the products according to a decreasing order $r_i-p_i$ and find an optimal revenue-order assortment among $\{1\}, \{1,2,\},\dots, \{1,2\dots,n\}$. \cite{rusmevichientong2010dynamic} further propose a simple polynomial-time algorithm to compute the optimal assortment where there is a capacity constraint. \cite{davis2013assortment} consider the assortment problem with totally unimodular constraints and show that the constrained problem could be solved as an equivalent linear program.\\
\textbf{Nested logit model}: \cite{li2015d} consider the assortment optimization problem under a $d$-level nested logit model and propose an algorithm to compute the optimal assortment in polynomial time. \cite{gallego2014constrained} consider the setting with cardinality constraint and show that the optimal assortment can be obtained efficiently by solving a linear program. \cite{gallego2014constrained} further consider the setting with a capacity constraint and propose an algorithm with constant competitiveness and good empirical performance.\\
\textbf{Markov chain based models}: \cite{blanchet2016markov} propose the Markov chain based model and develop polynomial-time solution algorithms. \cite{feldman2017revenue} propose a linear programming approach to obtain the optimal assortment under this model. \\
\textbf{General choice model}: There are also papers considering the assortment problem without exploiting the specific structures of the choice model. \cite{jagabathula2014assortment} studies a local search heuristic and shows its great empirical performance. Also, \cite{berbeglia2020assortment} analyze the performance of the revenue-order assortment heuristic under the general discrete choice model, which can be applied to solve \eqref{appen2} approximately.

\subsection{Proof of Lemma \ref{newupperlemma}}

\begin{proof}{Proof:}
Note that the following two properties are satisfied by the update rule \eqref{003}:
\begin{itemize}
  \item[(i).] If $\|\bm{p}_t\|_{\infty}\leq q$, then we must have $\|\bm{p}_{t+1}\|_{\infty}\leq q+1$ by noting that for each $i$, the $i$-th component of $\bm{p}_t$, denoted as $p_{t,i}$, is nonnegative and $g_i(\cdot,\bm{\theta}_t)$ is normalized within $[0,1]$.
  \item[(ii).] If there exists $i$ such that $p_{t,i}>q$, then we must have $p_{t+1,i}<p_{t,i}$. Specifically, when $p_{t,i}>q$, we must have that $g_i(\tilde{\bm{x}}_t;\bm{\theta}_t)=0$, otherwise we would have that
      \[
      f(\tilde{\bm{x}}_t;\bm{\theta}_t)-\bm{p}_t^\top\cdot \bm{g}(\tilde{\bm{x}}_t;\bm{\theta}_t)\leq f(\tilde{\bm{x}}_t;\bm{\theta}_t)-p_{t,i}\cdot g_i(\tilde{\bm{x}}_t;\bm{\theta}_t)<0
      \]
      which contradicts the definition of $\tilde{\bm{x}}_t$ in Algorithm \ref{alg:SOA} since we could always select $\bm{x}_t=\bm{0}$ to obtain a zero objective value as per Assumption \ref{assume}. Then from \eqref{003}, it holds that $p_{t+1,i}<p_{t,i}$.
\end{itemize}
Starting from $\bm{p}_1=\bm{0}$ and iteratively applying the above two property to control the increase of $\bm{p}_t$ from $t=1$ to $T$, we obtain that for the first time that one component of $\bm{p}_t$ exceeds the threshold $q$, it is upper bounded by $q+1$ and this component will continue to decrease until it falls below the threshold $q$. Thus, we have $\|\bm{p}_t\|_{\infty}\leq q+1$ with probability $1$ for each $t$.
\Halmos
\end{proof}

\subsection{Proof of Theorem \ref{newnonstationtheorem}}
\begin{proof}{Proof:}
In IGD($\bm{\gamma}$), the true action $\bm{x}_t$ taken by the decision maker differs from the virtual action $\tilde{\bm{x}}_t$ if and only if $\bm{c}_t$ cannot fully satisfy $\bm{g}(\tilde{\bm{x}}_t;\bm{\theta}_t)$. Thus, we have that
\[
f(\tilde{\bm{x}}_t;\bm{\theta}_t)-f(\bm{x}_t;\bm{\theta}_t)\leq f(\tilde{\bm{x}}_t;\bm{\theta}_t)\cdot \mathbb{I}\left\{\exists i: c_{t,i}<g_i(\tilde{\bm{x}}_t;\bm{\theta}_t)\right\}
\]
where $c_{t,i}$ denotes the $i$-th component of $\bm{c}_t$ and $\mathbb{I}\left\{\cdot\right\}$ denotes the indicator function. Moreover, we know
\[
\mathbb{I}\left\{\exists i: c_{t,i}<g_i(\tilde{\bm{x}}_t;\bm{\theta}_t)\right\}
\leq \sum_{i=1}^{m}\mathbb{I}\left\{\sum_{j=1}^t g_i(\tilde{\bm{x}}_j;\bm{\theta}_j)>c_i\right\}.
\]
Recall that the maximum reward generated by consuming per unit of budget of each constraint is upper bounded by $q$. We have
\[
f(\tilde{\bm{x}}_t;\bm{\theta}_t)\cdot \mathbb{I}\left\{\exists i: c_{t,i}<g_i(\tilde{\bm{x}}_t;\bm{\theta}_t)\right\}
\leq q \sum_{i=1}^{m} g_i(\tilde{\bm{x}}_t;\bm{\theta})\cdot\mathbb{I}\left\{\sum_{j=1}^t g_i(\tilde{\bm{x}}_j;\bm{\theta}_j)>c_i\right\}
\]
From the fact that $\bm{g}(\cdot;\bm{\theta}_t)\in[0,1]^m$,
\[\begin{aligned}
\sum_{t=1}^{T}f(\tilde{\bm{x}}_t;\bm{\theta}_t)-\sum_{t=1}^{T}f(\bm{x}_t;\bm{\theta}_t)&\leq q\cdot\sum_{i=1}^{m}\sum_{t=1}^{T}g_i(\tilde{\bm{x}}_t;\bm{\theta}_t)\cdot\mathbb{I}\left\{\sum_{j=1}^t g_i(\tilde{\bm{x}}_j;\bm{\theta}_j)>c_i\right\} \\
&\leq q\cdot\sum_{i=1}^{m}\left[\sum_{t=1}^{T}g_i(\tilde{\bm{x}}_t;\bm{\theta}_t)-(c_i-1) \right]^{+}
\end{aligned}\]
which related the total collected reward by the true action $\{\bm{x}_t\}_{t=1}^T$ and the virtual action $\{\tilde{\bm{x}}_t\}_{t=1}^T$. Further from Proposition \ref{uppernew2}, we have that
\[\begin{aligned}
\text{Reg}_T(\pi)\leq  \min_{\bm{p}\geq0}\sum_{t=1}^{T} L_t(\bm{p})-\mathbb{E}\left[\sum_{t=1}^{T}f(\bm{x}_t;\bm{\theta}_t) \right]&\leq \underbrace{\sum_{t=1}^{T} \min_{\bm{p}\geq0} L_t(\bm{p})-\mathbb{E}\left[\sum_{t=1}^{T}f(\tilde{\bm{x}}_t;\bm{\theta}_t)\right]}_{\text{I}}\\
&+\underbrace{q\cdot\mathbb{E}\left[\sum_{i=1}^{m}\left[\sum_{t=1}^{T}g_i(\tilde{\bm{x}}_t;\bm{\theta}_t)-(c_i-1) \right]^{+}\right]}_{\text{II}}
\end{aligned}\]
We then bound the term I and term II separately to derive our regret bound.\\
\textbf{Bound I}: 
Note that for each $t$, the distribution of $\bm{p}_t$ is independent from the distribution of $\bm{\theta}_{\tau}$ for any $\tau\leq t$, then we have that
\[
\min_{\bm{p}\geq\bm{0}}L_t(\bm{p})\leq \mathbb{E}_{\bm{p}_t}\left[ L_t(\bm{p}_t) \right] =\mathbb{E}_{\bm{p}_t}\left[  \bm{\gamma}_t^\top\bm{p}_t+\mathcal{P}_th(\bm{p}_t;\bm{\theta}_t)  \right]
\]
where the expectation is taken with respect to the randomness of the dual price $\bm{p}_t.$ Thus, we have
\[
\text{I}\leq \sum_{t=1}^{T}\mathbb{E}_{\bm{p}_t}\left[\bm{\gamma}_t^\top\bm{p}_t+\mathcal{P}_t\left\{h(\bm{p}_t;\bm{\theta}_t)-f(\tilde{\bm{x}}_t;\bm{\theta}_t) \right\} \right]
\]
From the definition of $\tilde{\bm{x}}_t$, we get that $h(\bm{p}_t;\bm{\theta}_t)-f(\tilde{\bm{x}}_t;\bm{\theta}_t)=-\bm{p}_t^\top\cdot \bm{g}(\tilde{\bm{x}}_t;\bm{\theta}_t)$,
which implies that
\[
\text{I}\leq \sum_{t=1}^{T}\mathbb{E}_{\bm{p}_t}\left[\bm{p}_t^\top\cdot\left( \bm{\gamma}_t-\mathcal{P}_t \bm{g}(\tilde{\bm{x}}_t;\bm{\theta}_t)\right)\right]
\]
Note that from the update rule \eqref{003}, we have that
\[
\|\bm{p}_{t+1}\|_2^2\leq\|\bm{p}_t\|_2^2+\frac{1}{T}\cdot\|\bm{g}(\tilde{\bm{x}}_t;\bm{\theta}_t)-\bm{\gamma}_t\|_2^2-\frac{2}{\sqrt{T}}\cdot \bm{p}_t^\top\cdot\left( \bm{\gamma}_t-\bm{g}(\tilde{\bm{x}}_t;\bm{\theta}_t)\right)
\]
which implies that
\[
\mathbb{E}_{\bm{p}_t}\left[\bm{p}_t^\top\cdot\left( \bm{\gamma}_t-\mathcal{P}_t \bm{g}(\tilde{\bm{x}}_t;\bm{\theta}_t)\right)\right]\leq \frac{\sqrt{T}}{2}\cdot\left(\mathbb{E}[\|\bm{p}_t\|_2^2]-\mathbb{E}[\|\bm{p}_{t+1}\|_2^2] \right)+\frac{m}{2\sqrt{T}}
\]
Thus, it holds that
\begin{equation}\label{newregret101}
\text{I}\leq \frac{m\sqrt{T}}{2}
\end{equation}
\textbf{Bound II}: Note that from the update rule \eqref{003}, we have that
\[
\sqrt{T}\cdot\bm{p}_{t+1}\geq \sqrt{T}\cdot\bm{p}_t+\bm{g}(\tilde{\bm{x}}_t;\bm{\theta}_t)-\bm{\gamma}_t
\]
which implies that
\[
\sum_{t=1}^{T}\bm{g}(\tilde{\bm{x}}_t;\bm{\theta}_t)-\bm{c}\leq \sum_{t=1}^{T}\bm{g}(\tilde{\bm{x}}_t;\bm{\theta}_t)-\sum_{t=1}^{T}\bm{\gamma}_t \leq \sqrt{T}\cdot\bm{p}_{T+1}
\]
Thus, it holds that
\begin{equation}\label{newregret102}
\text{II}=q\cdot\mathbb{E}\left[\sum_{i=1}^{m}\left[\sum_{t=1}^{T}g_i(\tilde{\bm{x}}_t;\bm{\theta}_t)-(c_i-1) \right]^{+}\right]\leq mq(q+1)\cdot\sqrt{T}+qm
\end{equation}
We obtain the $O(\sqrt{T})$ regret bound immediately by combining \eqref{newregret101} and \eqref{newregret102}.
\Halmos
\end{proof}

\subsection{Proof of \Cref{prop:infinitesupportlower}}
\begin{proof}{Proof:}
We construct the following distribution $\mathcal{P}$ over the parameter $\theta$. Suppose that there is a single resource. Denote the support set of $\mathcal{P}$ as $\{\theta^{(0)}, \theta^{(1)}, \theta^{(2)}, \dots, \theta^{(j)}, \dots\}$ and for each $j=1,2,\dots$, we have that
\[
\mathcal{P}(\theta=\theta^{(j)})=\frac{1}{2}\cdot \frac{1}{2^j}\text{~~and~~}\mathcal{P}(\theta=\theta^{(0)})=\frac{1}{2}
\]
Moreover, we have the following condition over the support set $\{\theta^{(1)}, \theta^{(2)}, \dots, \theta^{(j)}, \dots\}$ of $\mathcal{P}$. Denote sequences of constants $\{a_j\}$, $\{b_j\}$, where
\[
a_j=1+\frac{2j+1}{2j(j+1)}, ~b_j=1+1/(j+1)
\]
for $j=1,2,....$ Clearly, we have $1\leq\dots\leq a_{j+1}<b_j<a_{j}\dots\leq2$. The reward function satisfies
\[
0\leq\min_{0<x\leq 1}\frac{f(x;\theta^{(0)})}{x}\leq\max_{0<x\leq 1}\frac{f(x;\theta^{(0)})}{x}\leq 1\text{~and~}   b_j\leq\min_{0<x\leq 1}\frac{f(x;\theta^{(j)})}{x}\leq\max_{0<x\leq 1}\frac{f(x;\theta^{(j)})}{x}\leq a_j,~~\forall j=1,2,\dots
\]
There is a single resource and the resource consumption function satisfies $g(x;\theta)=x$ for each $\theta$. We are now ready to construct the example and show the lower bound.

Suppose there is a single resource with an initial capacity $\frac{T}{2}$, where $T$ is the total time periods. For a fixed $T$, there must exists an integer $k$ such that
\[
\frac{1}{2\sqrt{T}}\leq\frac{1}{2^{k+2}}\leq\frac{1}{\sqrt{T}}
\]
Now, we assume without loss of generality that $T$ is set such that
\[
\frac{1}{2^{k+2}}=\frac{1}{\sqrt{T}}
\]
We now divide the support set $\{\theta^{(0)}, \theta^{(1)}, \theta^{(2)}, \dots, \theta^{(j)}, \dots\}$ into three subsets:
\[
A_1=\{\theta^{(j)}: 1\leq j<k\}, ~A_2=\{\theta^{(k)}\}\text{~and~} A_3=\{\theta^{(0)}, \theta^{(j)}: j\ge k+1\}.
\]
Clearly, at each time period $t$, we have that
\[
P(\theta_t\in A_1)=\frac{1}{2}+\frac{1}{2^{k+1}}, P(\theta_t\in A_2)=\frac{1}{2^{k+1}}\text{~and~}P(\theta_t\in A_3)=\frac{1}{2}-\frac{2}{2^{k+1}}.
\]
Denote $\varepsilon=\frac{1}{2^{k+2}}$. We call that we encounter a type $q$ request whenever the event $\{\theta_t\in A_q\}$ happens, where $q=1,2,3$.
We now denote $Z_{q}^t$ as the number type $q$ requests in the first $t$ periods. We next introduce the event:
\[\begin{aligned}
&\mathcal{H}_1=\left\{ \frac{1}{\epsilon}\leq Z_2^{T}\leq \min\{ 2Z_2^{T/2}, \frac{2}{\epsilon} \} \right\}=\left\{ \frac{1}{2}\mathbb{E}[Z_2^T]\leq Z_2^T\leq \min\{2Z_2^{T/2},\mathbb{E}[Z_2^T]  \} \right\}\\
&\mathcal{H}_2=\left\{ Z_1^T\geq\frac{T}{2}+\frac{2}{\epsilon}   \right\}=\left\{ Z_1^T\geq\mathbb{E}[Z_1^T]+\frac{6}{\epsilon}  \right\}\\
&\mathcal{H}_3=\left\{ Z_1^T\leq\frac{T}{2}-\frac{4}{\epsilon} \right\}=\left\{ Z_1^T\leq \mathbb{E}[Z_1^T] \right\}
\end{aligned}\]
It is obvious that on event $\mathcal{H}_2$, particularly on the event $\mathcal{H}_2\cap\mathcal{H}_1$, all the resources will be consumed by type 1 requests by the offline optimum. That is, at each period $t$, the offline optimum will set the decision variable $x_t$ non-zero only when the event $\{\theta_t\in A_1\}$ happens. Similarly, on event $\mathcal{H}_1\cap\mathcal{H}_3$, the offline optimum will set the decision variable $x_t=1$ whenever the event $\{\theta_t\in A_1\}$ or $\{\theta_t\in A_2\}$ happens. Following the same argument as Arlloto and Gurvich (2019), we know that there exists a constant $\alpha_1>0$, independent of $\epsilon$, such that
\[
P(\mathcal{H}_1\cap\mathcal{H}_2)\geq\alpha_1 \text{~~and~~}P(\mathcal{H}_1\cap\mathcal{H}_3)\geq\alpha_1
\]
Next, we consider the difference between the offline optimum and the dynamic programming policy. We denote $S^{T/2}_2$ as the amount of resource consumed by type 2 request during the first $T/2$ periods, and we consider the event
\[
\mathcal{H}_4=\left\{ S^{T/2}_2\geq \frac{Z_2^{T}}{4} \right\}
\]
Now, on the event $\mathcal{H}_4^c\cap\mathcal{H}_1\cap\mathcal{H}_3$, the offline optimum will set $x_t=1$ for all type 1 and type 2 request. On the other hand, the optimal online policy consumes at most  $Z^{T}_2/4$ amount of resource using type 2 requests during the first $T/2$ periods. Note that as induced by the event $\mathcal{H}_1$, at most $Z^T_2/2$ number of type 2 request can arrive during the last $T/2$ periods. Thus, for the optimal online policy, type 2 request can consume at most $\frac{3}{4}\cdot Z^{T}_2$ amount of resource, while the offline optimum consume $Z^T _2$ amount of resource with type 2 request. Further note that the reward/size for type 2 request is at least $b_k$ and the reward/size for type 3 request is at most $a_{k+1}$. This will incur a regret at least
\[
\mathbb{E}[\text{Regret}]\geq \frac{b_{k}-a_{k+1}}{4}\cdot\mathbb{E}[Z^T_2\cdot 1\{ \mathcal{H}_4^c\cap\mathcal{H}_1\cap\mathcal{H}_3 \}]\geq \frac{b_{k}-a_{k+1}}{4\epsilon}\cdot P(\mathcal{H}_4^c\cap\mathcal{H}_1\cap\mathcal{H}_3)
\]
For each arrival sample path that falls in the set $\mathcal{H}_4\cap\mathcal{H}_1\cap\mathcal{H}_3$, we can find another sample path in $\mathcal{H}_4\cap\mathcal{H}_1\cap\mathcal{H}_2$ by keeping the first $T/2$ arrivals unchanged, and for the last $T/2$ periods, replacing at most $6/\epsilon$ type 3 requests with type 1 requests. We denote this resulting set of sample path as $\mathcal{L}$. Note that for each period, the arrival probability for type 1 request and type 3 request are both bounded away from 0. We know that there exists a constant $\alpha_2>0$, independent of $\epsilon$, such that
\[
P(\mathcal{L})\geq\alpha_2\cdot P(\mathcal{H}_4\cap\mathcal{H}_1\cap\mathcal{H}_3).
\]
Moreover, for each sample path in the set $\mathcal{L}$, since it is on the event $\mathcal{H}_4\cap\mathcal{H}_1\cap\mathcal{H}_2$, we know that the offline optimum will consume all the resource with type 1 request, while the optimal online policy consumes at least $\frac{Z_2^{T}}{4}$ amount of resource with type 2 requests. Note that the reward/size for type 1 is at least $b_{k-1}$, while the reward/size for type 2 is at most $a_k$. This will incur a regret at least
\[
\mathbb{E}[\text{regret}]\geq \frac{b_{k-1}-a_{k}}{4}\cdot \mathbb{E}[Z^T_2\cdot 1\{ \mathcal{L} \}]\geq \frac{b_{k-1}-a_{k}}{4\epsilon}\cdot P(\mathcal{L})\geq \alpha_2\cdot\frac{b_{k-1}-a_{k}}{4\epsilon}\cdot P(\mathcal{H}_4\cap\mathcal{H}_1\cap\mathcal{H}_3).
\]
Since
\[
P(\mathcal{H}_1\cap\mathcal{H}_3)=P(\mathcal{H}_4^c\cap\mathcal{H}_1\cap\mathcal{H}_3)+P(\mathcal{H}_4\cap\mathcal{H}_1\cap\mathcal{H}_3)\geq\alpha_1,
\]
we know that there exists a constant $\alpha_3>0$, independent of $\epsilon$, such that
\[
\mathbb{E}[\text{regret}]\geq \frac{\alpha_3}{\epsilon}\cdot\min\{ b_{k-1}-a_{k}, b_{k}-a_{k+1} \}.
\]
Recall that $\epsilon=1/\sqrt{T}$ and
\[
\min\{ b_{k-1}-a_{k}, b_{k}-a_{k+1} \}=\frac{1}{(k+1)(k+2)}
\]
and that $\frac{1}{2^{k+2}}=\frac{1}{\sqrt{T}}$, i.e., $k+2=\frac{\log T}{2}$. We have
\[
\mathbb{E}[\text{regret}]\geq \frac{\alpha_3}{4}\cdot\frac{\sqrt{T}}{(\log T)^2}
\]
which completes our proof.
\end{proof}

\section{Proofs of Section \ref{WBNB_P}}

\renewcommand{\thesubsection}{B\arabic{subsection}}

\subsection{Proof of Theorem \ref{newlowertheorem}}
\begin{proof}{Proof:}
It follows directly from \Cref{prop:infinitesupportlower} that for any policy $\pi$, we have $\text{Reg}_T(\pi)\geq\Omega(\sqrt{T})$. Thus, it is enough to consider the $\Omega(W_T)$ part in the lower bound. We consider the following estimated problem, where the true coefficients in \eqref{PCP} are replaced by the estimates:
\begin{align}
   \max \ \ &  x_1+...+x_{c}+x_{c+1}+...+x_{T}  \label{newappendixeg0} \\
    \text{s.t. }\ & x_1+...+x_{c}+x_{c+1}+...+x_{T} \le c\nonumber \\
    & 0 \le x_t \le 1\ \text{ for } t=1,...,T. \nonumber
\end{align}
where $c=\frac{T}{2}$ and the prior estimate $\hat{\mathcal{P}}_t$ is simply a one-point distribution for each $t$. Now we consider the following two possible true problems, the distributions of which are all one-point distributions and belong to the set $\Xi_P$ with variation budget $W_T$:
\begin{align}
   \max \ \ &  x_1+...+x_{c}+\left(1+\frac{W_T}{T}\right)x_{c+1}+...+\left(1+\frac{W_T}{T}\right)x_{T}  \label{newappendixeg3} \\
    \text{s.t. }\ & x_1+...+x_{c}+x_{c+1}+...+x_{T} \le c\nonumber \\
    & 0 \le x_t \le 1\ \text{ for } t=1,...,T. \nonumber \\
   \max \ \ &  x_1+...+x_{c} +\left(1-\frac{W_T}{T}\right)x_{c+1}+...+\left(1-\frac{W_T}{T}\right)x_{T} \label{newappendixeg4} \\
    \text{s.t. }\ & x_1+...+x_{c}+x_{c+1}+...+x_{T} \le c \nonumber\\
    & 0 \le x_t \le 1\ \text{ for } t=1,...,T.\nonumber
\end{align}
where $c=\frac{T}{2}$. Denote $x^1_t(\pi)$ as the decision of any policy $\pi$ at period $t$ for scenario \eqref{newappendixeg3} and denote $x^2_t(\pi)$ as the decision of policy $\pi$ at period $t$ for scenario \eqref{newappendixeg4}. Further define $T_1(\pi)$ (resp. $T_2(\pi)$) as the expected capacity consumption of policy $\pi$ on scenario \eqref{newappendixeg3} (resp. scenario \eqref{newappendixeg4}) during the first $\frac{T}{2}$ time periods:
\[
T_1(\pi)=\mathbb{E}\left[\sum_{t=1}^{\frac{T}{2}}x^1_t(\pi)\right] \text{~~~and~~~} T_2(\pi)=\mathbb{E}\left[\sum_{t=1}^{\frac{T}{2}}x^2_t(\pi)\right]
\]
Then, we have that
\[
R_T^1(\pi)=\frac{T+W_T}{2}-\frac{W_T}{T}\cdot T_1(\pi)\text{~~~and~~~}R_T^2(\pi)=\frac{T-W_T}{2}+\frac{W_T}{T}\cdot T_2(\pi)
\]
where $R_T^1(\pi)$ (resp. $R_T^2(\pi)$) denotes the expected reward collected by policy $\pi$ on scenario \eqref{newappendixeg3} (resp. scenario \eqref{newappendixeg4}). Thus, the regret of policy $\pi$ on scenario \eqref{newappendixeg3} and \eqref{newappendixeg4} are $\frac{W_T}{T}\cdot T_1(\pi)$ and $W_T-\frac{W_T}{T}\cdot T_2(\pi)$ respectively. Further note that since the implementation of policy $\pi$ at each time period should be only dependent on the historical information and the coefficients in the estimated problem \eqref{newappendixeg0}, we must have $T_1(\pi)=T_2(\pi)$. Thus, we have that
\[
\text{Reg}_T(\pi)\geq\max\left\{\frac{W_T}{T}\cdot T_1(\pi),W_T-\frac{W_T}{T}\cdot T_1(\pi) \right\}\geq\frac{W_T}{2}=\Omega(W_T)
\]
which completes our proof.
\Halmos
\end{proof}

\subsection{Proof of Lemma \ref{wasserlemma}}

\begin{proof}{Proof:}
Due to symmetry, it is sufficient to show that for every $\bm{p}$ such that $\bm{p}\in\Omega_{\bar{p}}$,
\[
L_{\mathcal{Q}_2}(\bm{p})-L_{\mathcal{Q}_1}(\bm{p})\leq \max\{1,\bar{p}\}\cdot W(\mathcal{Q}_1, \mathcal{Q}_2).
\]
Denote $\mathcal{Q}^*_{1,2}$ as the optimal coupling of the distribution $\mathcal{Q}_1$ and $\mathcal{Q}_2$, i.e., the optimal solution to \eqref{wasserstein}, and denote
$$\bm{x}^*(\bm{\theta})=\text{argmax}_{\bm{x}\in\mathcal{X}}\left\{f(\bm{x}; \bm{\theta})-\sum_{i=1}^{m}p_i\cdot g_i(\bm{x};\bm{\theta}) \right\}$$
Then for each $\bm{\theta}_1\in\Theta$, we define
$$
\bm{\hat{x}}(\bm{\theta}_1)=\int_{\bm{\theta}_2\in\Theta}\bm{x}^*(\bm{\theta}_2)\frac{d\mathcal{Q}^*_{1,2}(\bm{\theta}_1,\bm{\theta}_2)}{d\mathcal{Q}_1(\bm{\theta}_1)}
$$
where $\frac{d\mathcal{Q}^*_{1,2}(\bm{\theta}_1,\bm{\theta}_2)}{d\mathcal{Q}_1(\bm{\theta}_1)}$ is the Radon–Nikodym derivative of $\mathcal{Q}^*_{1,2}$ with respect $\mathcal{Q}_1$ and it can be interpreted as the conditional distribution of $\bm{\theta}_2$ given $\bm{\theta}_1.$ Note that from the definition of $\mathcal{Q}^*_{1,2}$, we have that $\int_{\bm{\theta}_2\in\Theta}\frac{d\mathcal{Q}^*_{1,2}(\bm{\theta}_1,\bm{\theta}_2)}{d\mathcal{Q}_1(\bm{\theta}_1)}=1$. Thus, $\bm{\hat{x}}(\bm{\theta}_1)$ is actually a convex combination of $\{\bm{x}^*(\bm{\theta}_2) \}_{\forall \bm{\theta}_2\in\Theta}$. Moreover, from the concavity of $f(\cdot;\bm{\theta}_1)$ and the convexity of $g_i(\cdot;\bm{\theta}_1)$ for each $i$, we have that
\[
f(\bm{\hat{x}}(\bm{\theta}_1);\bm{\theta}_1)\geq \int_{\bm{\theta}_2\in\Theta}f(\bm{x}^*(\bm{\theta}_2);\bm{\theta}_1)\cdot\frac{d\mathcal{Q}^*_{1,2}(\bm{\theta}_1,\bm{\theta}_2)}{d\mathcal{Q}_1(\bm{\theta}_1)}
\]
and
\[
g_i(\bm{\hat{x}}(\bm{\theta}_1);\bm{\theta}_1)\leq \int_{\bm{\theta}_2\in\Theta}g_i(\bm{x}^*(\bm{\theta}_2);\bm{\theta}_1)\cdot\frac{d\mathcal{Q}^*_{1,2}(\bm{\theta}_1,\bm{\theta}_2)}{d\mathcal{Q}_1(\bm{\theta}_1)}
\]
Thus, we have that
\[\begin{aligned}
L_{\mathcal{Q}_1}(\bm{p})&=
\int_{\bm{\theta}_1\in\Theta}\max_{\bm{x}\in\mathcal{X}}\left\{f(\bm{x}; \bm{\theta}_1)-\sum_{i=1}^{m}p_i\cdot g_i(\bm{x};\bm{\theta}_1) \right\}d\mathcal{Q}_1(\bm{\theta}_1)\\
&\geq \int_{\bm{\theta}_1\in\Theta}\left\{f(\bm{\hat{x}}(\bm{\theta}_1); \bm{\theta}_1)-\sum_{i=1}^{m}p_i\cdot g_i(\bm{\hat{x}}(\bm{\theta}_1);\bm{\theta}_1) \right\}d\mathcal{Q}_1(\bm{\theta}_1)\\
&\geq \int_{\bm{\theta}_1\in\Theta}\int_{\bm{\theta}_2\in\Theta}\left\{f(\bm{x}^*(\bm{\theta}_2); \bm{\theta}_1)-\sum_{i=1}^{m}p_i\cdot g_i(\bm{x}^*(\bm{\theta}_2);\bm{\theta}_1) \right\}d\mathcal{Q}^*_{1,2}(\bm{\theta}_1,\bm{\theta}_2)
\end{aligned}\]
Also, note that for any $\bm{\theta}_1, \bm{\theta}_2\in\Theta$, it holds that
\[
f(\bm{x}^*(\bm{\theta}_2); \bm{\theta}_1)-\sum_{i=1}^{m}p_i\cdot g_i(\bm{x}^*(\bm{\theta}_2);\bm{\theta}_1)\geq
f(\bm{x}^*(\bm{\theta}_2); \bm{\theta}_2)-\sum_{i=1}^{m}p_i\cdot g_i(\bm{x}^*(\bm{\theta}_2);\bm{\theta}_2) -\max\{1,\bar{p}\}\cdot(m+1)\rho(\bm{\theta}_1, \bm{\theta}_2)
\]
which follows the definition of $\rho(\bm{\theta}_1, \bm{\theta}_2)$ in \eqref{rho_dist}. Thus, we get that
\[\begin{aligned}
L_{\mathcal{Q}_1}(\bm{p})&\geq\int_{\bm{\theta}_1\in\Theta}\int_{\bm{\theta}_2\in\Theta}\left\{f(\bm{x}^*(\bm{\theta}_2); \bm{\theta}_2)-\sum_{i=1}^{m}p_i\cdot g_i(\bm{x}^*(\bm{\theta}_2);\bm{\theta}_2) \right\}d\mathcal{Q}^*_{1,2}(\bm{\theta}_1,\bm{\theta}_2)\\
&~~-\max\{1,\bar{p}\}\cdot(m+1)\int_{\bm{\theta}_1\in\Theta}\int_{\bm{\theta}_2\in\Theta}\rho(\bm{\theta}_1, \bm{\theta}_2)d\mathcal{Q}^*_{1,2}(\bm{\theta}_1,\bm{\theta}_2)\\
&=\int_{\bm{\theta}_2\in\Theta}\left\{f(\bm{x}^*(\bm{\theta}_2); \bm{\theta}_2)-\sum_{i=1}^{m}p_i\cdot g_i(\bm{x}^*(\bm{\theta}_2);\bm{\theta}_2) \right\}d\mathcal{Q}_2(\bm{\theta}_2)-\max\{1,\bar{p}\}\cdot (m+1) \mathcal{W}(\mathcal{Q}_1,\mathcal{Q}_2)\\
&=L_{\mathcal{Q}_2}(\bm{p})-\max\{1,\bar{p}\}\cdot(m+1) \mathcal{W}(\mathcal{Q}_1,\mathcal{Q}_2)
\end{aligned}\]
where the first equality holds by noting that $\int_{\bm{\theta}_1\in\Theta}d\mathcal{Q}^*_{1,2}(\bm{\theta}_1,\bm{\theta}_2)=d\mathcal{Q}_2(\bm{\theta}_2)$.
\Halmos
\end{proof}

As a remark, we note that the proof of \Cref{wasserlemma} will still go through even when the concavity and convexity of $f(\cdot;\bm{\theta})$ and $g_i(\cdot;\bm{\theta})$ do not hold. To see this, we use $F$ to denote a distribution over the action set $\mathcal{X}$ and accordingly,
\[
\hat{f}(F;\bm{\theta})=\int_{\bm{x}\in\mathcal{X}}\hat{f}(\bm{x};\bm{\theta})dF(\bm{x})\text{~and~}\hat{g}_i(F;\bm{\theta})=\int_{\bm{x}\in\mathcal{X}}g_i(\bm{x};\bm{\theta})dF(\bm{x}).
\]
Then, we denote
\[
\hat{h}(\bm{p};\bm{\theta}) \coloneqq \max_{F}\left\{\hat{f}(F; \bm{\theta})-\bm{p}^\top \hat{\bm{g}}(F;\bm{\theta}) \right\}, ~~\hat{L}_{\mathcal{Q}}(\bm{p})=\mathcal{Q}\hat{h}(\bm{p};\bm{\theta})
\]
and
\[
\hat{\rho}(\bm{\theta}, \bm{\theta}')=\sup_{F}\| (\hat{f}(F;\bm{\theta}), \hat{\bm{g}}(F;\bm{\theta}))-(\hat{f}(F;\bm{\theta}'), \hat{\bm{g}}(F;\bm{\theta}')) \|_{\infty},~\hat{\mathcal{W}}(\mathcal{Q}_1, \mathcal{Q}_2) \coloneqq \inf_{\mathcal{Q}_{1,2} \in \mathcal{J}(\mathcal{Q}_1, \mathcal{Q}_2)} \int \hat{\rho}(\bm{\theta_1},\bm{\theta_2}) d\mathcal{Q}_{1,2}(\bm{\theta_1},\bm{\theta_2}).
\]
Now that $\hat{f}(F;\bm{\theta})$ and $\hat{\bm{g}}(F;\bm{\theta})$ can be regarded as linear functions of $(dF(\bm{x}), \forall \bm{x}\in\mathcal{X})$, which fully characterizes the distribution $F$, we can apply the same procedure as the proof of \Cref{wasserlemma} to show that
\[
\sup_{\bm{p}\in\Omega_{\bar{p}}}\left|\hat{L}_{\mathcal{Q}_1}(\bm{p})-\hat{L}_{\mathcal{Q}_2}(\bm{p}) \right|\leq  \max\{1,\bar{p}\}\cdot (m+1)\hat{\mathcal{W}}(\mathcal{Q}_1, \mathcal{Q}_2).
\]
On the other hand, note that
\[
\hat{h}(\bm{p};\bm{\theta}) \coloneqq \max_{F}\left\{\hat{f}(F; \bm{\theta})-\bm{p}^\top \hat{\bm{g}}(F;\bm{\theta}) \right\}=\max_{\bm{x}\in\mathcal{X}}\left\{f(\bm{x}; \bm{\theta})-\bm{p}^\top \bm{g}(\bm{x};\bm{\theta}) \right\}=h(\bm{p};\bm{\theta})
\]
and $\hat{\rho}(\bm{\theta}, \bm{\theta}')=\rho(\bm{\theta}, \bm{\theta}')$. We know that
\[
\left|\hat{L}_{\mathcal{Q}_1}(\bm{p})-\hat{L}_{\mathcal{Q}_2}(\bm{p}) \right|=\left|L_{\mathcal{Q}_1}(\bm{p})-L_{\mathcal{Q}_2}(\bm{p}) \right|\text{~and~}\hat{\mathcal{W}}(\mathcal{Q}_1, \mathcal{Q}_2)=\mathcal{W}(\mathcal{Q}_1, \mathcal{Q}_2).
\]
Therefore, \eqref{02} can be proved to hold for general $f(\cdot;\bm{\theta})$ and $g_i(\cdot;\bm{\theta})$ without any convexity or concavity structure.

\subsection{Proof of Lemma \ref{upperlemmawp2}}
\begin{proof}{Proof:}
Note that the following two property is satisfied by the update rule \eqref{2003}:
\begin{itemize}
  \item[(i).] If $\|\bm{p}_t\|_{\infty}\leq q$, then we must have $\|\bm{p}_{t+1}\|_{\infty}\leq q+1$ by noting that for each $i$, the $i$-th component of $\bm{p}_t$, denoted as $p_{t,i}$, is nonnegative and $g_i(\cdot,\bm{\theta}_t)$ is normalized within $[0,1]$.
  \item[(ii).] If there exists $i$ such that $p_{t,i}>q$, then we must have $p_{t+1,i}<p_{t,i}$. Specifically, when $p_{t,i}>q$, we must have that $g_i(\tilde{\bm{x}}_t;\bm{\theta}_t)=0$, otherwise we would have that
      \[
      f(\tilde{\bm{x}}_t;\bm{\theta}_t)-\bm{p}_t^\top\cdot \bm{g}(\tilde{\bm{x}}_t;\bm{\theta}_t)\leq f(\tilde{\bm{x}}_t;\bm{\theta}_t)-p_{t,i}\cdot g_i(\tilde{\bm{x}}_t;\bm{\theta}_t)<0
      \]
      which contradicts the definition of $\tilde{\bm{x}}_t$ in IGD($\hat{\bm{\gamma}}$) since we could always select $\tilde{\bm{x}}_t=\bm{0}$ to obtain $0$ in the objective value. Then from the non-negativity of $\hat{\bm{c}}_t$, it holds that $p_{t+1,i}<p_{t,i}$ in \eqref{2003}.
\end{itemize}
Starting from $\bm{p}_1=0$ and iteratively applying the above two property to control the increase of $\bm{p}_t$ from $t=1$ to $T$, we obtain that for the first time that one component of $\bm{p}_t$ exceeds the threshold $q$, it is upper bounded by $q+1$ and this component will continue to decrease until it falls below the threshold $q$. Thus, it is obvious that we have $\|\bm{p}_t\|_{\infty}\leq q+1$ with probability $1$ for each $t$.
\Halmos
\end{proof}

\subsection{Proof of Theorem \ref{uppertheorem1}}
Similar to case of known distribution, we define the following function $\hat{L}_t(\cdot)$, based on the prior estimate $\hat{\mathcal{P}}_t$.
\begin{equation}\label{2008}
\hat{L}_t(\bm{p})\coloneqq \hat{\bm{\gamma}}_t^\top\bm{p}+\hat{\mathcal{P}}_t h(\bm{p};\bm{\theta}).
\end{equation}
Then we have the following relation between $\hat{L}(\cdot)$ and $\hat{L}_t(\cdot).$ As its analysis is identical to Proposition \ref{uppernew2}, we omit its proof for simplicity.

\begin{lemma}\label{upperlemmawp1}
For each $t=1,...,T$, it holds that
\begin{equation}\label{2010}
\hat{\bm{p}}^* \in \text{argmin}_{\bm{p}\geq0}\hat{L}_t(\bm{p})
\end{equation}
where $\hat{\bm{p}}^*$ is defined in \eqref{p_star} as the minimizer of the function $\hat{L}(\cdot)$. Moreover, it holds that
\begin{equation}\label{2011}
\hat{L}(\hat{\bm{p}}^*)=\sum_{t=1}^{T}\hat{L}_t(\hat{\bm{p}}^*).
\end{equation}
\end{lemma}

\bigskip

Now we proof Theorem \ref{uppertheorem1} and the idea of proof is similar to Theorem \ref{newnonstationtheorem}.

\begin{proof}{Proof:} From the proof of Theorem \ref{newnonstationtheorem}, we have
\[\begin{aligned}
\sum_{t=1}^{T}f(\tilde{\bm{x}}_t;\bm{\theta}_t)-\sum_{t=1}^{T}f(\bm{x}_t;\bm{\theta}_t)&\leq q\cdot\sum_{i=1}^{m}\sum_{t=1}^{T}g_i(\tilde{\bm{x}}_t;\bm{\theta}_t)\cdot\mathbb{I}\left\{\sum_{j=1}^t g_i(\tilde{\bm{x}}_j;\bm{\theta}_j)>c_i\right\} \\
&\leq q\cdot\sum_{i=1}^{m}\left[\sum_{t=1}^{T}g_i(\tilde{\bm{x}}_t;\bm{\theta}_t)-(c_i-1) \right]^{+}
\end{aligned}\]
which related the total collected reward by the true action $\{\bm{x}_t\}_{t=1}^T$ and the virtual action $\{\tilde{\bm{x}}_t\}_{t=1}^T$ of the algorithm IGD($\hat{\bm{\gamma}}$). Further from Proposition \ref{uppernew2}, we have that
\[\begin{aligned}
\text{Reg}_T(\pi)\leq  \min_{\bm{p}\geq0} L(\bm{p})-\mathbb{E}\left[\sum_{t=1}^{T}f(\bm{x}_t;\bm{\theta}_t) \right]&\leq \underbrace{ \min_{\bm{p}\geq0} L(\bm{p})-\mathbb{E}\left[\sum_{t=1}^{T}f(\tilde{\bm{x}}_t;\bm{\theta}_t)\right]}_{\text{I}}\\
&+\underbrace{q\cdot\mathbb{E}\left[\sum_{i=1}^{m}\left[\sum_{t=1}^{T}g_i(\tilde{\bm{x}}_t;\bm{\theta}_t)-(c_i-1) \right]^{+}\right]}_{\text{II}}
\end{aligned}\]
We then bound the term I and term II separately to derive our regret bound.\\
\textbf{Bound I}: Note that $\|\hat{\bm{p}}^*\|_{\infty}\leq q$, it follows directly from Lemma \ref{wasserlemma} and Lemma \ref{upperlemmawp1} that
\[
 \min_{\bm{p}\geq\bm{0}} L(\bm{p})\leq L(\hat{\bm{p}}^*)\leq \hat{L}(\hat{\bm{p}}^*)+\max\{q,1\}\cdot(m+1)\cdot W_T=\sum_{t=1}^{T}\hat{L}_t(\hat{\bm{p}}^*)+\max\{q,1\}\cdot(m+1)\cdot W_T
\]
Note that from Lemma \ref{upperlemmawp2}, we have that for each $t$, $\|\bm{p}_t\|_{\infty}\leq (q+1)$ with probability $1$. Further note that for each $t$, the distribution of $\bm{p}_t$ is independent from the distribution of $\bm{\theta}_t$, then from Lemma \ref{wasserlemma} and Lemma \ref{upperlemmawp1}, we have that
\[
\hat{L}_t(\hat{\bm{p}}^*)=\min_{\bm{p}\geq\bm{0}}\hat{L}_t(\bm{p})\leq \mathbb{E}_{\bm{p}_t}\left[ \hat{L}_t(\bm{p}_t) \right] \leq \mathbb{E}_{\bm{p}_t}\left[ \hat{\bm{\gamma}}_t^\top\bm{p}_t+\mathcal{P}_t h(\bm{p}_t;\bm{\theta}_t) \right]+(q+1)(m+1)\cdot W(\mathcal{P}_t,\hat{\mathcal{P}})
\]
where the expectation is taken with respect to the randomness of the dual price $\bm{p}_t.$ Thus, we have
\[
\text{I}\leq \sum_{t=1}^{T}\mathbb{E}_{\bm{p}_t}\left[\hat{\bm{\gamma}}_t^\top\bm{p}_t+\mathcal{P}_t\left\{h(\bm{p}_t;\bm{\theta}_t)-f(\tilde{\bm{x}}_t;\bm{\theta}_t) \right\} \right]+2(q+1)(m+1)\cdot W_T.
\]
From the definition of $\tilde{\bm{x}}_t$, we get that $h(\bm{p}_t;\bm{\theta}_t)-f(\tilde{\bm{x}}_t;\bm{\theta}_t)=-\bm{p}_t^\top\cdot \bm{g}(\tilde{\bm{x}}_t;\bm{\theta}_t)$,
which implies that
\[
\text{I}\leq \sum_{t=1}^{T}\mathbb{E}_{\bm{p}_t}\left[\bm{p}_t^\top\cdot\left( \hat{\bm{\gamma}}_t-\mathcal{P}_t \bm{g}(\tilde{\bm{x}}_t;\bm{\theta}_t)\right)\right]+2(q+1)(m+1)\cdot W_T
\]
Note that from the update rule \eqref{2003}, we have that
\[
\|\bm{p}_{t+1}\|_2^2\leq\|\bm{p}_t\|_2^2+\frac{1}{T}\cdot\|\bm{g}(\tilde{\bm{x}}_t;\bm{\theta}_t)-\hat{\bm{\gamma}}_t\|_2^2-\frac{2}{\sqrt{T}}\cdot \bm{p}_t^\top\cdot\left( \hat{\bm{\gamma}}_t-\bm{g}(\tilde{\bm{x}}_t;\bm{\theta}_t)\right)
\]
which implies that
\[
\mathbb{E}_{\bm{p}_t}\left[\bm{p}_t^\top\cdot\left( \hat{\bm{\gamma}}_t-\mathcal{P}_t \bm{g}(\tilde{\bm{x}}_t;\bm{\theta}_t)\right)\right]\leq \frac{\sqrt{T}}{2}\cdot\left(\mathbb{E}[\|\bm{p}_t\|_2^2]-\mathbb{E}[\|\bm{p}_{t+1}\|_2^2] \right)+\frac{m}{2\sqrt{T}}
\]
Thus, it holds that
\begin{equation}\label{regret101}
\text{I}\leq \frac{m\sqrt{T}}{2}+2(q+1)(m+1)\cdot W_T
\end{equation}
\textbf{Bound II}: Note that from the update rule \eqref{2003}, we have that
\[
\sqrt{T}\cdot\bm{p}_{t+1}\geq \sqrt{T}\cdot\bm{p}_t+\bm{g}(\tilde{\bm{x}}_t;\bm{\theta}_t)-\hat{\bm{\gamma}}_t
\]
which implies that
\[
\sum_{t=1}^{T}\bm{g}(\tilde{\bm{x}}_t;\bm{\theta}_t)-\bm{c}\leq \sum_{t=1}^{T}\bm{g}(\tilde{\bm{x}}_t;\bm{\theta}_t)-\sum_{t=1}^{T}\hat{\bm{\gamma}}_t \leq \sqrt{T}\cdot\bm{p}_{T+1}
\]
Thus, it holds that
\begin{equation}\label{regret102}
\text{II}=q\cdot\mathbb{E}\left[\sum_{i=1}^{m}\left[\sum_{t=1}^{T}g_i(\tilde{\bm{x}}_t;\bm{\theta}_t)-(c_i-1) \right]^{+}\right]\leq mq(q+1)\cdot\sqrt{T}+qm
\end{equation}
We obtain the $O(\max\{\sqrt{T},W_T\})$ regret bound immediately by combining \eqref{regret101} and \eqref{regret102}.
\Halmos
\end{proof}

\section{Proofs of Section \ref{WBNB}}

\renewcommand{\thesubsection}{C\arabic{subsection}}

\subsection{Proof of Proposition \ref{worstCase}}
\begin{proof}{Proof:}
We consider the implementation of any online policy $\pi$ on the two scenarios \eqref{eg1} and \eqref{eg2} for $\kappa=1$, which is replicated as follows for completeness:
\begin{align}
   \max \ \ &  x_1+...+x_{c}+2x_{c+1}+...+2x_{T}  \label{appendixeg1} \\
    \text{s.t. }\ & x_1+...+x_{c}+x_{c+1}+...+x_{T} \le c\nonumber \\
    & 0 \le x_t \le 1\ \text{ for } t=1,...,T. \nonumber \\
   \max \ \ &  x_1+...+x_{c}  \label{appendixeg2} \\
    \text{s.t. }\ & x_1+...+x_{c}+x_{c+1}+...+x_{T} \le c \nonumber\\
    & 0 \le x_t \le 1\ \text{ for } t=1,...,T.\nonumber
\end{align}
where $c=\frac{T}{2}$. Denote $x^1_t(\pi)$ as the decision of policy $\pi$ at period $t$ for scenario \eqref{appendixeg1} and denote $x^2_t(\pi)$ as the decision of policy $\pi$ at period $t$ for scenario \eqref{appendixeg2}. Further define $T_1(\pi)$ (resp. $T_2(\pi)$) as the expected capacity consumption of policy $\pi$ on scenario \eqref{appendixeg1} (resp. scenario \eqref{appendixeg2}) during the first $\frac{T}{2}$ time periods:
\[
T_1(\pi)=\mathbb{E}\left[\sum_{t=1}^{\frac{T}{2}}x^1_t(\pi)\right] \text{~~~and~~~} T_2(\pi)=\mathbb{E}\left[\sum_{t=1}^{\frac{T}{2}}x^2_t(\pi)\right]
\]
Then, we have that
\[
R_T^1(\pi)=T-T_1(\pi)\text{~~~and~~~}R_T^2(\pi)=T_2(\pi)
\]
where $R_T^1(\pi)$ (resp. $R_T^2(\pi)$) denotes the expected reward collected by policy $\pi$ on scenario \eqref{appendixeg1} (resp. scenario \eqref{appendixeg2}). Thus, the regret of policy $\pi$ on scenario \eqref{appendixeg1} and \eqref{appendixeg2} are $T_1(\pi)$ and $T-T_2(\pi)$ respectively. Further note that since the implementation of policy $\pi$ at each time period should be independent of the future information, we must have $T_1(\pi)=T_2(\pi)$. Thus, we have that
\[
\text{Reg}_T(\pi)\geq\max\{T_1(\pi),T-T_1(\pi) \}\geq\frac{T}{2}=\Omega(T)
\]
which completes our proof.
\Halmos
\end{proof}

\subsection{Proof of Theorem \ref{lowertheorem}}
The proof of the theorem can be directly obtained from Theorem \ref{newlowertheorem}.

\subsection{Proof of Theorem \ref{uppertheorem}}
We first prove the following lemma, which implies that the dual variable updated in \eqref{new2003} is always bounded. It derivation is essentially the same as Lemma \ref{newupperlemma}, so we omit its proof for simplicity.

\begin{lemma}\label{upperlemma}
Under Assumption \ref{assume}, for each $t=1,2,\dots,T$, the dual price vector satisfies $\|\bm{p}_t\|_{\infty}\leq q+1$, where $\bm{p}_t$ is specified by \eqref{new2003} in Algorithm \ref{alg:SOAWOP} and the constant $q$ is defined in Assumption \ref{assume} (c).
\end{lemma}

Now we proceed to prove Theorem \ref{uppertheorem}.

\begin{proof}{Proof:}
From the proof of Theorem \ref{newnonstationtheorem}, we have
\[\begin{aligned}
\sum_{t=1}^{T}f(\tilde{\bm{x}}_t;\bm{\theta}_t)-\sum_{t=1}^{T}f(\bm{x}_t;\bm{\theta}_t)
&\leq q\cdot\sum_{i=1}^{m}\left[\sum_{t=1}^{T}g_i(\tilde{\bm{x}}_t;\bm{\theta}_t)-(c_i-1) \right]^{+}
\end{aligned}\]
which relates the total collected reward by the true action $\{\bm{x}_t\}_{t=1}^T$ and the virtual action $\{\tilde{\bm{x}}_t\}_{t=1}^T$. Here $[\cdot]^+$ denotes the positive part function. Furthermore, from Proposition \ref{uppernew2} and the feasibility, we have that
\[\begin{aligned}
\text{Reg}_T(\pi)\leq  \min_{\bm{p}\geq\bm{0}} L(\bm{p})-\mathbb{E}\left[\sum_{t=1}^{T}f(\bm{x}_t;\bm{\theta}_t) \right]&\leq \underbrace{\min_{\bm{p}\geq\bm{0}} L(\bm{p})-\mathbb{E}\left[\sum_{t=1}^{T}f(\tilde{\bm{x}}_t;\bm{\theta}_t)\right]}_{\text{I}}\\
&+\underbrace{q\cdot\mathbb{E}\left[\sum_{i=1}^{m}\left[\sum_{t=1}^{T}g_i(\tilde{\bm{x}}_t;\bm{\theta}_t)-(c_i-1) \right]^{+}\right]}_{\text{II}}
\end{aligned}\]
Next, we bound the term I and term II separately to derive our regret bound.\\
\textbf{Bound I}: We first define the following function $\bar{L}(\cdot)$:
\[
\bar{L}(\bm{p})\coloneqq \frac{1}{T} \bm{p}^\top\bm{c}+\bar{\mathcal{P}}_Th(\bm{p};\bm{\theta})
\]
Note that $\hat{\mathcal{P}}_T=\frac{1}{T}\sum_{t=1}^{T}\mathcal{P}_t$, it holds that $L(\bm{p})=T\cdot\bar{L}(\bm{p})$ for any $\bm{p}$. From Lemma \ref{upperlemma}, we know that for each $t$, $\|\bm{p}_t\|_{\infty}\leq q+1$ with probability $1$. In addition, for each $t$, the distribution of $\bm{p}_t$ is independent from the distribution of $\bm{\theta}_t$, then from Lemma \ref{wasserlemma}, we have that
\begin{equation}
   \min_{\bm{p}\geq\bm{0}} \bar{L}(\bm{p})\leq \mathbb{E}_{\bm{p}_t}\left[ \bar{L}(\bm{p}_t) \right] \leq \mathbb{E}_{\bm{p}_t}\left[ \frac{1}{T} \bm{p}_t^\top\bm{c}+\mathcal{P}_th(\bm{p}_t;\bm{\theta}_t) \right]+(q+1)(m+1)\cdot \mathcal{W}(\mathcal{P}_t,\bar{\mathcal{P}}_T),
   \label{stepTmp0}
\end{equation}
where the expectation is taken with respect to $\bm{p}_t$ in a random realization of the algorithm.
Thus, we have the first term
\[
\text{I}\leq \sum_{t=1}^{T}\mathbb{E}_{\bm{p}_t}\left[\frac{1}{T}\bm{c}^\top\bm{p}_t+\mathcal{P}_t\left\{h(\bm{p}_t;\bm{\theta}_t)-f(\tilde{\bm{x}}_t;\bm{\theta}_t) \right\} \right]+(q+1)(m+1)\cdot \mathcal{W}(\mathcal{P}_t,\bar{\mathcal{P}}_T)
\]
which comes from combining \eqref{stepTmp0} with the relation $L(\bm{p})=T\cdot\bar{L}(\bm{p})$.
By the definition of $\tilde{\bm{x}}_t$, $h(\bm{p}_t;\bm{\theta}_t)-f(\tilde{\bm{x}}_t;\bm{\theta}_t)=-\bm{p}_t^\top\cdot \bm{g}(\tilde{\bm{x}}_t;\bm{\theta}_t)$,
which implies that
\begin{equation}
\text{I}\leq \sum_{t=1}^{T}\mathbb{E}_{\bm{p}_t}\left[\bm{p}_t^\top\cdot\left( \frac{\bm{c}}{T}-\mathcal{P}_t \bm{g}(\tilde{\bm{x}}_t;\bm{\theta}_t)\right)\right]+(q+1)(m+1)\cdot \mathcal{W}(\mathcal{P}_t,\bar{\mathcal{P}}_T)
\label{stepTmp1}
\end{equation}
Note that from the update rule \eqref{new2003}, we have that
\[
    \|\bm{p}_{t+1}\|_2^2\leq\|\bm{p}_t\|_2^2+\frac{1}{T}\cdot\|\bm{g}(\tilde{\bm{x}}_t;\bm{\theta}_t)-\frac{\bm{c}}{T}\|_2^2-\frac{2}{\sqrt{T}}\cdot \bm{p}_t^\top\cdot\left( \frac{\bm{c}}{T}-\bm{g}(\tilde{\bm{x}}_t;\bm{\theta}_t)\right).
\]
By taking expectation with respect to both sides,
\begin{equation}
   \mathbb{E}_{\bm{p}_t}\left[\bm{p}_t^\top\cdot\left( \frac{\bm{c}}{T}-\mathcal{P}_t \bm{g}(\tilde{\bm{x}}_t;\bm{\theta}_t)\right)\right]\leq \frac{\sqrt{T}}{2}\cdot\left(\mathbb{E}[\|\bm{p}_t\|_2^2]-\mathbb{E}[\|\bm{p}_{t+1}\|_2^2] \right)+\frac{m}{2\sqrt{T}}
\label{stepTmp2}
\end{equation}
Plugging \eqref{stepTmp2} into \eqref{stepTmp1}, we obtain an upper bound on Term I,
\begin{equation}\label{regret01}
\text{I}\leq \frac{m\sqrt{T}}{2}+(q+1)(m+1)\cdot W_T
\end{equation}
\textbf{Bound II}: Note that from the update rule \eqref{new2003}, we have
\[
\sqrt{T}\cdot\bm{p}_{t+1}\geq \sqrt{T}\cdot\bm{p}_t+\bm{g}(\tilde{\bm{x}}_t;\bm{\theta}_t)-\frac{\bm{c}}{T}
\]
Taking a summation with respect to both sides,
\[
\sum_{t=1}^{T}\bm{g}(\tilde{\bm{x}}_t;\bm{\theta}_t)-\bm{c}\leq \sqrt{T}\cdot\bm{p}_{T+1}
\]
Applying Lemma \ref{upperlemma} for a bound on $\bm{p}_{T+1},$ we obtain the upper bound for Term II,
\begin{equation}\label{regret02}
\text{II}=q\cdot\mathbb{E}\left[\sum_{i=1}^{m}\left[\sum_{t=1}^{T}g_i(\tilde{\bm{x}}_t;\bm{\theta}_t)-(c_i-1) \right]^{+}\right]\leq mq(q+1)\cdot\sqrt{T}+qm
\end{equation}
We obtain the desired regret bound by combining \eqref{regret01} and \eqref{regret02}.
\Halmos
\end{proof}

\renewcommand{\thesubsection}{D\arabic{subsection}}

\section{Proofs of Section \ref{sec_extension}}

\renewcommand{\thesubsection}{D\arabic{subsection}}

\subsection{Proof of \Cref{thm:finitesupport}}\label{pf:thmfinite}
\begin{proof}{Proof:}
Note that the regret can be denoted as
\begin{equation}\label{04}
\begin{aligned}
\text{Reg}_T(\mathcal{H}, \pi_{\text{Resolve}})&=\mathbb{E}_{\mathcal{H}\sim\mathcal{P}}[R^*_1(\mathbf{c}_1,\mathcal{H}(1))]-\mathbb{E}_{\mathcal{H}\sim\mathcal{P}}[\sum_{t=1}^{T}r_{t}\cdot x_t]\\
&=\sum_{t=1}^{T}\underbrace{\mathbb{E}_{\mathcal{H}\sim\mathcal{P}}[R^*_t(\mathbf{c}_t, \mathcal{H}(t))-R^*_{t+1}(\mathbf{c}_{t+1},\mathcal{H}(t+1))-r_{t}\cdot x_t]}_{\text{I}_t}
\end{aligned}
\end{equation}
Now we denote $\{x^*_j(\mathbf{c}_t,\mathcal{H}(t))\}$ as one optimal solution to $R^*_t(\mathbf{c}_t,\mathcal{H}(t))$. Clearly, we have the following: if $\bm{\theta}_t$ is realized as $\bm{\theta}^{(j)}$, i.e., $\mathcal{H}_j(t)=\mathcal{H}_j(t+1)+1$ and $\mathcal{H}_{j'}(t)=\mathcal{H}_{j'}(t+1)$ for other $j'$, it holds that
\begin{equation}\label{05}
\begin{aligned}
&R^*_t(\mathbf{c}_t,\mathcal{H}(t))=r_j+R^*_{t+1}(\mathbf{c}_t-\bm{a}_j, \mathcal{H}(t+1)),&&\text{if~}x^*_j(\mathbf{c}_t,\mathcal{H}(t))\geq1\\
&R^*_t(\mathbf{c}_t,\mathcal{H}(t))=R^*_{t+1}(\mathbf{c}_t, \mathcal{H}(t+1)),&&\text{if~}\mathcal{H}_j(t)-x^*_j(\mathbf{c}_t,\mathcal{H}(t))\geq1
\end{aligned}
\end{equation}
We also denote the event $\mathcal{A}_t=\{ x^*_j(\mathbf{c}_t,\mathcal{H}(t))\geq1 \}$ and the event $\mathcal{B}_t=\{ \mathcal{H}_j(t)-x^*_j(\mathbf{c}_t,\mathcal{H}(t))\geq1 \}$.

In the rest of the proof, we mainly analyze the term $\text{I}_{t}$ in \eqref{04}. We first focus on the case when $t\leq T+1-\frac{4(\alpha\beta_2+\beta_2)}{p_{\min}}\cdot W_T-\frac{4}{p_{\min}}$, where $W_T$ is the deviation budget from the estimates $\{\hat{\mathcal{P}}_t\}_{t=1}^T$ to the true distributions $\{\mathcal{P}_t\}_{t=1}^T$, $p_{\min}$ is given in \Cref{assump:finitesupport}, and $\alpha, \beta_2$ are parameters to be determined later.\\ 
\textbf{Case I:} $\hat{x}_{j_t}(\mathbf{c}_t)\geq\frac{1}{2}\cdot \mathbb{E}_{\mathcal{H}\sim\hat{\mathcal{P}}} [\mathcal{H}_{j_t}(t)]$.
Then we have $x_t=1$, and from \eqref{05}, we have
\[\begin{aligned}
\text{I}_t=&\mathbb{E}_{\mathcal{H}\sim\mathcal{P}}[R^*_t(\mathbf{c}_t,\mathcal{H}(t))-R^*_{t+1}(\mathbf{c}_{t+1},\mathcal{H}(t+1))-r_{t}\cdot x_t]\\
=&\mathbb{E}\left[\mathbb{E}[r_{j_t}-r_{j_t}\cdot x_t+R^*_{t+1}(\mathbf{c}_t-\bm{a}_{j_t},\mathcal{H}(t+1))-R^*_{t+1}(\mathbf{c}_t-\bm{a}_{j_t}\cdot x_t,\mathcal{H}(t+1))|\mathcal{A}_t]\right]\\
&+\mathbb{E}\left[\mathbb{E}[-r_{j_t}\cdot x_t+R^*_{t+1}(\mathbf{c}_t,\mathcal{H}(t+1))-R^*_{t+1}(\mathbf{c}_t-\bm{a}_{j_t}\cdot x_t,\mathcal{H}(t+1))|\mathcal{A}^c_t]\right]\\
=&\mathbb{E}\left[\mathbb{E}[-r_{j_t}\cdot x_t+R^*_{t+1}(\mathbf{c}_t,\mathcal{H}(t+1))-R^*_{t+1}(\mathbf{c}_t-\bm{a}_{j_t}\cdot x_t,\mathcal{H}(t+1))|\mathcal{A}^c_t]\right]\\
\leq& P(\mathcal{A}^c_t)
\end{aligned}\]
where $\mathcal{A}^c_t$ denotes the complementary event of $\mathcal{A}_t$. We now bound $P(\mathcal{A}^c_t)$.

From Theorem 2.4 in \citep{mangasarian1987lipschitz}, there exists a constant $\alpha>0$, which depends only on $\{(r_1,\bm{a}_1),\dots,(r_n,\bm{a}_n)\}$, such that
\[
\|\bm{x}^*(\mathbf{c}_t,\mathcal{H}(t))-\hat{\bm{x}}(\mathbf{c}_t)\|_{\infty}\leq\alpha\cdot\|\mathcal{H}(t)-\mathbb{E}_{\mathcal{H}\sim\hat{\mathcal{P}}}[\mathcal{H}(t)]\|_{\infty}.
\]
Moreover, from Hoeffdings' inequality, we know that
\[
P\left(\|\mathcal{H}(t)-\mathbb{E}_{\mathcal{H}\sim\hat{\mathcal{P}}}[\mathcal{H}(t)]\|_{\infty}\leq \frac{p_{\min}\cdot (T-t+1)}{4\alpha}\right)\geq 1-\exp(-\beta_1\cdot (T-t+1))
\]
for some constant $\beta_1>0$, and
\[
\|\mathbb{E}_{\mathcal{H}\sim\hat{\mathcal{P}}}[\mathcal{H}(t)]-\mathbb{E}_{\mathcal{H}\sim\mathcal{P}}[\mathcal{H}(t)]\|_{\infty}\leq \beta_2\cdot W_T.
\]
Thus, we have that
\[
P\left(\|x^*(\mathbf{c}_t,\mathcal{H}(t))-\hat{x}(\mathbf{c}_t)\|_{\infty}\leq \frac{p_{\min}\cdot (T-t+1)}{4}+\alpha\beta_2\cdot W_T\right)\geq1-\exp(-\beta_1\cdot (T-t+1))
\]
Since we have
\[
\hat{x}_{j_t}(\mathbf{c}_t)\geq\frac{1}{2}\cdot \mathbb{E}_{\mathcal{H}\sim\hat{\mathcal{P}}} [\mathcal{H}_{j_t}(t)]\geq \frac{1}{2}\cdot p_{\min}\cdot (T-t+1)-\beta_2\cdot W_T,
\]
we have that
\[
P\left(x^*_{j_t}(\mathbf{c}_t,\mathcal{H}(t))\geq \frac{p_{\min}\cdot (T-t+1)}{4}-(\alpha\beta_2+\beta_2)\cdot W_T\geq 1\right)\geq 1-\exp(-\beta_1\cdot (T-t+1)).
\]
when $t\leq T+1-\frac{4(\alpha\beta_2+\beta_2)}{p_{\min}}\cdot W_T-\frac{4}{p_{\min}}$.

Thus, when $t\leq T+1-\frac{4(\alpha+1)}{p_{\min}}\cdot W_T-\frac{4}{p_{\min}}$, we know that $P(\mathcal{A}^c_t)\leq \exp(-\beta_1\cdot (T-t+1))$.\\
\textbf{Case II:} $\hat{x}_{j_t}(\mathbf{c}_t)<\frac{1}{2}\cdot \mathbb{E}_{\mathcal{H}\sim\hat{\mathcal{P}}} [\mathcal{H}_{j_t}(t)]$. Then we have $x_t=0$, and from \eqref{05}, we have
\[\begin{aligned}
\text{I}_t=&\mathbb{E}[R^*_t(\mathbf{c}_t,\mathcal{H}(t))-R^*_{t+1}(\mathbf{c}_{t+1},\mathcal{H}(t+1))-r_{t}\cdot x_t]\\
=&\mathbb{E}\left[\mathbb{E}[-r_{j_t}\cdot x_t+R^*_{t+1}(\mathbf{c}_t,\mathcal{H}(t+1))-R^*_{t+1}(\mathbf{c}_t-\bm{a}_{j_t}\cdot x_t,\mathcal{H}(t+1))|\mathcal{B}_t]\right]\\
&+\mathbb{E}\left[\mathbb{E}[r_{j_t}-r_{j_t}\cdot x_t+R^*_{t+1}(\mathbf{c}_t-\bm{a}_{j_t},\mathcal{H}(t+1))-R^*_{t+1}(\mathbf{c}_t-\bm{a}_{j_t}\cdot x_t,\mathcal{H}(t+1))|\mathcal{B}_t^c]\right]\\
=&\mathbb{E}\left[\mathbb{E}[r_{j_t}+R^*_{t+1}(\mathbf{c}_t-\bm{a}_{j_t},\mathcal{H}(t+1))-R^*_{t+1}(\mathbf{c}_t,\mathcal{H}(t+1))|\mathcal{B}_t^c]\right]\\
\leq & P(\mathcal{B}^c_t)
\end{aligned}\]
Following the same approach, we have that $P(\mathcal{B}^c_t)\leq\exp(-\beta_1\cdot (T-t+1))$ when $t\leq T+1-\frac{4(\alpha\beta_2+\beta_2)}{p_{\min}}\cdot W_T-\frac{4}{p_{\min}}$.

Thus, on both cases, we conclude that
\[
\text{I}_t\leq \exp(-\beta_1\cdot (T-t+1))
\]
when $t\leq T+1-\frac{4(\alpha\beta_2+\beta_2)}{p_{\min}}\cdot W_T-\frac{4}{p_{\min}}$. Thus, from \eqref{04}, we have that
\[\begin{aligned}
\text{Reg}_T(\mathcal{H}, \pi_{\text{Resolve}})&\leq\sum_{t=1}^{T+1-\frac{4(\alpha\beta_2+\beta_2)}{p_{\min}}\cdot W_T-\frac{4}{p_{\min}}}\text{I}_t+ \frac{4(\alpha\beta_2+\beta_2)}{p_{\min}}\cdot W_T+\frac{4}{p_{\min}}\\
&\leq \sum_{t=1}^T\exp(-\beta_1\cdot(T-t+1))+\frac{4(\alpha\beta_2+\beta_2)}{p_{\min}}\cdot W_T+\frac{4}{p_{\min}}\\
&=\max\{ O(1), O(W_T) \}
\end{aligned}\]
which completes our proof.
\Halmos
\end{proof}

\textbf{Remark:} Note that \citet{vera2020bayesian,fruend2020a,fruend2020b} utilize the condition of $p_{\min}$ in Assumption \ref{assump:finitesupport} to control the non-stationarity of the distribution $\mathcal{P}_t$. To be specific, suppose the support set $\Theta=\{\bm{\theta}^{(1)},\dots,\bm{\theta}^{(n)}\}$ and for each $j=1,\dots,n$, we denote $H_j(t)$ as the number of times that $\bm{\theta}_{\tau}$ is realized as $\bm{\theta}_j$ for $\tau=t,\dots,T$. Denote by $\mathcal{H}(t)=(\mathcal{H}_1(t),\dots,\mathcal{H}_n(t))$. Then this concentration condition of $p_{\min}$ in \citet{vera2020bayesian,fruend2020a,fruend2020b} essentially requires that
\begin{equation}\label{eqn:concentration}
P\left(\|\mathcal{H}(t)-\mathbb{E}[\mathcal{H}(t)]\|\geq \frac{\mathbb{E}[\mathcal{H}_j(t)]}{2\kappa^1_j} \right)\leq\frac{c_j}{(T-t)^2},~~\forall t\leq T-\kappa_j^2, ~~\forall j=1,\dots,n
\end{equation}
for some constants $\kappa^1_j, \kappa^2_j$ and $c_j$.

We now show through the following example for which  our WBNB is sublinear in $T$ and the example can not covered by the above concentration condition. The example can be constructed as follows. For $t=1,\dots,T-\sqrt{T}$, we have $\mathcal{P}_t=\mathcal{P}_1$ and for $t=T-\sqrt{T}+1,\dots,T$, we have $\mathcal{P}_t=\mathcal{P}_2$, where $\mathcal{P}_1$ and $\mathcal{P}_2$ are two different distributions and it is satisfied that $\mathcal{P}_2(\bm{\theta}=\bm{\theta}_j')=0$ for a $j'$. Clearly, \eqref{eqn:concentration} is not satisfied for $j=j'$, thus, the concentration condition does not hold. However, it is direct to check that the WBNB can be upper bounded by $O(\sqrt{T})$, which is sublinear in $T$.

\subsection{Analysis of the sub-optimality of static policies}

\subsubsection{Proof of \Cref{prop:staticpolicy}}\label{pfprop:static}
\begin{proof}{Proof:}
{We first describe a problem instance. Consider a single resource with initial capacity $c=3T/4$. The parameter $\bm{\theta}_t=(r_t,1)$; the functions $f(x_t,\bm{\theta}_t)=r_t\cdot x_t$ and $g(x_t,\bm{\theta}_t)= x_t$, where $x_t\in[0,1]$ is the decision variable at time period $t$. Suppose that the prior estimate $\hat{\mathcal{P}}$ is given by
\[
r_t=\left\{\begin{aligned}
&1, &\text{w.p.} ~\frac{3}{4}\\
&\text{Unif}[0,1],& \text{w.p.} ~\frac{1}{4}
\end{aligned}\right.
\]
for each $t$. Here $\text{Unif}[0,1]$ denotes a uniform distribution over $[0,1]$. Clearly, the deterministic upper bound under prior estimate $\hat{\mathcal{P}}$ takes a value of $\frac{3T}{4}$. Furthermore, we denote $\mathbb{E}[h_t^{\pi}(r)]$ as the expectation of the decision variable under the static policy $\pi$ when the reward of period $t$ is realized as $r$. Then, we make following claim.\\
\textbf{Claim:} There exists a $\hat{r}\in[1-\frac{2W_T}{T}, 1-\frac{W_T}{T}]$ such that $\sum_{t=1}^{3T/4}\mathbb{E}[h_t^{\pi}(\hat{r})]\leq T/4$.\\
Otherwise, suppose that for each $r\in[1-\frac{2W_T}{T}, 1-\frac{W_T}{T}]$, we have $\sum_{t=1}^{3T/4}\mathbb{E}[h_t^{\pi}(r)]\geq T/4$. Then, we compute the expected reward gained by the policy $\pi$ from $r\in[1-\frac{2W_T}{T}, 1-\frac{W_T}{T}]$ during the first $3T/4$ periods, where the budget will never be violated.}{ To be specific, note that the event $r\in[1-\frac{2W_T}{T}, 1-\frac{W_T}{T}]$ happens with probability $\frac{W_T}{4T}$ at each period $t$. Comparing with $R_T^{UB}$, this will cause a gap of at least $\frac{W_T}{T}$ for $\pi$ to collect a reward $r\in[1-\frac{2W_T}{T}, 1-\frac{W_T}{T}]$. Thus, the gap of $R^{\text{UB}}_T$ and $\pi$ on $\hat{\mathcal{P}}$, caused from $\pi$ obtaining reward from $r\in[1-\frac{2W_T}{T}, 1-\frac{W_T}{T}]$ during the first $3T/4$ periods is at least
\[
\frac{W_T}{T}\cdot \frac{W_T}{4T}\cdot \frac{T}{4}= C_1\cdot T^{1/2}
\]
which violates the regret upper bound. Thus, the claim is proved.

Denote by $\hat{r}$ the value described in the claim.
Now we construct a true distribution $\mathcal{P}=\{\mathcal{P}_1,\dots,\mathcal{P}_T\}$ as follows:
\[
r_t=\left\{\begin{aligned}
&\hat{r}, &\text{w.p.} ~\frac{3}{4}\\
&\text{Unif}[0,1],& \text{w.p.} ~\frac{1}{4}
\end{aligned}\right.
\]
Clearly, the deviation budget is upper bounded by $W_T$. Then, we compute the total expected reward that policy $\pi$ can collect on the true distribution $\mathcal{P}$.
The expected reward collected by policy $\pi$ during the first $3T/4$ periods is at most
\[
\frac{1}{4}\cdot \frac{1}{2}\cdot \frac{3T}{4}+\frac{3}{4}\cdot \frac{T}{4}\cdot\hat{r}\leq \frac{9T}{32}
\]
where the first term in the LHS denotes the upper bound of the expected reward collected from $\text{Unif}[0,1]$, and the second term in the LHS denotes the upper bound of the expected reward collect from $\hat{r}$, which follows from the claim. Finally, without considering the budget violation and we let $\pi$ to collect every reward during the last $T/4$ periods. The policy $\pi$ can collect reward at most
\[
\frac{T}{4}\cdot (\frac{1}{2}\cdot\frac{1}{4}+\frac{3}{4})=\frac{7T}{32}
\]
during the last $T/4$ periods. Thus, the policy $\pi$ can collect at most $T/2$ reward on the true distribution $\mathcal{P}$ during the entire horizon. However, the value of $\mathbb{E}_{\mathcal{P}}[R^{\text{UB}}_T]$ is at least $\frac{3T}{4}\cdot\hat{r}\geq\frac{3T}{4}\cdot(1-\frac{2W_T}{T})\geq \frac{9T}{16}$, when $W_T\leq\frac{T}{8}$ which clearly holds since $W_T$ grows sublinearly in $T$. Thus, the regret of policy $\pi$ on the true distribution $\mathcal{P}$ is at least $\frac{T}{16}$.}
\end{proof}

\subsubsection{Discussion on Bid Price Policy} \
\label{bidPrice}

The well-known bid price policy (\cite{talluri1998analysis}) computes a dual optimal solution (bid price) $\hat{\bm{p}}^*$ based on the prior estimates. Specifically, the policy makes the decision $\bm{x}_t$ at each period $t$ based on the fixed $\hat{\bm{p}}^*$ throughout the procedure:
\[
\bm{x}_t\in\text{argmax}_{\bm{x}\in\mathcal{X}}f(\bm{x},\bm{\theta}_t)-(\hat{\bm{p}}^*)^\top\cdot \bm{g}(\bm{x},\bm{\theta}_t).
\]
When there is no deviation between the true distributions and prior estimates, i.e., $\mathcal{P}_t=\hat{\mathcal{P}}_t$ for each $t$, \cite{talluri1998analysis} show that bid price policy is asymptotically optimal if each period is repeated sufficiently many times and the capacity is scaled up accordingly. The following example shows that the bid price policy may fail drastically even when the deviations between the true distributions and prior estimates are small. The example follows the same spirit as the lower bound examples in our paper.

Consider the following linear program as the underlying problem (\ref{PCP}) for the online stochastic optimization problem:
\begin{align}
   \max \ \ &  x_1+...+x_{2c}+ \frac{1}{2} x_{2c+1}+...+\frac{1}{2} x_{T}  \label{egg1} \\
    \text{s.t. }\ & x_1+...+x_{2c}+x_{2c+1}+...+x_{T} \le c\nonumber \\
    & 0 \le x_t \le 1\ \text{ for } t=1,...,T. \nonumber
\end{align}
where $c=\frac{T}{3}$ and without loss of generality, we assume $c$ is an integer. Suppose that the prior estimates for the coefficients in the objective function is larger than the true coefficients by $2\epsilon$ for the first $\frac{T}{3}$ time periods and by $\epsilon$ for the last $\frac{2T}{3}$ time periods. Then we obtain the following linear program based on the prior estimates.
\begin{align}
   \max \ \ & (1+2\epsilon) x_1+...+(1+2\epsilon) x_{c}+(1+\epsilon)x_{c+1}+...+(1+\epsilon)x_{2c}+(\frac{1}{2}+\epsilon)x_{2c+1}+...+(\frac{1}{2}+\epsilon)x_{T}  \nonumber \\
    \text{s.t. }\ & x_1+...+x_{2c}+x_{2c+1}+...+x_{T} \le c \label{egg2}\\
    & 0 \le x_t \le 1\ \text{ for } t=1,...,T.\nonumber
\end{align}
Obviously, the optimal dual solution for \eqref{egg2} can take any value in $(1+\epsilon,1+2\epsilon)$.  When we apply such a bid price policy to the true problem \eqref{egg1}, the policy will set $x_t=0$ throughout the horizon as the reward per time period under the true problem \eqref{egg1} is no greater than 1. Given the optimal objective value is $\frac{T}{3}$, the bid price policy will incur a regret of $\frac{T}{3}$.

As a remark, we note that the regret bound for our algorithm IGD($\hat{\gamma}$) is upper bounded by $2\epsilon T +\sqrt{T}$ which can be much smaller than $\frac{T}{3}$ for small $\epsilon.$


\subsection{Proof of \Cref{thm:random}}\label{pf:randomtheorem}

The proof follows the same argument as that of the previous upper bounds in Theorem \ref{uppertheorem1} and Theorem \ref{uppertheorem}. We prove the details for completeness.

\begin{proof}{Proof:}
From the proof of Theorem \ref{newnonstationtheorem}, we have
\[\begin{aligned}
\sum_{t=1}^{T}f(\tilde{\bm{x}}_t;\bm{\theta}_t)-\sum_{t=1}^{T}f(\bm{x}_t;\bm{\theta}_t)\leq q\cdot\sum_{i=1}^{m}\left[\sum_{t=1}^{T}g_i(\tilde{\bm{x}}_t;\bm{\theta}_t)-(c_i-1) \right]^{+}
\end{aligned}\]
which relates the total collected reward by the true action $\{\bm{x}_t\}_{t=1}^T$ and the virtual action $\{\tilde{\bm{x}}_t\}_{t=1}^T$. Here $[\cdot]^+$ denotes the positive part function. Furthermore, from Proposition \ref{uppernew2}, we have that
\[\begin{aligned}
\text{Reg}_T(\pi)\leq  \min_{\bm{p}\geq\bm{0}} L(\bm{p})-\mathbb{E}\left[\sum_{t=1}^{T}f(\bm{x}_t;\bm{\theta}_t) \right]&\leq \underbrace{\min_{\bm{p}\geq\bm{0}} L(\bm{p})-\mathbb{E}\left[\sum_{t=1}^{T}f(\tilde{\bm{x}}_t;\bm{\theta}_t)\right]}_{\text{I}}\\
&+\underbrace{q\cdot\mathbb{E}\left[\sum_{i=1}^{m}\left[\sum_{t=1}^{T}g_i(\tilde{\bm{x}}_t;\bm{\theta}_t)-(c_i-1) \right]^{+}\right]}_{\text{II}}
\end{aligned}\]
where the function $L(\cdot)$ is given in \eqref{eqn:Lrandom}.
Next, we bound the term I and term II separately to derive our regret bound.\\
\textbf{Bound I}: We first define the following function $\bar{L}(\cdot)$:
\[
\bar{L}(\bm{p})\coloneqq \frac{1}{T} \bm{p}^\top\bm{c}+\bar{\mathcal{P}}_Th(\bm{p};\bm{\theta}).
\]
Note that $\hat{\mathcal{P}}_T=\frac{1}{T}\sum_{t=1}^{T}\mathcal{P}_t$, it holds that $L(\bm{p})=T\cdot\bar{L}(\bm{p})$ for any $\bm{p}$. From \Cref{newupperlemma}, we know that for each $t$, $\|\bm{p}_t\|_{\infty}\leq q+1$ with probability $1$. In addition, for each $t$, the distribution of $\bm{p}_t$ is independent from the distribution of $\bm{\theta}_t$, then from Lemma \ref{wasserlemma}, we have that
\begin{equation}
   \min_{\bm{p}\geq\bm{0}} \bar{L}(\bm{p})\leq \mathbb{E}_{\bm{p}_t}\left[ \bar{L}(\bm{p}_t) \right] \leq \mathbb{E}_{\bm{p}_t}\left[ \frac{1}{T} \bm{p}_t^\top\bm{c}+\mathcal{P}_th(\bm{p}_t;\bm{\theta}_t) \right]+(q+1)(m+1)\cdot \mathcal{W}(\mathcal{P}_t,\bar{\mathcal{P}}_T),
   \label{randomstepTmp0}
\end{equation}
where the expectation is taken with respect to $\bm{p}_t$ in a random realization of the algorithm.
Thus, we have the first term
\[
\text{I}\leq \sum_{t=1}^{T}\mathbb{E}_{\bm{p}_t}\left[\frac{1}{T}\bm{c}^\top\bm{p}_t+\mathcal{P}_t\left\{h(\bm{p}_t;\bm{\theta}_t)-\hat{f}(\tilde{\bm{x}}_t;\bm{\theta}_t) \right\} \right]+(q+1)(m+1)\cdot \mathcal{W}(\mathcal{P}_t,\bar{\mathcal{P}}_T)
\]
which comes from first taking expectation over $f(\tilde{\bm{x}}_t;\bm{\theta}_t)$ for given $\tilde{\bm{x}}_t, \bm{\theta}_t$, and then combining \eqref{randomstepTmp0} with the relation $L(\bm{p})=T\cdot\bar{L}(\bm{p})$.
By the definition of $\tilde{\bm{x}}_t$, $h(\bm{p}_t;\bm{\theta}_t)-\hat{f}(\tilde{\bm{x}}_t;\bm{\theta}_t)=-\bm{p}_t^\top\cdot \hat{\bm{g}}(\tilde{\bm{x}}_t;\bm{\theta}_t)$,
which implies that
\begin{equation}\label{randomstepTmp1}
\text{I}\leq \sum_{t=1}^{T}\mathbb{E}_{\bm{p}_t}\left[\bm{p}_t^\top\cdot\left( \frac{\bm{c}}{T}-\mathcal{P}_t \hat{\bm{g}}(\tilde{\bm{x}}_t;\bm{\theta}_t)\right)\right]+(q+1)(m+1)\cdot \mathcal{W}(\mathcal{P}_t,\bar{\mathcal{P}}_T)
\end{equation}
Note that from the update rule \eqref{new2003}, we have that
\[
    \|\bm{p}_{t+1}\|_2^2\leq\|\bm{p}_t\|_2^2+\frac{1}{T}\cdot\|\bm{g}(\tilde{\bm{x}}_t;\bm{\theta}_t)-\frac{\bm{c}}{T}\|_2^2-\frac{2}{\sqrt{T}}\cdot \bm{p}_t^\top\cdot\left( \frac{\bm{c}}{T}-\bm{g}(\tilde{\bm{x}}_t;\bm{\theta}_t)\right).
\]
By taking expectation with respect to both sides,
\begin{equation}
   \mathbb{E}_{\bm{p}_t}\left[\bm{p}_t^\top\cdot\left( \frac{\bm{c}}{T}-\mathcal{P}_t \hat{\bm{g}}(\tilde{\bm{x}}_t;\bm{\theta}_t)\right)\right]\leq \frac{\sqrt{T}}{2}\cdot\left(\mathbb{E}[\|\bm{p}_t\|_2^2]-\mathbb{E}[\|\bm{p}_{t+1}\|_2^2] \right)+\frac{m}{2\sqrt{T}}
\label{randomstepTmp2}
\end{equation}
Plugging \eqref{randomstepTmp2} into \eqref{randomstepTmp1}, we obtain an upper bound on Term I,
\begin{equation}\label{regret01x}
\text{I}\leq \frac{m\sqrt{T}}{2}+(q+1)(m+1)\cdot W_T
\end{equation}
\textbf{Bound II}: Note that from the update rule \eqref{new2003}, we have
\[
\sqrt{T}\cdot\bm{p}_{t+1}\geq \sqrt{T}\cdot\bm{p}_t+\bm{g}(\tilde{\bm{x}}_t;\bm{\theta}_t)-\frac{\bm{c}}{T}
\]
Taking a summation with respect to both sides,
\[
\sum_{t=1}^{T}\bm{g}(\tilde{\bm{x}}_t;\bm{\theta}_t)-\bm{c}\leq \sqrt{T}\cdot\bm{p}_{T+1}
\]
Applying \Cref{newupperlemma} for a bound on $\bm{p}_{T+1},$ we obtain the upper bound for Term II,
\begin{equation}\label{regret02x}
\text{II}=q\cdot\mathbb{E}\left[\sum_{i=1}^{m}\left[\sum_{t=1}^{T}g_i(\tilde{\bm{x}}_t;\bm{\theta}_t)-(c_i-1) \right]^{+}\right]\leq mq(q+1)\cdot\sqrt{T}+qm
\end{equation}
We obtain the desired regret bound by combining \eqref{regret01x} and \eqref{regret02x}.
\Halmos
\end{proof}

\subsection{Proof of \Cref{thm:batch}}

The proof follows the same argument as that of the previous upper bounds in Theorem \ref{uppertheorem1} and Theorem \ref{uppertheorem}. We prove the details for completeness.

We first prove the following lemma, which implies that the dual variable updated in \eqref{batch003} is always bounded. The analysis is similar to that of Lemma \ref{newupperlemma}.

\begin{lemma}\label{lem:batch}
Under Assumption \ref{assume}, for each $l=1,2,\dots,\lfloor T/K\rfloor$, the dual price vector satisfies $\|\bm{p}_l\|_{\infty}\leq q+\alpha_{T,K}\cdot K$, where $\bm{p}_l$ is specified by \eqref{batch003} in \Cref{alg:BatchSOA} and the constant $q$ is defined in Assumption \ref{assume} (c).
\end{lemma}
\begin{proof}{Proof:}
Note that the following two properties are satisfied by the update rule \eqref{batch003}:
\begin{itemize}
  \item[(i).] If $\|\bm{p}_l\|_{\infty}\leq q$, then we must have $\|\bm{p}_{l+1}\|_{\infty}\leq q+\alpha_{T,K}\cdot K$ by noting that for each $i$, the $i$-th component of $\bm{p}_l$, denoted as $p_{l,i}$, is nonnegative and $g_i(\cdot,\bm{\theta}_t)$ is normalized within $[0,1]$ for each $t=(l-1)K+1,\dots,lK$.
  \item[(ii).] If there exists $i$ such that $p_{l,i}>q$, then we must have $p_{l+1,i}<p_{l,i}$. Specifically, when $p_{l,i}>q$, we must have that $g_i(\tilde{\bm{x}}_t;\bm{\theta}_t)=0$ for each $t=(l-1)K+1,\dots,lK$, otherwise, we would have that
      \[
      f(\tilde{\bm{x}}_t;\bm{\theta}_t)-\bm{p}_l^\top\cdot \bm{g}(\tilde{\bm{x}}_t;\bm{\theta}_t)\leq f(\tilde{\bm{x}}_t;\bm{\theta}_t)-p_{l,i}\cdot g_i(\tilde{\bm{x}}_t;\bm{\theta}_t)<0
      \]
      which contradicts the definition of $\tilde{\bm{x}}_t$ in Algorithm \ref{alg:BatchSOA} since we could always select $\bm{x}_t=\bm{0}$ to obtain a zero objective value as per Assumption \ref{assume}. Then from \eqref{batch003}, it holds that $p_{l+1,i}<p_{l,i}$.
\end{itemize}
Starting from $\bm{p}_1=\bm{0}$ and iteratively applying the above two property to control the increase of $\bm{p}_l$ from $l=1$ to $\lfloor T/K\rfloor$, we obtain that for the first time that one component of $\bm{p}_l$ exceeds the threshold $q$, it is upper bounded by $q+\alpha_{T,K}\cdot K$ and this component will continue to decrease until it falls below the threshold $q$. Thus, it is obvious that we have $\|\bm{p}_l\|_{\infty}\leq q+\alpha_{T,K}\cdot K$ with probability $1$ for each $l$.
\Halmos
\end{proof}

\bigskip

Now we proceed to prove \Cref{thm:batch}.

\begin{proof}{Proof:}
From the proof of Theorem \ref{newnonstationtheorem}, we have
\[\begin{aligned}
\text{Reg}_T(\pi)\leq \underbrace{\sum_{t=1}^{T} \min_{\bm{p}\geq0} L_t(\bm{p})-\mathbb{E}\left[\sum_{t=1}^{T}f(\tilde{\bm{x}}_t;\bm{\theta}_t)\right]}_{\text{I}}
+\underbrace{q\cdot\mathbb{E}\left[\sum_{i=1}^{m}\left[\sum_{t=1}^{T}g_i(\tilde{\bm{x}}_t;\bm{\theta}_t)-(c_i-1) \right]^{+}\right]}_{\text{II}}
\end{aligned}\]
We then bound the term I and term II separately to derive our regret bound.\\
\textbf{Bound I}: 
Note that for each $t$, the distribution of $\bm{p}_l$ is independent from the distribution of $\bm{\theta}_{t}$ for any $t=(l-1)K+1,\dots,lK$, then we have that
\[
\min_{\bm{p}\geq\bm{0}}L_t(\bm{p})\leq \mathbb{E}_{\bm{p}_l}\left[ L_t(\bm{p}_l) \right] =\mathbb{E}_{\bm{p}_l}\left[  \bm{\gamma}_t^\top\bm{p}_l+\mathcal{P}_th(\bm{p}_l;\bm{\theta}_t)  \right]
\]
where the expectation is taken with respect to the randomness of the dual price $\bm{p}_t.$ Thus, we have
\[
\text{I}\leq \sum_{l=1}^{\lfloor T/K\rfloor}\sum_{t=(l-1)K+1}^{lK\wedge T}\mathbb{E}_{\bm{p}_l}\left[\bm{\gamma}_t^\top\bm{p}_l+\mathcal{P}_t\left\{h(\bm{p}_l;\bm{\theta}_t)-f(\tilde{\bm{x}}_t;\bm{\theta}_t) \right\} \right]
\]
From the definition of $\tilde{\bm{x}}_t$, we get that $h(\bm{p}_t;\bm{\theta}_t)-f(\tilde{\bm{x}}_t;\bm{\theta}_t)=-\bm{p}_l^\top\cdot \bm{g}(\tilde{\bm{x}}_t;\bm{\theta}_t)$ for $t=(l-1)K+1,\dots,lK$,
which implies that
\[
\text{I}\leq \sum_{l=1}^{\lfloor T/K\rfloor}\sum_{t=(l-1)K+1}^{lK\wedge T}\mathbb{E}_{\bm{p}_l}\left[\bm{p}_l^\top\cdot\left( \bm{\gamma}_t-\mathcal{P}_t \bm{g}(\tilde{\bm{x}}_t;\bm{\theta}_t)\right)\right]
\]
Note that from the update rule \eqref{batch003}, we have that
\[
\|\bm{p}_{l+1}\|_2^2\leq\|\bm{p}_l\|_2^2+\alpha_{T,K}^2\cdot\|\sum_{t=(l-1)K+1}^{lK}\bm{g}(\tilde{\bm{x}}_t;\bm{\theta}_t)-\bm{\gamma}_t\|_2^2-2\alpha_{T,K}\cdot\sum_{t=(l-1)K+1}^{lK\wedge T} \bm{p}_t^\top\cdot\left( \bm{\gamma}_t-\bm{g}(\tilde{\bm{x}}_t;\bm{\theta}_t)\right)
\]
which implies that
\[
\sum_{t=(l-1)K+1}^{lK\wedge T}\mathbb{E}_{\bm{p}_l}\left[\bm{p}_l^\top\cdot\left( \bm{\gamma}_t-\mathcal{P}_t \bm{g}(\tilde{\bm{x}}_t;\bm{\theta}_t)\right)\right]\leq \frac{1}{2\alpha_{T,K}}\cdot\left(\mathbb{E}[\|\bm{p}_l\|_2^2]-\mathbb{E}[\|\bm{p}_{l+1}\|_2^2] \right)+\frac{mK^2\cdot\alpha_{T,K}}{2}
\]
Thus, it holds that
\begin{equation}\label{batchregret101}
\text{I}\leq \frac{mTK\cdot\alpha_{T,K}}{2}
\end{equation}
\textbf{Bound II}: Note that from the update rule \eqref{batch003}, we have that
\[
\frac{1}{\alpha_{T,K}}\cdot\bm{p}_{l+1}\geq \frac{1}{\alpha_{T,K}}\cdot\bm{p}_l+\sum_{t=(l-1)K+1}^{lK}(\bm{g}(\tilde{\bm{x}}_t;\bm{\theta}_t)-\bm{\gamma}_t)
\]
which implies that
\[
\sum_{t=1}^{T}\bm{g}(\tilde{\bm{x}}_t;\bm{\theta}_t)-\bm{c}\leq \sum_{l=1}^{\lfloor T/K\rfloor}\sum_{t=(l-1)K+1}^{lK\wedge T}(\bm{g}(\tilde{\bm{x}}_t;\bm{\theta}_t)-\bm{\gamma}_t) \leq \frac{1}{\alpha_{T,K}}\cdot\bm{p}_{\lfloor T/K\rfloor+1}
\]
Thus, it holds that
\begin{equation}\label{batchregret102}
\text{II}=q\cdot\mathbb{E}\left[\sum_{i=1}^{m}\left[\sum_{t=1}^{T}g_i(\tilde{\bm{x}}_t;\bm{\theta}_t)-(c_i-1) \right]^{+}\right]\leq mq(q+\alpha_{T,K}\cdot K)\cdot\frac{1}{\alpha_{T,K}}+qm
\end{equation}
We obtain the $O(TK\alpha_{T,K}+\frac{1}{\alpha_{T,K}}+K)$ regret bound immediately by combining \eqref{batchregret101} and \eqref{batchregret102}.
\Halmos
\end{proof}

\renewcommand{\thesubsection}{E\arabic{subsection}}

\section{More Numerical Results for Experiment I}

\renewcommand{\thesubsection}{E\arabic{subsection}}

\DoubleSpacedXI
\begin{table}[ht!]
  \centering
  \begin{tabular}{|c|c|c|c|c|c|c|}
    \hline
    \multicolumn{2}{|c|}{}  & $\alpha=1$ & $\alpha=1.5$ & $\alpha=2$ & $\alpha=2.5$ & $\alpha=3$ \\
    \hline
    \multicolumn{2}{|c|}{Upper Bound} & $282.5433$ & $363.7044$ & $459.7807$ & $563.3545$ & $670.5960$ \\
    \hline
    \hline
    \multirow{4}{*}{$\beta=0$} &IGDP & $270.2411$ (96\%) & $349.1769$ (96\%)  & $441.6677$ (96\%) & $543.3373$ (96\%) & $645.6582$ (96\%) \\
    \cline{2-7}
     &UGD & $270.3621(96\%)$ & $337.3192$ (93\%) & $403.7044$ (88\%) & $469.7643$ (83\%) & $535.0654$ (80\%)  \\
     \cline{2-7}
     &FBP & $270.1211$ (96\%) & $347.4997$ (96\%) & $439.7016$ (96\%)  & $539.9865$ (96\%) & $642.3940$ (96\%)  \\
     \cline{2-7}
     &O2O &  270.6872 (96\%)      & 351.2534 (96\%)  & 440.9020 (96\%) & 545.2106 (96\%) & 644.4059 (96\%)\\
     \hline
     \hline
     \multirow{4}{*}{$\beta=0.5$} &IGDP & $270.1595$ (96\%) & $347.9148$ (96\%)  & $439.6166$ (96\%) & $539.8719$ (96\%) &$643.6777$ (96\%)  \\
    \cline{2-7}
     &UGD & $270.3568$ (96\%)  & $338.6916$ (93\%) & $405.7927$ (88\%) & $473.6640$ (84\%) &$540.0894$ (81\%)  \\
     \cline{2-7}
     &FBP & $64.7526$ (23\%) & $174.2642$ (48\%)  & $314.7806$ (68\%) &$446.7646$ (79\%)  &$582.7744$ (87\%)  \\
     \cline{2-7}
     &O2O & 202.3966 (72\%) & 289.6858 (80\%) & 423.5247 (92\%) & 536.0362 (95\%) & 621.4144 (93\%)\\
     \hline
     \hline

     \multirow{4}{*}{$\beta=1$} & IGDP & $269.8058$(95\%) & $347.1246$ (95\%) &$437.6279$ (95\%)  & $535.3521$ (95\%) & $638.8322$ (95\%) \\
    \cline{2-7}
     &UGD & $269.6893$ (95\%) & $339.1676$ (93\%) &$408.3862$ (89\%)  &$477.2329$ (84\%) & $544.9401$ (81\%) \\
     \cline{2-7}
     &FBP & $4.8549$ (2\%) & $53.7881$ (15\%) &$188.3271$ (41\%)  &$340.850$ (61\%)  & $486.4006$ (73\%) \\
     \cline{2-7}
     &O2O & 202.6792 (72\%) & 202.9555 (56\%) & 330.5638 (72\%) & 480.2536 (85\%)  &620.2584 (92\%) \\
     \hline
     \hline

     \multirow{4}{*}{$\beta=2$} & IGDP & $265.1512$ (94\%) & $343.7802$ (95\%) & $432.2275$ (94\%)  & $527.4351$ (94\%) & $627.7440$ (94\%) \\
    \cline{2-7}
     &UGD & $265.4187$ (94\%)   & $337.4751$ (93\%) & $410.3510$ (89\%) & $482.8652$ (86\%) & $554.0038$ (83\%) \\
     \cline{2-7}
     &FBP & $0.0171$ (0\%) & $1.6210$ (0.5\%)  & $22.5201$ (5\%) & $104.5026$ (19\%) & $243.1575$ (36\%) \\
     \cline{2-7}
     &O2O & 203.5851 (72\%) & 206.4316 (57\%) & 210.6264 (46\%) & 361.1129 (64\%) & 514.6730 (77\%)
 \\
     \hline
     \hline
  \end{tabular}
  \caption{Computation results for Experiment I under the uniform setting: The results are reported based on $500$ simulation trials. For each entry b(c) of the table, b denotes the expected reward collected by the algorithm and c denotes the percentage of the expected reward of the algorithm over the upper bound.}\label{tablenumerical2}
\end{table}
\OneAndAHalfSpacedXI

\DoubleSpacedXI
\begin{table}[ht!]
  \centering
  \begin{tabular}{|c|c|c|c|c|c|c|}
    \hline
    \multicolumn{2}{|c|}{}  & $\alpha=1$ & $\alpha=1.5$ & $\alpha=2$ & $\alpha=2.5$ & $\alpha=3$ \\
    \hline
    \multicolumn{2}{|c|}{Upper Bound} & 705.1450& 803.5559& 921.6550& 1060.5567& 1213.3552 \\
    \hline
    \hline
    \multirow{4}{*}{$\beta=0$} &IGDP & 679.9687 (96\%)  & 771.9511 (96\%)& 884.2661 (96\%)& 1022.9439 (96\%) &1171.1208 (97\%) \\
    \cline{2-7}
     &UGD & 681.1875 (97\%)& 764.6037 (95\%)& 840.5685 (91\%)& 928.6610 (88\%)& 1018.0635 (84\%)  \\
     \cline{2-7}
     &FBP & 674.968  (96\%)    &            768.0327  (96\%)     &           883.1946 (96\%)     &           1013.5706   (96\%)      &        1165.1827 (96\%) \\
     \cline{2-7}
     &O2O &  676.3823  (96\%)     &            769.396  (96\%)     &           882.5831 (96\%)   &             1017.4399   (96\%)       &       1159.9922(96\%)\\
     \hline
     \hline
     \multirow{4}{*}{$\beta=0.5$} &IGDP & 680.4068 (97\%) & 773.4652 (96\%) &890.2161 (96\%)& 1023.3161 (97\%) &1170.2227 (96\%)  \\
    \cline{2-7}
     &UGD & 680.931 (97\%)& 765.1636 (95\%)& 851.4681 (92\%)& 934.5877 (88\%)& 1028.4509 (85\%)  \\
     \cline{2-7}
     &FBP & 459.16   (65\%)     &          525.5739  (66\%)     &            619.198  (67\%)     &           751.7758   (71\%)        &       907.9537 (75\%) \\
     \cline{2-7}
     &O2O & 656.6648   (93\%)    &            712.215  (89\%)         &        784.1766 (85\%)      &           873.1237(82\%)       &            996.943 (82\%)\\
     \hline
     \hline
     \multirow{4}{*}{$\beta=1$} & IGDP & 680.5808 (97\%)& 777.7524 (97\%)& 888.5206 (96\%)& 1024.4519 (97\%) & 1175.2438 (97\%) \\
    \cline{2-7}
     &UGD & 681.824 (97\%)& 766.3631 (95\%)& 851.6300 (92\%)& 944.0707 (89\%)& 1035.5280 (85\%) \\
     \cline{2-7}
     &FBP & 254.2547 (36\%)     &            296.1898  (37\%)     &           375.3518 (41\%)    &             476.4165  (45\%)           &     610.2982 (50\%) \\
     \cline{2-7}
     &O2O & 588.1039   (83\%)    &           744.5618 (93\%)        &          833.9981 (90\%)        &          916.2876 (86\%)         &        1009.1384 (83\%) \\
     \hline
     \hline
     \multirow{4}{*}{$\beta=2$} & IGDP & 666.6237 (94\%) & 764.0318 (95\%)& 873.8978 (95\%)& 1014.7539 (96\%)& 1160.7693 (96\%) \\
    \cline{2-7}
     &UGD & 667.4160 (95\%)    &             756.7188 (94\%)       &          851.0122 (92\%)        &          944.1322 (89\%)           &      1046.7647 (86\%)
 \\
     \cline{2-7}
     &FBP & 51.9166  (7\%)    &             68.2078(8\%)        &          100.4066  (11\%)   &             148.1866 (14\%)          &       212.5127 (18\%) \\
     \cline{2-7}
     &O2O & 466.4348   (66\%)     &          535.7102  (67\%)    &            728.6945  (79\%)      &          957.5749  (90\%)     &          1072.6659(88\%)
 \\
     \hline
     \hline
  \end{tabular}
  \caption{Computation results for Experiment I under the normal setting: The results are reported based on $500$ simulation trials. For each entry b(c) of the table, b denotes the expected reward collected by the algorithm and c denotes the percentage of the expected reward of the algorithm over the upper bound.}\label{tablenumericalnormal}
\end{table}
\OneAndAHalfSpacedXI

\DoubleSpacedXI
\begin{table}[ht!]
  \centering
  \begin{tabular}{|c|c|c|c|c|c|c|}
    \hline
    \multicolumn{2}{|c|}{}  & $\alpha=1$ & $\alpha=1.5$ & $\alpha=2$ & $\alpha=2.5$ & $\alpha=3$ \\
    \hline
    \multicolumn{2}{|c|}{Upper Bound} & 532.6379&
630.1063&  746.5027&  871.63281  & 1010.7956 \\
    \hline
    \hline
    \multirow{4}{*}{$\beta=0$} &IGDP & 513.2625 (96\%) & 609.6643 (97\%) & 717.3434 (96\%) & 840.7550 (96\%) & 973.4778 (96\%) \\
    \cline{2-7}
     &UGD & 513.9690 (97\%) & 592.9739 (94\%) & 672.7648 (90\%) & 758.0174 (87\%) & 840.9000 (83\%)  \\
     \cline{2-7}
     &FBP & 511.9639 (96\%) & 601.4187 (95\%) & 714.4849 (96\%) & 831.7872 (95\%) & 964.0174 (95\%)  \\
     \cline{2-7}
     &O2O &  512.9872 (96\%) & 603.5403 (96\%) & 713.4253 (96\%) & 834.2972 (96\%) & 970.8423 (96\%)\\
     \hline
     \hline
     \multirow{4}{*}{$\beta=0.5$} &IGDP & 517.2177 (97\%) & 607.7310 (96\%) & 714.3025 (96\%) & 843.8988 (97\%) & 975.2676 (97\%)  \\
    \cline{2-7}
     &UGD & 514.1228 (96\%) & 596.0002 (95\%) & 678.3624 (91\%) & 762.7194 (87\%) & 844.0139 (84\%)  \\
     \cline{2-7}
     &FBP & 372.5857 (70\%) & 441.8714 (70\%) & 542.3811 (73\%)&  667.7373 (77\%) & 821.8302 (82\%)  \\
     \cline{2-7}
     &O2O & 478.3224 (90\%) & 600.7489 (95\%) & 681.4778 (91\%) & 774.0669 (89\%) & 881.9813 (87\%)\\
     \hline
     \hline

     \multirow{4}{*}{$\beta=1$} & IGDP & 517.8502 (97\%) & 609.3739 (97\%) & 717.3040 (96\%) & 842.7028 (97\%) & 977.9344 (97\%) \\
    \cline{2-7}
     &UGD & 515.1376 (97\%) & 598.2675 (95\%) & 683.0459 (92\%) & 769.4935 (88\%) & 853.1061 (85\%) \\
     \cline{2-7}
     &FBP & 253.2191 (47\%) & 295.4627 (47\%) & 372.9051 (50\%) & 469.2851 (54\%) & 613.1751 (61\%) \\
     \cline{2-7}
     &O2O & 391.8374 (73\%) & 509.6126 (81\%) & 674.1803 (90\%) & 813.8023 (93\%) & 912.8408 (91\%) \\
     \hline
     \hline

     \multirow{4}{*}{$\beta=2$} & IGDP & 507.9654 (95\%) & 598.2148 (95\%) & 713.2802 (96\%) & 835.5463 (96\%) & 963.8020 (95\%) \\
    \cline{2-7}
     &UGD & 505.8237 (95\%) & 591.2771 (94\%) & 684.4290 (92\%) & 770.5260 (88\%) & 866.7157 (86\%)
 \\
     \cline{2-7}
     &FBP & 78.3954 (15\%) &  96.3983 (15\%) & 136.5682 (18\%)  &197.7166 (23\%) & 281.8724 (28\%) \\
     \cline{2-7}
     &O2O & 343.4216 (64\%) & 358.0907 (57\%) & 494.0880 (66\%) & 686.6650 (79\%) & 890.9912 (88\%) \\
     \hline
     \hline
  \end{tabular}
  \caption{Computation results for Experiment I under the mixed setting: The results are reported based on $500$ simulation trials. For each entry b(c) of the table, b denotes the expected reward collected by the algorithm and c denotes the percentage of the expected reward of the algorithm over the upper bound.}\label{tablenumericalmixed}
\end{table}
\OneAndAHalfSpacedXI

\end{APPENDIX}
\end{document}